\documentclass[11pt]{article}

\usepackage{natbib}
\usepackage[margin=1in]{geometry}
\usepackage{palatino} 
\title{Transformers Learn to Achieve Second-Order Convergence Rates for In-Context Linear Regression}

\newcommand{\aspace}{\hspace{5mm}}

\author{%
\textbf{Deqing Fu} \aspace
\textbf{Tian-Qi Chen} \aspace
\textbf{Robin Jia}\aspace
\textbf{Vatsal Sharan}\\ 
Department of Computer Science\\
University of Southern California\\
\texttt{\{deqingfu,tchen939,robinjia,vsharan\}@usc.edu} \\
}

\date{}

\usepackage[utf8]{inputenc} 
\usepackage[T1]{fontenc}    
\usepackage[backref=section]{hyperref}
\usepackage{url}            
\usepackage{booktabs}       
\usepackage{amsfonts}       
\usepackage{nicefrac}       
\usepackage{microtype}      
\usepackage{xcolor}         

\usepackage{graphicx}
\usepackage{subfigure}

\usepackage{hyperref}
\hypersetup{breaklinks=true,colorlinks=true,citecolor=violet}

\usepackage{amsmath}
\usepackage{amssymb}
\usepackage{mathtools}
\usepackage{amsthm}

\usepackage[capitalize,noabbrev]{cleveref}

\theoremstyle{plain}
\newtheorem{theorem}{Theorem}[section]
\newtheorem{proposition}[theorem]{Proposition}
\newtheorem{lemma}[theorem]{Lemma}

\theoremstyle{definition}
\newtheorem{definition}[theorem]{Definition}

\theoremstyle{remark}
\newtheorem{remark}[theorem]{Remark}

\usepackage[textsize=tiny]{todonotes}

\usepackage{amsthm}
\usepackage{amsmath}
\usepackage{amssymb}
\usepackage{mathrsfs}
\usepackage{graphicx}
\usepackage{mathtools}
\usepackage{enumerate}
\usepackage{enumitem}
\usepackage{footnote}
\usepackage{float}
\usepackage{xspace}
\usepackage{multirow}
\usepackage{nicefrac}
\usepackage{wrapfig}
\usepackage{framed}
\usepackage{url}
\usepackage{caption}
\usepackage{blkarray}
\usepackage{makecell}
\usepackage[makeroom]{cancel}
\usepackage{soul}
\usepackage{comment}
\usepackage{lipsum}  
\usepackage{titletoc}
\usepackage{thmtools,thm-restate}

\usepackage{tikz}
\usetikzlibrary{arrows,chains,matrix,positioning,scopes}

\newtheorem*{theorem*}{Theorem}

\newcommand{\calA}{{\mathcal{A}}}
\newcommand{\calC}{{\mathcal{C}}}

\newcommand{\calE}{{\mathcal{E}}}

\newcommand{\zero}{\boldsymbol{0}}

\newcommand{\bA}{\boldsymbol{A}}
\newcommand{\bB}{\boldsymbol{B}}

\newcommand{\bd}{\boldsymbol{d}}

\newcommand{\bx}{\boldsymbol{x}}
\newcommand{\bxi}{\boldsymbol{\xi}}
\newcommand{\bh}{\boldsymbol{h}}

\newcommand{\bJ}{\boldsymbol{J}}
\newcommand{\bK}{\boldsymbol{K}}
\newcommand{\bX}{\boldsymbol{X}}
\newcommand{\bG}{\boldsymbol{G}}
\newcommand{\bO}{\boldsymbol{O}}
\newcommand{\bM}{\boldsymbol{M}}
\newcommand{\bu}{\boldsymbol{u}}
\newcommand{\by}{\boldsymbol{y}}
\newcommand{\bY}{\boldsymbol{Y}}

\newcommand{\ba}{\boldsymbol{a}}

\newcommand{\bH}{\boldsymbol{H}}
\newcommand{\bQ}{\boldsymbol{Q}}

\newcommand{\bI}{\boldsymbol{I}}

\newcommand{\bS}{\boldsymbol{S}}
\newcommand{\bT}{\boldsymbol{T}}
\newcommand{\bw}{\boldsymbol{w}}
\newcommand{\bW}{\boldsymbol{W}}

\newcommand{\bv}{\boldsymbol{v}}
\newcommand{\bV}{\boldsymbol{V}}
\newcommand{\bs}{\boldsymbol{s}}

\newcommand{\bSigma}{\boldsymbol{\Sigma}}

\newcommand{\btheta}{\boldsymbol{\theta}}

\newcommand{\bbeta}{\boldsymbol{\beta}}

\newcommand{\inner}[1]{ \left\langle {#1} \right\rangle }

\usepackage{cleveref}

\makeatletter
\renewcommand*\env@matrix[1][\arraystretch]{%
  \edef\arraystretch{#1}%
  \hskip -\arraycolsep
  \let\@ifnextchar\new@ifnextchar
  \array{*\c@MaxMatrixCols c}}
\makeatother

\makeatletter
\def\blfootnote{\xdef\@thefnmark{}\@footnotetext}
\makeatother

\newcommand{\gd}{Gradient Descent}
\newcommand{\newton}{Iterative Newton}
\newcommand{\iid}{\overset{\text{i.i.d.}}{\sim}}

\newcommand{\new}[1]{{\color{black} {#1}}}

\usepackage{amsmath,amsfonts,bm}

\def\eqref#1{equation~\ref{#1}}

\def\1{\bm{1}}

\DeclareMathAlphabet{\mathsfit}{\encodingdefault}{\sfdefault}{m}{sl}
\SetMathAlphabet{\mathsfit}{bold}{\encodingdefault}{\sfdefault}{bx}{n}

\newcommand{\R}{\mathbb{R}}

\DeclareMathOperator*{\argmax}{arg\,max}
\DeclareMathOperator*{\argmin}{arg\,min}

\renewcommand{\paragraph}[1]{\vspace{1mm}\noindent\textbf{#1}}

\begin{document}
\maketitle

\begin{abstract}
Transformers excel at \textit{in-context learning} (ICL)---learning from demonstrations without parameter updates---but how they do so remains a mystery. 
Recent work suggests that Transformers may internally run \gd\ (GD), a first-order optimization method, to perform ICL.
In this paper, we instead demonstrate that Transformers learn to approximate second-order optimization methods for ICL.
For in-context linear regression, Transformers share a similar convergence rate as \textit{\newton's Method}; both are exponentially faster than GD.
Empirically, predictions from successive Transformer layers closely match different iterations of Newton's Method \textit{linearly}, with each middle layer roughly computing 3 iterations; thus, Transformers and Newton's method converge at roughly the same rate. 
In contrast, \gd\ converges \textit{exponentially} more slowly.
We also show that Transformers can learn in-context on ill-conditioned data, a setting where \gd\ struggles but \newton\ succeeds. 
Finally, to corroborate our empirical findings, we prove that Transformers can implement $k$ iterations of Newton's method with $k + \mathcal{O}(1)$ layers.
\blfootnote{Our codes are available at \url{https://github.com/DeqingFu/transformers-icl-second-order}.}
\end{abstract}

\section{Introduction} \label{sec:intro}
\vspace{-1mm}

Transformer neural networks \citep{Vaswani2017AttentionIA} have become the default architecture for natural language processing \citep{devlin-etal-2019-bert,brown2020lm,openai2023gpt}.
    As first demonstrated by GPT-3 \citep{brown2020lm}, Transformers excel at \textit{in-context learning} (ICL)---learning from prompts consisting of input-output pairs, without updating model parameters. 
Through in-context learning, Transformer-based Large Language Models (LLMs) can achieve state-of-the-art few-shot performance across a variety of downstream tasks \citep{rae2022scaling,smith2022using,thoppilan2022lamda,chowdhery2022palm}.

Given the importance of Transformers and ICL, many prior efforts have attempted to understand how Transformers perform in-context learning.
Prior work suggests Transformers can approximate various linear functions well in-context \citep{Garg2022WhatCT}.
Specifically to linear regression tasks, prior work has tried to understand the ICL mechanism, and the dominant hypothesis is that Transformers learn in-context by running optimization internally through gradient-based algorithms \citep{Oswald2022TransformersLI,Oswald2023UncoveringMA,Ahn2023TransformersLT,Dai2023WhyCG,mahankali2024one}. 

This paper presents strong evidence for a competing hypothesis:
Transformers trained to perform in-context linear regression learn a strategy much more similar to a second-order optimization method than a first-order method like \gd\ (GD). 
In particular, Transformers approximately implement a second-order method with a convergence rate very similar to Newton-Schulz's Method, also known as the \textit{Iterative Newton's Method}, which iteratively improves an estimate of the inverse of the data matrix to compute the optimal weight vector.
Across many Transformer layers, subsequent layers approximately compute more and more iterations of Newton's Method, with increasingly better predictions; both eventually converge to the optimal minimum-norm solution found by ordinary least squares (OLS). 
Interestingly, this mechanism is specific to Transformers: LSTMs do not learn these same second-order methods, as their predictions do not even improve across layers.



We present both empirical and theoretical evidence for our claims. Empirically, Transformer layers demonstrate a similar rate of convergence to the OLS solution as second-order methods such as \newton, which is substantially faster than the rate of convergence of GD (\cref{fig:convergence}). 
The predictions made by the Transformer at successive layers closely match the predictions made by \newton\ after a proportional number of iterations, showing that they progress in similar ways at the same rate.
In contrast, to match the Transformer's predictions after $k$ layers, GD would have to run for  exponential in $k$ many steps (\cref{fig:heatmap}). Some individual Transformer layers make progress equivalent to hundreds of GD steps: these layers must be doing something more sophisticated than GD.
Furthermore, a crucial aspect of second-order methods is that they can handle ill-conditioned problems by correcting the curvature. We find that the convergence rate of Transformers is not significantly affected by ill-conditioning, which again matches \newton\ but not GD. 
To provide theoretical grounding to our empirical results, we show that Transformer circuits can efficiently implement \newton: one transformer layer can compute one Newton iteration (given $\mathcal O(1)$ pre/post-processing layers), and requires hidden states of dimension $\mathcal O(d)$ for a $d$-dimensional linear regression problem.
Overall, our work provides a mechanistic account of how Transformers perform ICL that explains model behavior better than previous hypotheses, and hints at why Transformers are well-suited for ICL relative to other architectures. 

\vspace{-1mm}
\section{Related Work} \label{sec:related_work}
\vspace{-1mm}

\paragraph{In-context learning by large language models. }
GPT-3 \citep{brown2020lm} first showed that Transformer-based large language models can ``learn'' to perform new tasks from in-context demonstrations (i.e., input-output pairs).
Since then, a large body of work in NLP has studied in-context learning, for instance by understanding how the choice and order of demonstrations affects results \citep{lu-etal-2022-fantastically,liu-etal-2022-makes,rubin-etal-2022-learning,su2023selective,chang-jia-2023-data,nguyen2023context}, studying the effect of label noise \citep{min-etal-2022-rethinking,yoo-etal-2022-ground,wei2023larger}, and proposing methods to improve ICL accuracy \citep{zhao2021calibrate,min-etal-2022-noisy,min-etal-2022-metaicl}.

\paragraph{In-context learning beyond natural language. }
Inspired by the phenomenon of ICL by large language models, subsequent work has studied how Transformers learn in-context beyond NLP tasks.
\citet{Garg2022WhatCT} first investigated Transformers' ICL abilities for various classical machine learning problems, including linear regression. 
We largely adopt their linear regression setup in this work.
\citet{Li2023TransformersAA} formalize in-context learning as an algorithm learning problem.
\citet{han2023incontext} suggests that Transformers learn in-context by performing Bayesian inference on prompts, which can be asymptotically interpreted as kernel regression. 
Other work has analyzed how Transformers do in-context classification \citep{AtaeeTarzanagh2023TransformersAS,tarzanagh2023maxmargin,Zhang2023TrainedTL}, the role of pertaining data \citep{raventós2023pretraining}, and the relationship between model architecture and ICL \citep{lee2023exploring}.

\paragraph{Do Transformers implement Gradient Descent? }
A growing body of work has suggested that Transformers learn in-context by implementing gradient descent within their internal representations.
\citet{Akyrek2022WhatLA} summarize operations that Transformers can implement, such as multiplication and affine transformations, and show that Transformers can implement gradient descent for linear regression using these operations.
Concurrently, \citet{Oswald2022TransformersLI} argue that Transformers learn in-context via gradient descent, where one layer performs one gradient update.
In subsequent work, \citet{Oswald2023UncoveringMA} further argue that Transformers are strongly biased towards learning to implement gradient-based optimization routines.
\citet{Ahn2023TransformersLT} extend the work of \citet{Oswald2022TransformersLI} by showing Transformers can learn to implement preconditioned \gd, where the pre-conditioner can adapt to the data. 
\citet{Bai2023TransformersAS} provide detailed constructions for how Transformers can implement a range of learning algorithms via gradient descent.
Finally, \citet{Dai2023WhyCG} conduct experiments on NLP tasks and conclude that Transformer-based language models performing ICL behave similarly to models fine-tuned via gradient descent\new{; however, concurrent work \citep{shen2023pretrained} argues that real-world LLMs do not perform ICL via gradient descent}.
\citet{mahankali2024one} showed that implementing gradient descent is a global minima for single layer linear self-attention. However, we study deeper models in this work, which can behave differently from single-layer models.
In this paper, we argue that Transformers actually learn to perform in-context learning by implementing a second-order optimization method, not gradient descent\footnote{After an initial version of this paper, \citet{vladymyrov2024linear} found that a variant of \gd\ can mimic Iterative Newton by approximating the inverse implicitly and getting second-order rates, which also supports our claim.}.

\paragraph{Mechanistic interpretability for Transformers. }
Our work attempts to understand the mechanism through which Transformers perform in-context learning.
Prior work has studied other aspects of Transformers' internal mechanisms, including 
reverse-engineering language models \citep{wang2022interpretability}, the grokking phenomenon  \citep{power2022grokking,nanda2023progress}, manipulating attention maps
\citep{hassid2022does}, and circuit finding \citep{conmy2023automated}. 

\paragraph{Theoretical Expressivity of Transformers. } \citet{pmlr-v202-giannou23a} provide a construction of looped transformers to implement \newton's method for solving pseudo-inverse, and each Newton iteration can be implemented by 13 looped Transformer layers. In contrast, our construction needs only one Transformer layer to compute one Newton iteration.

\section{Problem Setup} \label{sec:problem_definition}

\begin{wrapfigure}[]{R}{0.45\linewidth}
\vspace{-6mm}
    \includegraphics[width=\linewidth]{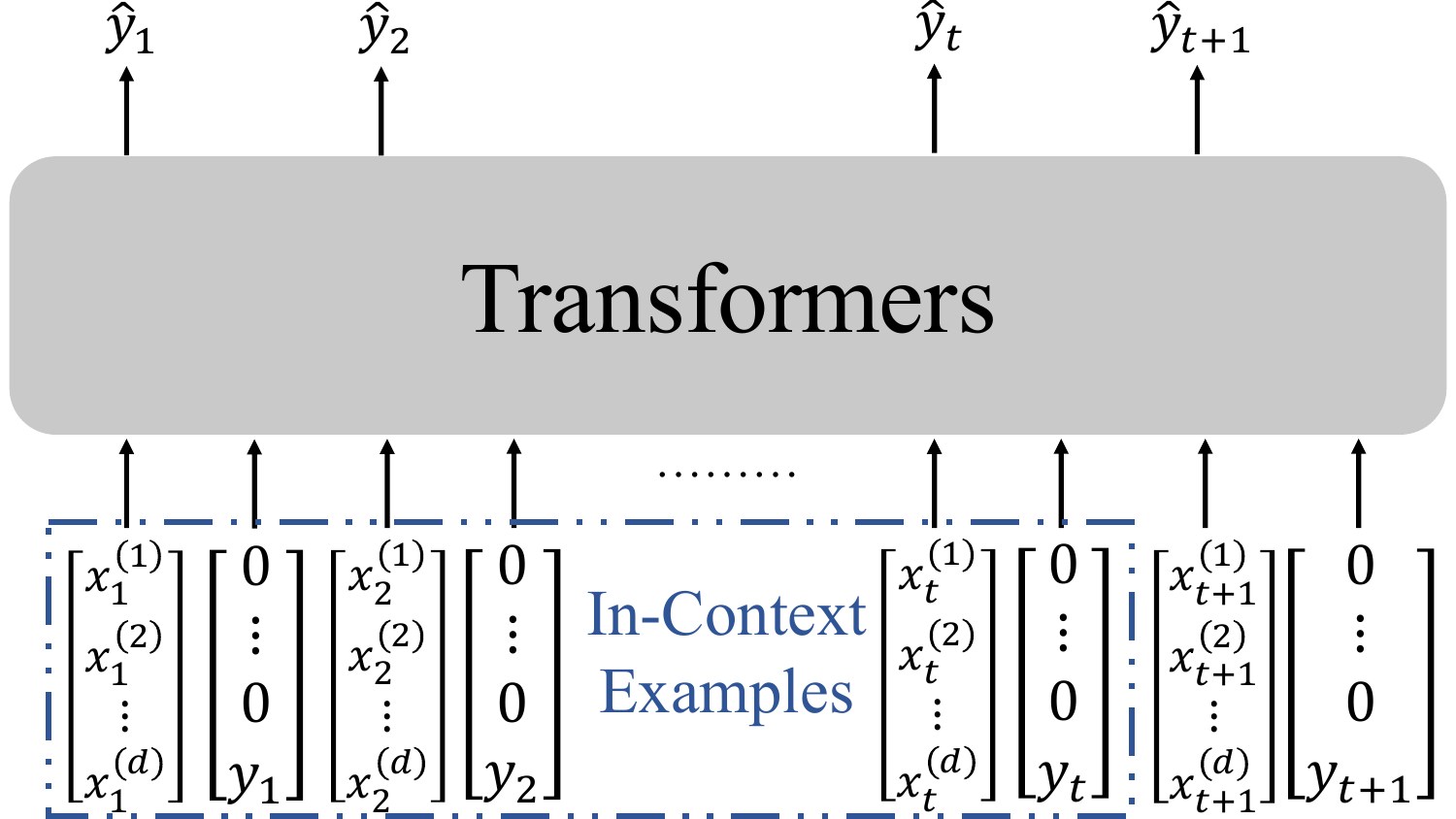}
    \captionof{figure}{Illustration of how Transformers are trained to do in-context linear regression.}
\vspace{-6mm}
\label{fig:transformer_illustration}
\end{wrapfigure}

In this paper, we  focus on the following linear regression task. The task involves $n$ examples $\{\bx_i, y_i\}_{i=1}^n$ where $\bx_i\in \R^d$ and $y_i \in \R$. The examples are generated from the following data generating distribution $P_{\mathcal{D}}$, parameterized by a distribution $\mathcal{D}$ over $(d \times d)$ positive semi-definite matrices. For each sequence of $n$ in-context examples, we first sample  a ground-truth weight vector $\bw^\star \iid \mathcal N(\zero, \bI) \in \mathbb R^d$ and a matrix $\bSigma\iid \mathcal{D}$. For $i\in [n]$, we sample each $\bx_i \iid \mathcal N(\zero, \bSigma)$. The label $y_i$ for each $\bx_i$ is given by $y_i = \bw^{\star\top} \bx_i$. Note that for much of our experiments $\mathcal{D}$ is only supported on the identity matrix $\bI$ and hence $\bSigma=\bI$, but we also consider some distributions over ill-conditioned matrices, which give rise to ill-conditioned regression problems.  Most of our results are on this noiseless setup and results with the noisy setup are in \Cref{ssec:noisy}.

\subsection{Standard Methods for Solving Linear Regression} \label{ssec:other_methods}

\new{Our central research question is:
\begin{center}
    \textit{\textbf{What convergence rate does the algorithm Transformers learn for linear regression achieve?}}
\end{center}}


To investigate this question, we first discuss various known algorithms for linear regression. We then compare them with Transformers empirically in \S\ref{sec:experiments} and theoretically in \S\ref{sec:mechanism}, to evaluate if Transformers are more similar to first-order or second-order methods. We care particularly about algorithms' convergence rates (the number of steps required to reach an $\epsilon$ error).

For any time step $t$, let $\bX^{(t)} = \begin{bmatrix}
    \bx_1  & \cdots & \bx_t
\end{bmatrix}^\top$ be the data matrix and $\by^{(t)} = \begin{bmatrix}
    y_1 & \cdots & y_t
\end{bmatrix}^\top$ be the labels for all the datapoints seen so far. Note that since $t$ can be smaller than the data dimension $d$, $\bX^{(t)}$ can be singular.  We now consider various algorithms for making predictions for $\bx_{t+1}$ based on $\bX^{(t)}$ and $\by^{(t)}$. When it is clear from context,  we drop the superscript and refer to $\bX^{(t)}$ and $\by^{(t)}$ as $\bX$ and $\by$, where  $\bX$ and $\by$ correspond to all the datapoints seen so far.

\paragraph{Ordinary Least Squares.}
This method finds the minimum-norm solution to the objective:
\begin{equation}
    \mathcal L(\bw \mid \bX, \by) = \frac{1}{2n}\|\by - \bX \bw \|_2^2.
    \label{eqn:obj}
\end{equation}
The Ordinary Least Squares (OLS) solution has a closed form given by the Normal Equations:\begin{equation}
    \hat{\bw}^\mathrm{OLS} = (\bX^\top \bX)^\dagger \bX^\top \by 
\end{equation}
where $\bS := \bX^\top \bX$  and $\bS^\dagger$ is the pseudo-inverse \citep{moonre1920pseudo} of $\bS$.

\paragraph{Gradient Descent.}  Gradient descent (GD) is a first-order method which finds the weight vector $\hat{\bw}^\mathrm{GD}$ with initialization $\hat{\bw}^\mathrm{GD}_0 = \zero$ using the iterative update rule:
\begin{equation}
\begin{aligned}
    & \hat{\bw}^\mathrm{GD}_{k+1} = \hat{\bw}^\mathrm{GD}_{k} - \eta \nabla_{\bw} \mathcal L(\hat{\bw}^\mathrm{GD}_{k} \mid \bX, \by).
\end{aligned}
\end{equation}
It is known that GD requires $\mathcal O\left(\kappa(\bS) \log(1/\epsilon)\right)$  steps to converge to an $\epsilon$ error where $\kappa(\bS) = \frac{\lambda_{\max}(\bS)}{\lambda_{\min}(\bS)}$ is the \textit{condition number}. Thus, when $\kappa(\bS)$ is large, GD converges slowly \citep{boyd2004convex}. 

\paragraph{Online Gradient Descent.}
While GD computes the gradient with respect to the full data matrix $\bX$ at each iteration, Online Gradient Descent (OGD) is an online algorithm that only computes gradients on the newly received data point $\{\bx_k, y_k\}$ at step $k$:
\begin{equation}
    \hat{\bw}^\mathrm{OGD}_{k+1} = \hat{\bw}^\mathrm{OGD}_{k} - \eta_k \nabla_{\bw} \mathcal L(\hat{\bw}^\mathrm{OGD}_{k} \mid \bx_k, y_k).
\end{equation}
Picking $\eta_k = \frac{1}{\|\bx_k\|_2^2}$ ensures that the new weight vector $\hat{\bw}^\mathrm{OGD}_{k+1}$ makes zero error on $\{\bx_k, y_k\}$.

\paragraph{\newton's Method.} This is a second-order method which finds the weight vector $ \hat{\bw}^\mathrm{Newton}$ by iteratively apply Newton's method to finding the pseudo inverse of $\bS = \bX^\top \bX$ \citep{schulz1933inverse,adi1965ai}. 
\begin{equation}
    \begin{aligned}
        &\bM_0 = \alpha \bS \textrm{, where } \alpha = \frac{2}{\|\bS \bS^\top\|_2 }, ~~ \hat{\bw}^\mathrm{Newton}_0 = \bM_0 \bX^\top \by, \\
        & \bM_{k+1} = 2 \bM_{k} - \bM_k \bS \bM_k, ~~ \hat{\bw}^\mathrm{Newton}_{k+1} =  \bM_{k+1}  \bX^\top \by. 
    \end{aligned}
\end{equation}
This computes an approximation of the  psuedo inverse using the moments of $\bS$. In contrast to GD, the \newton's method only requires $\mathcal O(\log \kappa(\bS) + \log \log (1/\epsilon))$ steps to converge to an $\epsilon$ error \citep{soderstrom1974onp,Pan1991AnIN}. Note that this is exponentially faster than the convergence rate of GD. 
We discuss additional algorithms such as Conjugate Gradient, BFGS, and L-BFGS in the Appendix \ref{sec:second-order}.

\vspace{-1mm}
\subsection{Solving Linear Regression with Transformers}
\vspace{-1mm}

We will use neural network models such as Transformers to solve this linear regression task. As shown in \Cref{fig:transformer_illustration}, at time step $t+1$, the model sees the first $t$ in-context examples $\{\bx_i,y_i\}_{i=1}^t$, and then makes predictions for $\bx_{t+1}$, whose label $y_{t+1}$ is not observed by the Transformers model.

We randomly initialize our models and then train them on the linear regression task to make predictions for every number of in-context examples $t$, where $t\in [n]$. Training and test data are both drawn from $P_{\mathcal{D}}$. To make the input prompts contain both $\bx_i$ and $y_i$,  we follow same the setup as \citet{Garg2022WhatCT}'s to zero-pad $y_i$'s, and use the same GPT-2 model \citep{radford2019language} with softmax activation and causal attention mask  (discussed later in \cref{def:attn}). 

We now present the key mathematical details for the Transformer architecture, and how  they can be used for in-context learning.
First, the causal attention mask enforces that attention heads can only attend to hidden states of previous time steps, and is defined as follows.

\begin{definition}[Causal Attention Layer]\label{def:attn}
    A \textbf{causal} attention layer with $M$ heads and activation function $\sigma$  is denoted as $\mathrm{Attn}$ on any input sequence $\bH = \begin{bmatrix}
        \bh_1, \cdots, \bh_N
    \end{bmatrix} \in \mathbb R^{D \times N}$, where $D$ is the dimension of hidden states and $N$ is the sequence length. In the vector form,
    \begin{equation}
        \tilde{\bh}_t = [\mathrm{Attn}(\bH)]_t = \bh_t +  \sum_{m=1}^M \sum_{j=1}^t \sigma \left(\inner{\bQ_m \bh_t, \bK_m \bh_j}\right) \cdot \bV_m \bh_j.
    \end{equation}
\end{definition}
\citet{Vaswani2017AttentionIA} originally proposed the Transformer architecture with the Softmax activation function for the attention layers. Later works have found that replacing $\mathrm{Softmax}(\cdot)$ with $\frac{1}{t} \mathrm{ReLU}(\cdot)$ does not hurt model performance \citep{Cai2022EfficientViTEL,shen2023study,wortsman2023replacing}. The Transformers architecture is defined by putting together attention layers with feed forward layers:

\begin{definition}[Transformers]
    An $L$-layer decoder-based transformer with Causal Attention Layers is denoted as $\mathrm{TF}_{\btheta}$ and is a composition of a MLP Layer (with a skip connection) and a Causal Attention Layers. For input sequence $\bH^{(0)}$, the transformers $\ell$-th hidden layer is given by
    \begin{equation}
      \mathrm{TF}_{\btheta}^{\ell}(\bH^{(0)}) :=  \bH^{(\ell)} = \mathrm{MLP}_{\btheta_\mathrm{mlp}^{(\ell)}} \left(\mathrm{Attn}_{\btheta_\mathrm{attn}^{(\ell)}} (\bH^{(\ell-1)})\right). \nonumber
    \end{equation}
    where $\btheta = \{\btheta_\mathrm{mlp}^{(\ell)}, \btheta_\mathrm{attn}^{(\ell)}\}_{\ell=1}^L$ and $\btheta_\mathrm{attn}^{(\ell)} = \{\bQ_m^{(\ell)}, \bK_m^{(\ell)}, \bV_m^{(\ell)}\}_{m=1}^M$ has $M$ heads at layer $\ell$. 
\label{def:transformers}
\end{definition}

In particular for the linear regression task, Transformers perform in-context learning as follows
\begin{definition}[Transformers for Linear Regression] \label{def:readout}
Given in-context examples $\{\bx_1, y_1, \dotsc, \bx_t, y_t\}$, Transformers make predictions on a query example $\bx_{t+1}$  through a readout layer parameterized as $\btheta_{\mathrm{readout}} = \{\bu, v\}$, and the prediction $\hat{y}_{t+1}^{\mathrm{TF}} $ is given by
\begin{align*}
    \hat{y}_{t+1}^{\mathrm{TF}} &:= \mathrm{ReadOut} \Big[\underbrace{\mathrm{TF}_{\btheta}^L (\{\bx_1, \by_1, \cdots, \bx_t, \by_t, \bx_{t+1}\})}_{\bH^{(L)}} \Big] = \bu^\top \bH^{(L)}_{:,2t+1} + v.
\end{align*}
\end{definition}

\new{To compare the rate of convergence of iterative algorithms to that of Transformers, we treat the layer index $\ell$ of Transformers as analogous to the iterative step $k$ of algorithms discussed in \S \ref{ssec:other_methods}. Note that for Transformers, we need to re-train the $\mathrm{ReadOut}$ layer for every layer index $\ell$ so that they can improve progressively (see \S\ref{ssec:progress} and for experimental details) for linear regression tasks.}

\begin{figure*}[t]
    \centering
    \subfigure[Transformers]{\includegraphics[width=0.315\linewidth]{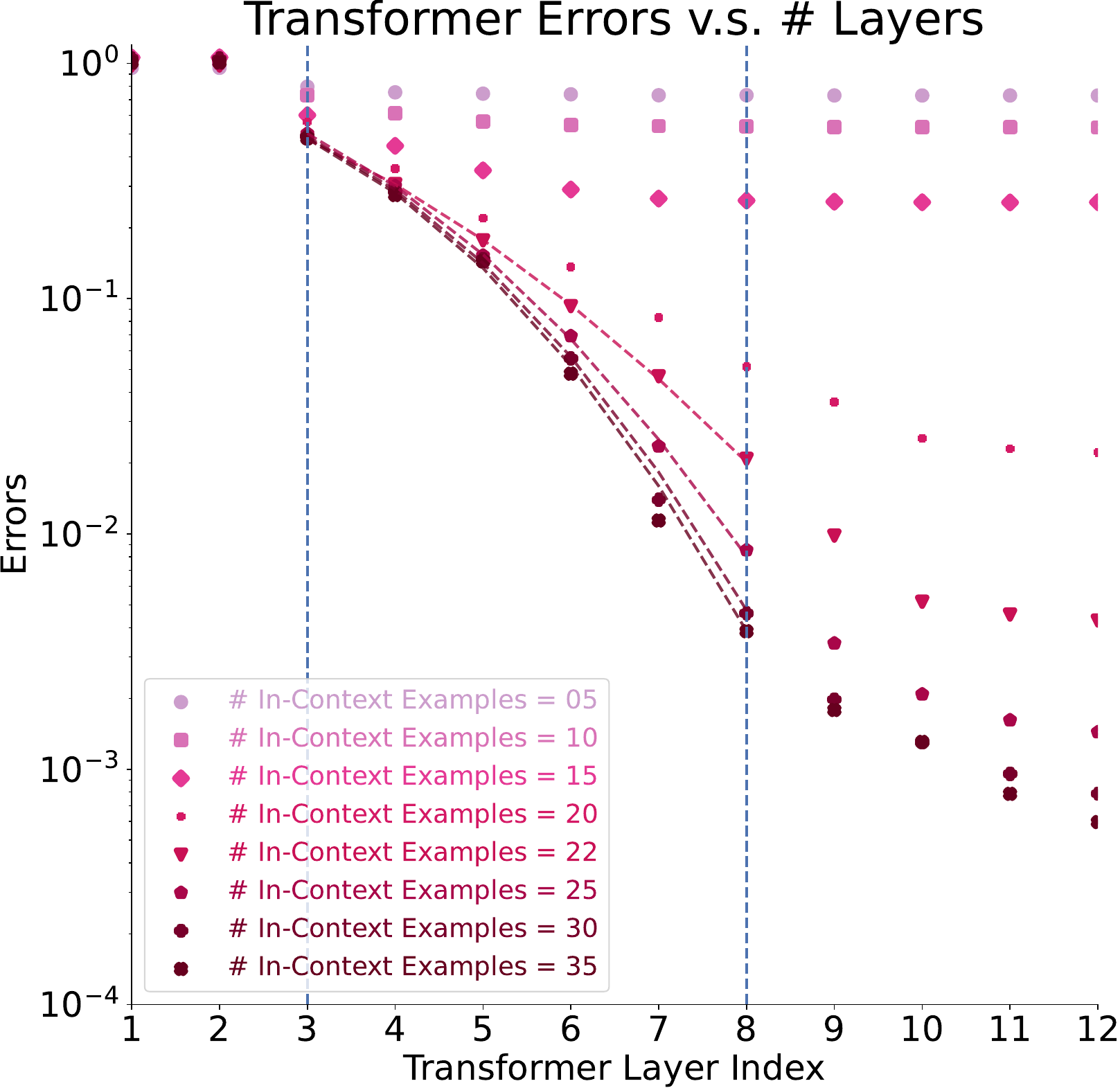} \label{subfig:transformers-convergence} } 
    \subfigure[Iterative Newton's Method]{\includegraphics[width=0.315\linewidth]{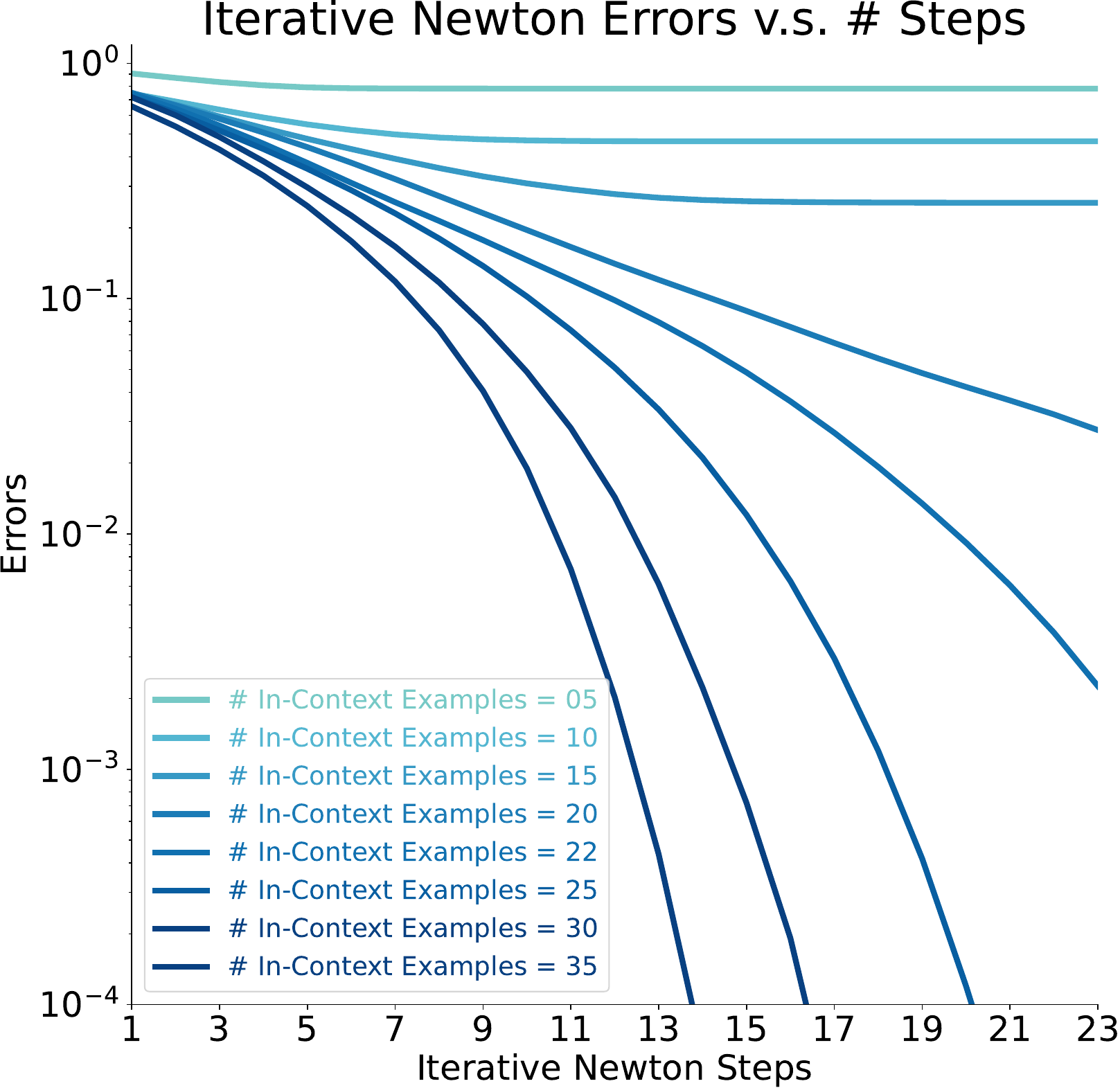} \label{subfig:newton-convergence} } 
    \subfigure[Gradient Descent]{\includegraphics[width=0.315\linewidth]{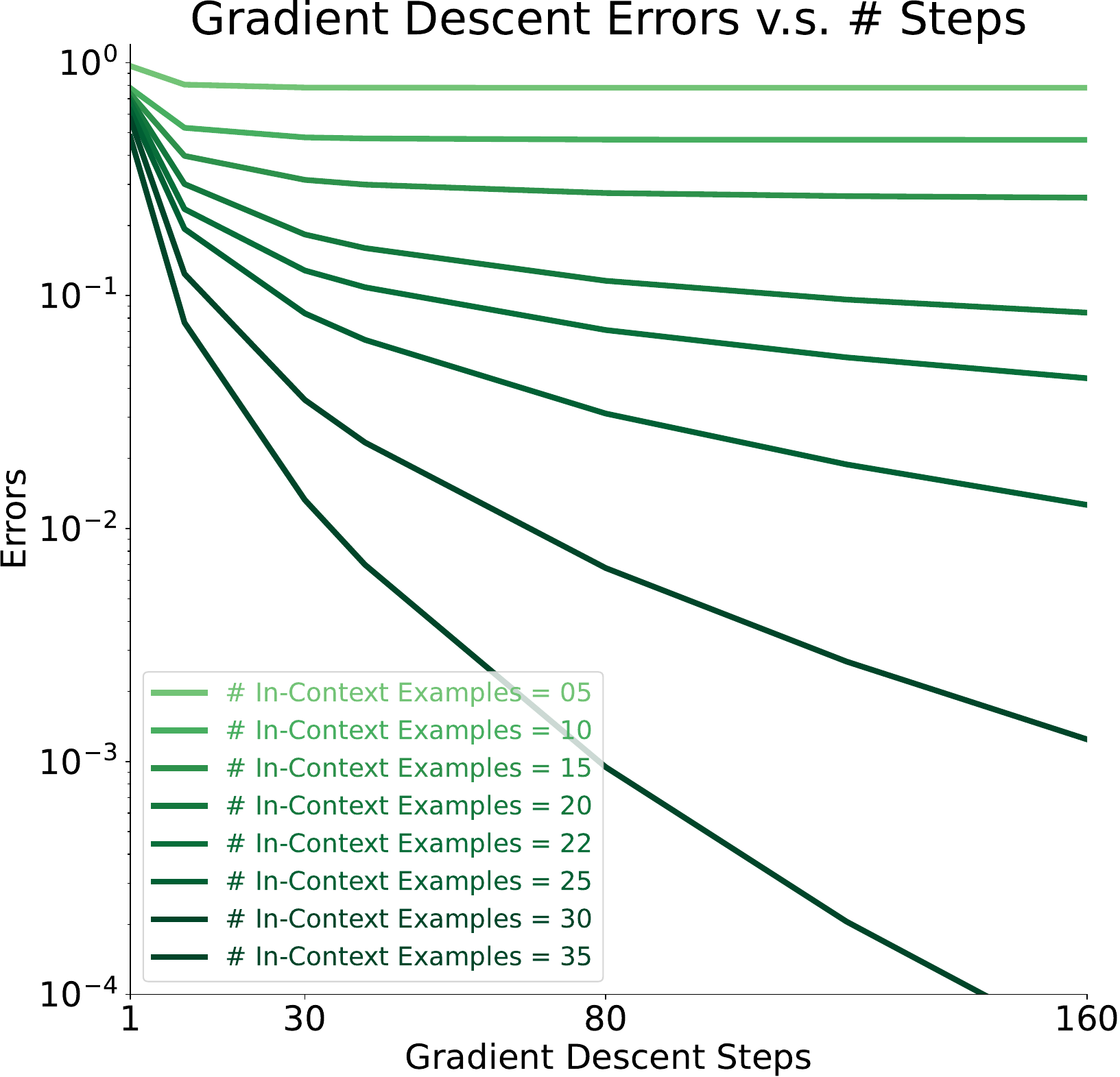} \label{subfig:gd-convergence}} 
    \caption{\textbf{Convergence of Algorithms. Similar to \newton\ and GD, Transformer's performance improve over the layer index $\ell$. When $n > d$, the Transformer model, from layers 3 to 8, demonstrates a superlinear convergence rate, similar to \newton, while GD, with fixed step size, is sublinear. Later layers of Transformers show a slower convergence rate, and we hypothesize they have little incentive to implement the algorithm precisely since the error is already very small. A 24-layer Transformer model exhibits the same superlinear convergence (\Cref{fig:convergence-24layer} in \S\ref{app:24-layer}).}}
    \vspace{-1.5ex}
    \label{fig:convergence}
\end{figure*}

\vspace{-1mm}
\subsection{Measuring Algorithmic Similarity}
\vspace{-1mm}

We propose two metrics to measure the similarity between linear regression algorithms.

\paragraph{Similarity of Errors.} This metric aims to measure similarity of algorithms through comparing prediction errors. For a linear regression algorithm $\calA$,
let $ \calA(\bx_{t+1} \mid \{\bx_i, y_i\}_{i=1}^{t})$ 
denote its prediction on the $(t+1)$-th in-context example $\bx_{t+1}$ after observing the first $t$ examples (see \Cref{fig:transformer_illustration}). 
We write $\calA(\bx_{t+1}) := \calA(\bx_{t+1} \mid \{\bx_i, y_i\}_{i=1}^{t})$ for brevity. Errors (i.e., residuals) on the sequence are:\footnote{the indices start from 2 to $n+1$ because we evaluate all cases where $t$ can choose from $1,\cdots, n$.}
\begin{equation}
    \calE(\calA \mid \{\bx_i, y_i\}_{i=1}^{n+1}) = \Big[
        \calA(\bx_{2}) - y_2 ,\cdots, \calA(\bx_{n+1}) - y_{n+1}
   \Big]^\top. \nonumber
\end{equation}
The similarity of errors for two algorithms $\calA_a$ and $\calA_b$ is the expected cosine similarity of their errors on a randomly sampled data sequence:
\begin{equation*}
    \mathrm{SimE}(\calA_a, \calA_b) 
    =   \mathop{\mathbb E}_{\{\bx_i, y_i\}_{i=1}^{n+1} \sim P_{\mathcal{D}}} \Bigg[
    \calC\Big(\calE(\calA_a | \{\bx_i, y_i\}_{i=1}^{n+1}), 
    \calE(\calA_b | \{\bx_i, y_i\}_{i=1}^{n+1})\Big)\Bigg]. 
\end{equation*}

Here $\calC(\bu, \bv) = \frac{
    \inner{\bu, \bv}
    }{\|\bu\|_2 \|\bv\|_2}$ is the cosine similarity, $n$ is the total number of in-context examples, and $P_{\mathcal{D}}$ is the data generation process discussed previously. 

\paragraph{Similarity of Induced Weights.}
All standard algorithms for linear regression estimate a weight vector $\hat{\bw}$. 
While neural ICL models like Transformers do not explicitly learn such a weight vector, similar to \citet{Akyrek2022WhatLA},
we can \emph{induce} an implicit weight vector $\tilde{\bw}$ learned by any algorithm $\calA$ by fitting a weight vector to its predictions.
We can then measure similarity of algorithms by comparing the induced $\tilde{\bw}$.
To do this, for any fixed sequence of $t$ in-context examples  $\{\bx_i, y_i\}_{i=1}^{t}$, we sample $T\gg d$ query examples  $\tilde{\bx}_{k} \iid \mathcal N(\zero, \bSigma)$, where $k\in [T]$. For this fixed sequence of in-context examples $\{\bx_i, y_i\}_{i=1}^{t}$, we create $T$ in-context prediction tasks and use the algorithm $\calA$ to make predictions $\calA(\tilde{\bx}_k \mid \{\bx_i, y_i\}_{i=1}^{t})$. We define the induced data matrix and labels as 
\begin{equation}
    \tilde{\bX} = \begin{bmatrix}
        \tilde{\bx}_1^\top \\  \vdots \\ \tilde{\bx}_T^\top
    \end{bmatrix}  \qquad \tilde{\bY} = \begin{bmatrix}
        \calA(\tilde{\bx}_1 \mid \{\bx_i, y_i\}_{i=1}^{t}) \\ 
        \vdots \\ 
        \calA(\tilde{\bx}_T \mid \{\bx_i, y_i\}_{i=1}^{t})
    \end{bmatrix}.
\end{equation}
The induced weight vector for $\calA$ and these $t$ examples is:
\begin{equation}
    \tilde{\bw}_t(\calA) := \tilde{\bw}_t(\calA \mid \{\bx_i, y_i\}_{i=1}^t) = (\tilde{\bX}^\top \tilde{\bX})^{-1} \tilde{\bX}^\top \tilde{\bY}.
\end{equation}
The similarity of induced weights between two algorithms $\calA_a$ and $\calA_b$ is the expected average cosine similarity\footnote{Alternative metrics such as $\ell_2$ distance gives the same observation. Here cosine similarity is better since errors usually have small magnitudes, and directions of induced weights are meaningful.} of induced weights $\tilde{\bw}_t(\calA_a)$ and $\tilde{\bw}_t(\calA_b)$ over all possible $1 \leq t \leq n$, on a randomly sampled data sequence:
\begin{equation*}
    \mathrm{SimW}(\calA_a, \calA_b) =  \mathop{\mathbb E}_{\{\bx_i, y_i\}_{i=1}^{n} \sim P_{\mathcal{D}}} \Bigg[
    \frac{1}{n} \sum_{t=1}^n
    \calC\Big(\tilde{\bw}_t(\calA_a | \{\bx_i, y_i\}_{i=1}^t), \tilde{\bw}_t(\calA_b | \{\bx_i, y_i\}_{i=1}^t))\Big)\Bigg].
\end{equation*}


\paragraph{Matching steps between algorithms.} Each algorithm converges to its predictions after several \textbf{steps} --- for example the number of iterations for \newton\ and GD, and  the number of layers for Transformers (see \Cref{ssec:progress}). 
When comparing two algorithms, given a choice of steps for the first algorithm, we match it with the steps for the second algorithm that maximize similarity. 
\begin{definition}[Best-matching Steps] Let $\mathcal M$ be the metric for evaluating similarities between two algorithms $\calA_a$ and $\calA_b$, which have steps $p_a \in [0, T_a]$ and $p_b \in [0, T_b]$, respectively. For a given choice of $p_a$, we define the best-matching number of steps of algorithm $\calA_b$ for $\calA_a$ as: 
\begin{equation}
    p_b^{\mathcal M}(p_a) := \argmax_{p_b \in [0, T_b]} \mathcal M(\calA_a(\cdot \mid p_a), \calA_b(\cdot \mid p_b)).
\end{equation} 
\vspace{-3.4mm}
\label{def:matching}
\end{definition}
In our experiments, we chose $T_a, T_b$ be large enough integers so the algorithms converge. The matching processes can be visualized as heatmaps as shown in \Cref{fig:heatmap}, where best-matching steps are highlighted. This enables us to compare the rate of convergence of algorithms. 
In particular, if two algorithms converge at the same rate, the best matching steps between the two algorithms should follow a linear trend.
We will discuss these results in \S\ref{sec:experiments}. See \Cref{fig:best-matching} on how best-matching steps help compare the convergence rates.

\vspace{-1mm}
\section{Experimental Evidence} \label{sec:experiments}
\vspace{-1mm}

We primarily study the Transformers-based GPT-2 model with 12 layers and 8 heads per layer. Alternative configurations with fewer heads per layer, or with more layers, also support our findings; we defer them to \S \ref{app:1-head} and \S\ref{app:24-layer}. We initially focus on isotropic cases where $\bSigma = \bI$ and later consider ill-conditioned $\bSigma$ in \S \ref{ssec:ill}. Our training setup is exactly the same as \citet{Garg2022WhatCT}: models are trained with at most $n = 40$ in-context examples for $d = 20$ (with the same learning rate, batch size etc.).  

We claim that Transformers learn high-order optimization methods in-context. We provide evidence that Transformers improve themselves with more layers in \S \ref{ssec:progress}; Transformers share the same rate of convergence as \newton, exponentially faster than that of GD, in \S \ref{ssec:tf_similar_newton}; and they also  perform well on ill-conditioned problems in \S \ref{ssec:ill}. Finally, we contrast  Transformers with LSTMs in \S \ref{ssec:lstm}.

\vspace{-1mm}
\subsection{Transformers improve progressively over layers} \label{ssec:progress}
\vspace{-1mm}

Many known algorithms for linear regression, including GD, OGD, and \newton, are \emph{iterative}:
their performance progressively improves as they perform more iterations, eventually converging to a final solution.
How can a Transformer implement such an iterative algorithm?
\citet{Oswald2022TransformersLI} propose that deeper \emph{layers} of the Transformer may correspond to more iterations;
in particular, they show that there exist Transformer parameters such that each attention layer performs one step of GD.

Following this intuition, we first investigate whether the predictions of a trained Transformer improve as the layer index $\ell$ increases. 
For each layer of hidden states $\bH^{(\ell)}$ (see \cref{def:transformers}), we re-train the \verb|ReadOut| to predict $y_t$ for each $t$; the new predictions are given by $\mathrm{ReadOut}^{(\ell)}\left[\bH^{(\ell)}\right]$. Thus for each input prompt, there are $L$ Transformer predictions parameterized by layer index $\ell$.
All parameters besides the \verb|ReadOut| layer parameters are kept frozen.

As shown in \Cref{subfig:transformers-convergence} (and \cref{subfig:transformers} in the Appendix), as we increase the layer index $\ell$, the prediction performance improves progressively. 
Hence, Transformers progressively improve their predictions over layers $\ell$, similar to how iterative algorithms improve over steps. \new{Such observations are consistent with language tasks where Transformers-based language models also improve their predictions along with layer progressions \citep{geva-etal-2022-transformer,chuang2023dola}.}

\begin{figure*}[t]
    \centering
\includegraphics[width=0.49\linewidth]{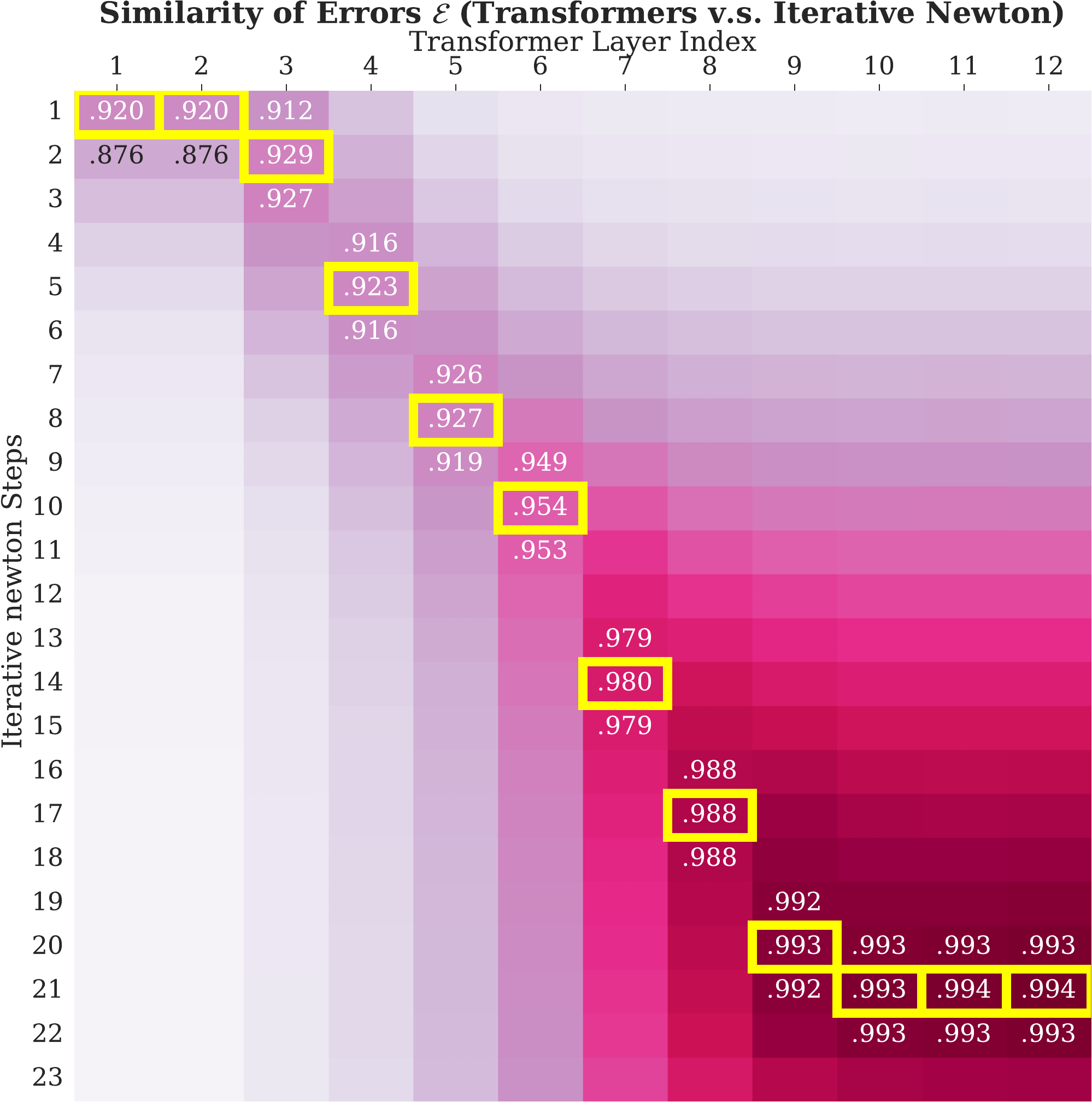}
\includegraphics[width=0.49\linewidth]{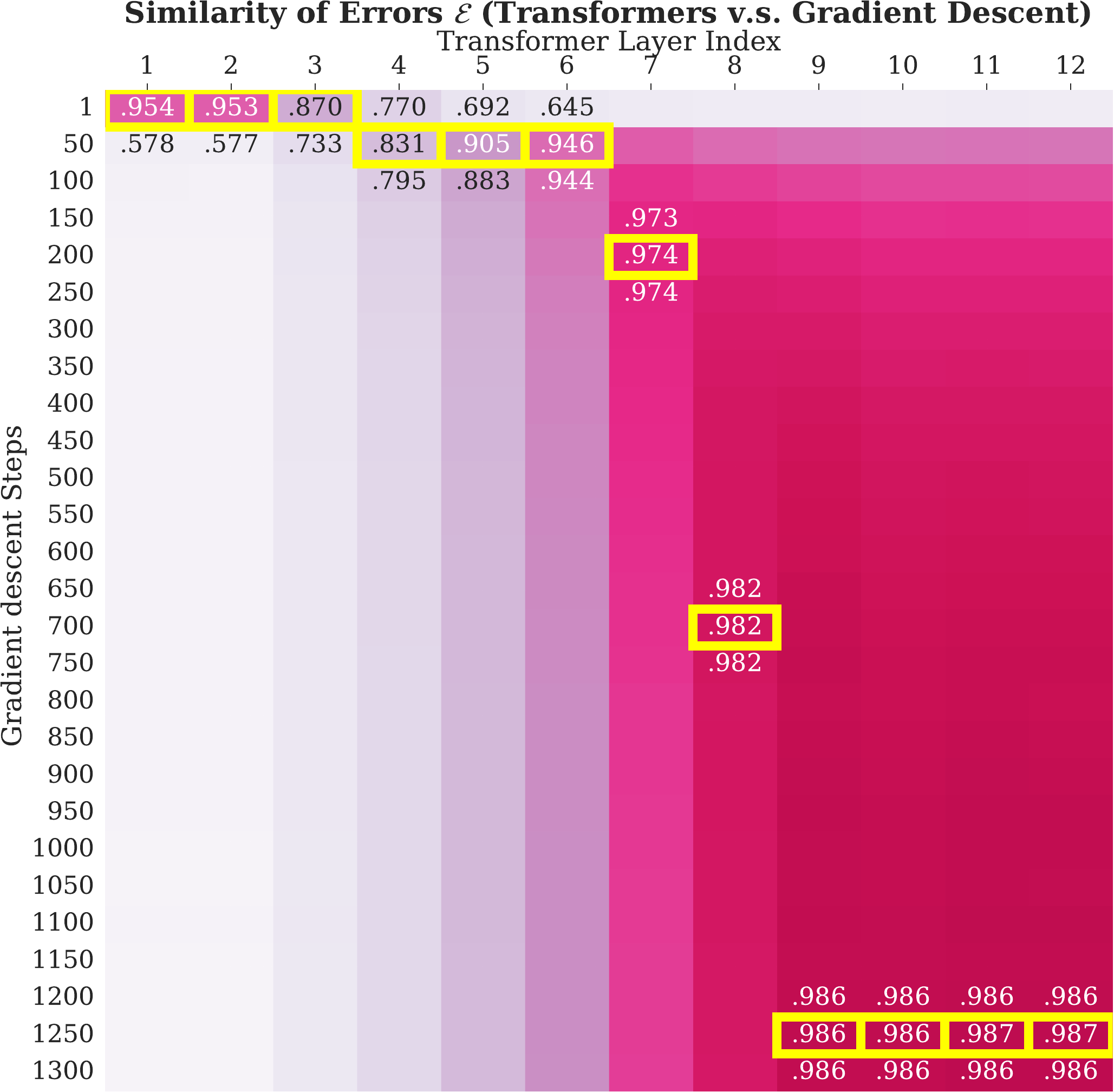}
\caption{\textbf{Heatmaps of Similarity.} The best matching steps are highlighted in yellow. Transformers layers show a linear trend with \newton\ steps but an exponential trend with GD. This suggests Transformers and \newton\ have the same convergence rate that is exponentially faster than GD. See \Cref{fig:gd_heatmap_log_scale} for an additional heatmap where GD's steps are shown in log scale: on that plot there is a linear correspondence between Transformers and GD's steps. This further strengthens the claim that Transformers have an exponentially faster rate of convergence than GD.}
\vspace{-1ex}
\label{fig:heatmap}
\end{figure*}

\vspace{-1mm}
\subsection{Transformers are more similar to second-order methods, such as \newton} \label{ssec:tf_similar_newton}
\vspace{-1mm}

We now test the more specific hypothesis that the iterative updates performed across Transformer layers are similar to the iterative updates for known iterative algorithms. First, \Cref{fig:convergence} shows that the middle layers of Transformers converge at a rate similar to \newton, and faster than GD.
In particular, the Transformer and \newton\ both converge at a superlinear rate, while GD converges at a sublinear rate. 

Next, we analyze whether each layer $\ell$ of the Transformer corresponds to performing $k$ steps of some iterative algorithm, for some $k$ depending on $\ell$.
We focus here on GD\ and \newton's Method; we will discuss online algorithms in \Cref{ssec:lstm}, and additional optimization methods in \Cref{sec:second-order}. We will discuss results on noisy linear regression tasks in \Cref{ssec:noisy}.

For each layer $\ell$ of the Transformer,
we measure the best-matching similarity (see Def.~\ref{def:matching}) with candidate iterative algorithms with the optimal choice of the number of steps $k$.
As shown in \Cref{fig:heatmap}, the Transformer has very high error similarity with \newton's method at all layers.
Moreover, we see a clear \emph{linear} trend between layer 3 and layer 9 of the Transformer, where each layer appears to compute roughly 3 additional iterations of \newton's method.
This trend only stops at the last few layers because both algorithms converge to the OLS solution;
Newton is known to converge to OLS (see \S\ref{ssec:other_methods}), and we verify in \Cref{app:exp_standard} that the last few layers of the Transformer also basically compute OLS (see \Cref{fig:over_examples} in the Appendix). 
We observe the same trends when using similarity of induced weights as our similarity metric (see \Cref{fig:full_heatmap_sim_w} in the Appendix).\Cref{fig:bfgs} in the Appendix shows that there is a similar \textit{linear} trend between Transformer and BFGS, an alternative quasi-Newton method. This is perhaps not surprising, given that BFGS also gets a superlinear convergence rate for linear regression \cite{nocedal1999numerical}.
Thus, we do not claim that Transformers specifically implement \newton, only that they (approximately) implement some second-order method.

In contrast, even though GD\ has a comparable similarity with the Transformers at later layers, their best matching follows an \textit{exponential} trend. As discussed in the \Cref{ssec:other_methods}, for well-conditioned problems where  $\kappa \approx 1$, to achieve $\epsilon$ error, the rate of convergence of GD\ is $\mathcal{O}(\log (1/\epsilon))$ while the rate of convergence of \newton\ is $\mathcal{O}(\log \log (1/\epsilon))$. Therefore the rate of convergence of \newton\ is exponentially faster than GD.  Transformer's \textit{linear} correspondence with \newton\ and its \textit{exponential} correspondence with GD\ provides strong evidence   that the rate of convergence of Transformers is similar to \newton, i.e., $\mathcal O(\log \log (1/\epsilon))$. We also note that it is not possible to significantly improve GD's convergence rate without using second-order methods: \citet{nemirovskii1983problem} showed a $ \Omega\big( \log(1/\epsilon) \big)$  lower bound on the convergence rate of gradient-based methods for smooth and strongly convex problems, and \cite{Arjevani2016OnLower} shows a similar lower bound specifically for quadratic problems. 

In the Appendix, we show that limited-memory BFGS \cite{Liu1989OnTL} and conjugate gradient (see \Cref{fig:conjugate-gradient}), which do not use full-second order information, also converge slower than Transformers. This provides further evidence for the usage of second-order information by Transformers. We also show more evidence by investigating alternative function classes such as linear regression with noises in \Cref{ssec:noisy} and 2-layer neural network with ReLU or Tanh activation function in \Cref{app:non-linear}. 

Overall, we conclude that a Transformer trained to perform in-context linear regression learns to implement an algorithm that is very similar to second-order methods, such as \newton's method, not GD. 
Starting at layer 3, subsequent layers of the Transformer compute more and more iterations of \newton's method.
This algorithm successfully solves the linear regression problem, as it converges to the optimal OLS solution in the final layers.

\begin{figure*}[t]
\centering
    \begin{minipage}[b]{0.32\linewidth}
            \centering
            \includegraphics[width=0.95\linewidth]{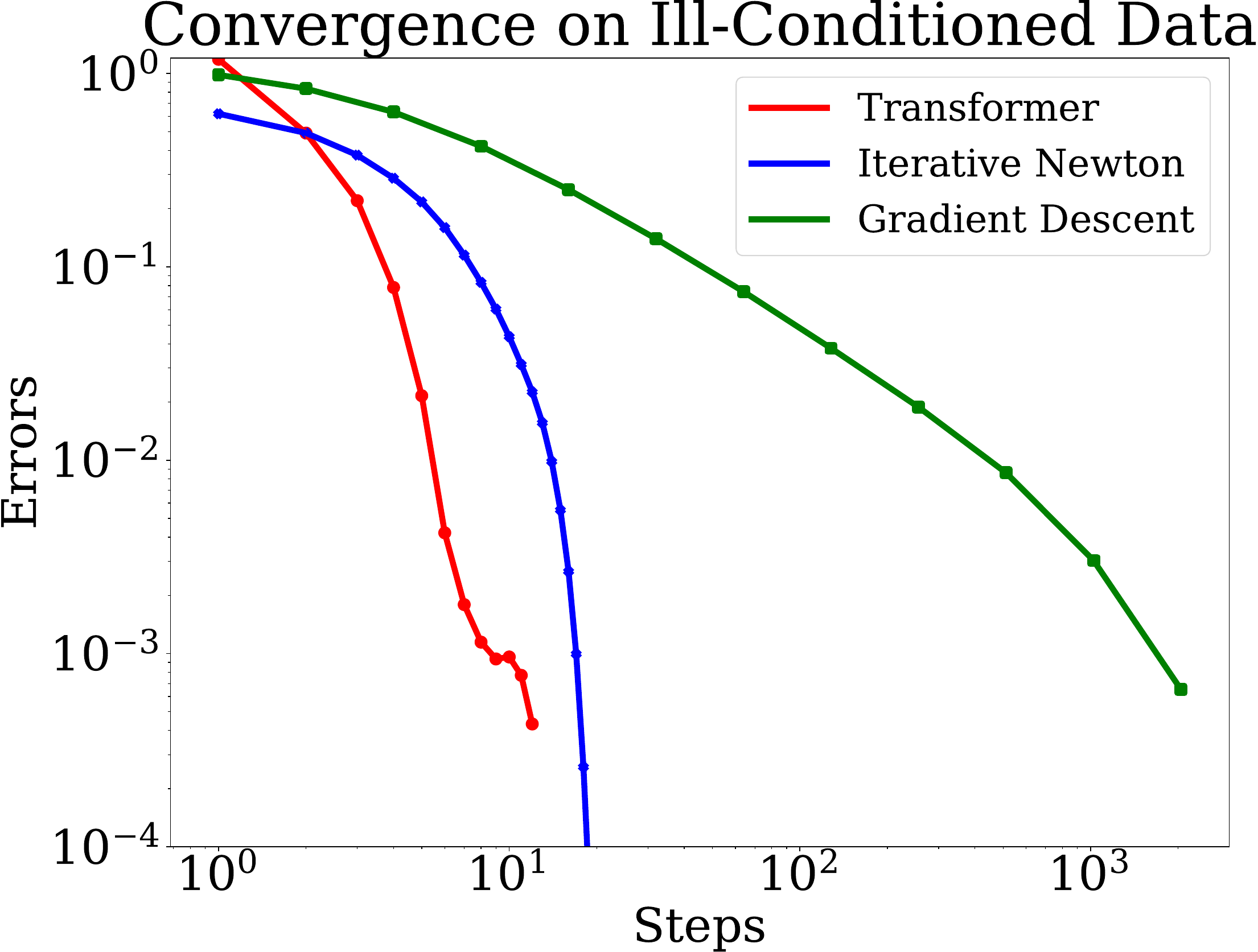}
            \vspace{1mm}
            \caption{Transformers performance on ill-conditioned data. Given 40 in-context examples, Transformers and \newton\ converge similarly and they both can converge to the OLS solution quickly whereas GD\ suffers.}
            \label{fig:ill}
    \end{minipage}
    \hfill
    \begin{minipage}[b]{0.65\linewidth}
    \centering
    \includegraphics[width=0.42\linewidth]{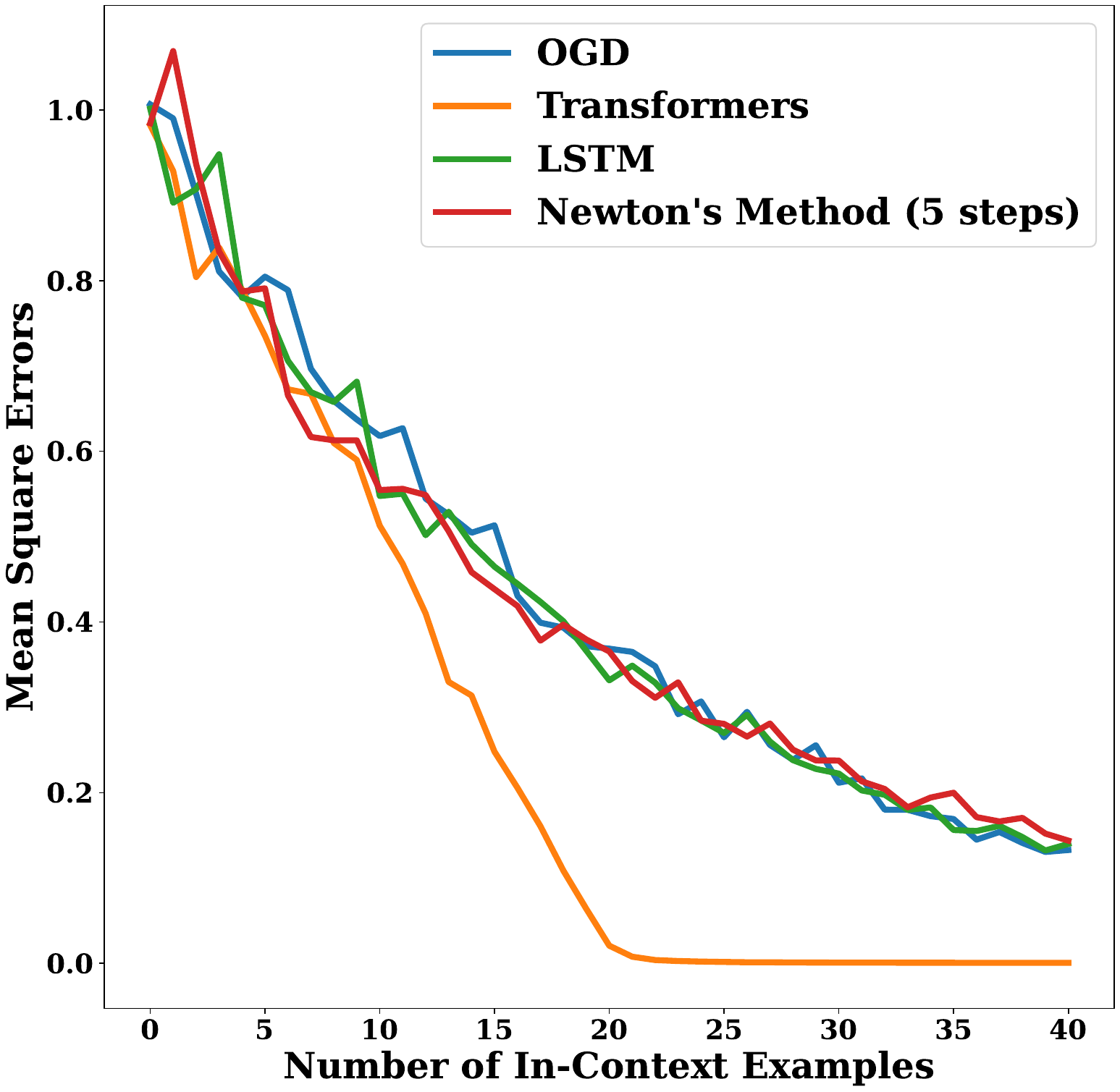}
    \includegraphics[width=0.43\linewidth]{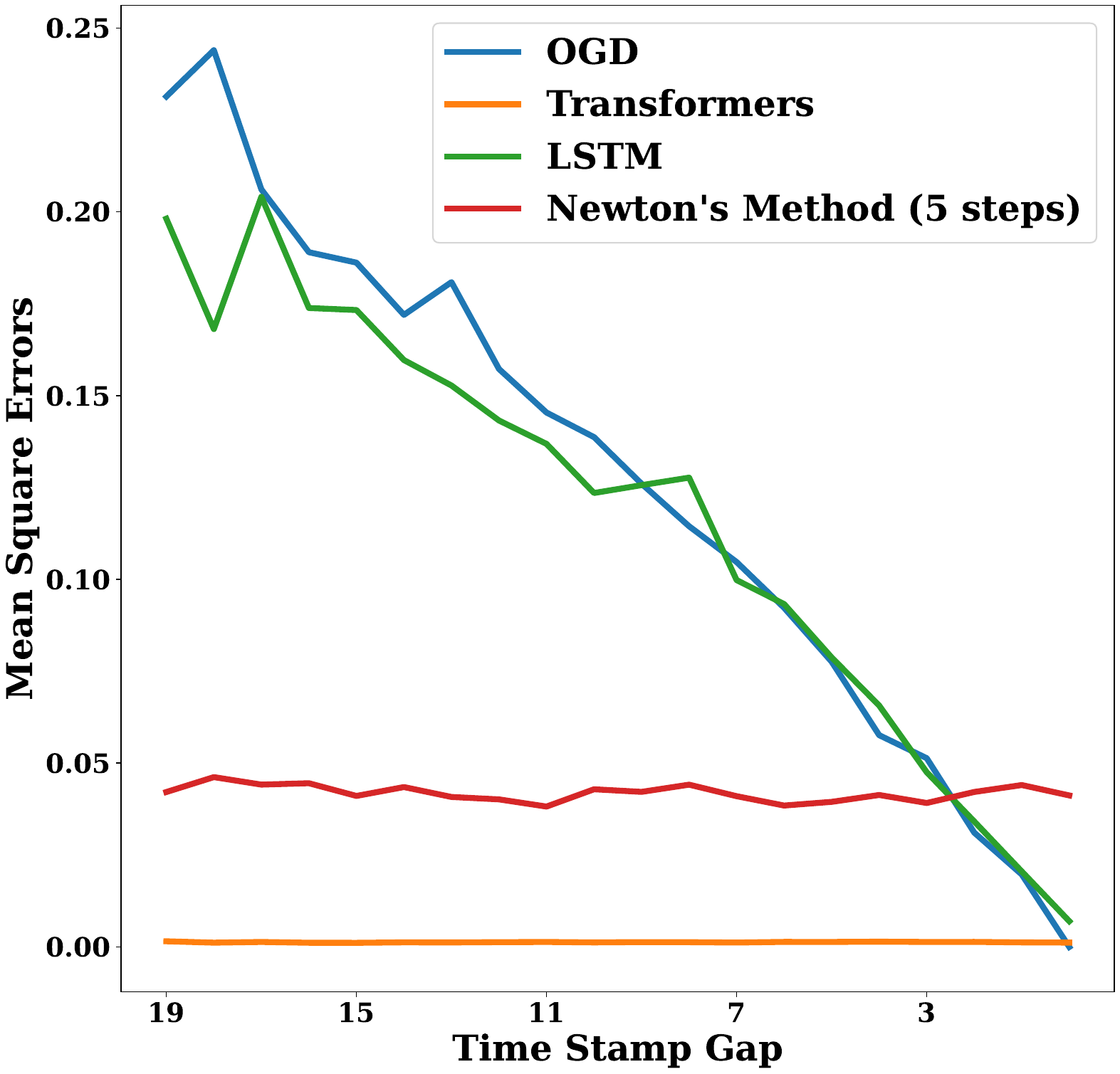}
    \caption{In the left figure, we measure model predictions with normalized MSE. Though LSTM is seemingly most similar to Newton's Method with only 5 steps, neither algorithm converges yet. OGD also has a similar trend as LSTM. In the right figure, we measure the model's error rate on example $\bx_{n-g}$ after seeing $n$ examples, for different values of the time stamp gap $g$ (see \cref{ssec:forgetting})}, and find both Transformers and not-converged Newton have better memorization than LSTM and OGD.
    \label{fig:lstm-experiment}
    \end{minipage}
\end{figure*}


\vspace{-1mm}
\subsection{Transformers perform well on ill-conditioned data}  \label{ssec:ill}
\vspace{-1mm}

\vspace{-1mm}

We repeat the same experiments with data $\bx_i \iid \mathcal N(\zero, \bSigma)$ sampled from an ill-condition covariance matrix $\bSigma$ with condition number $\kappa(\bSigma) = 100$, and eigenbasis chosen uniformly at random.  The first $d/2$ eigenvalues of $\bSigma$  are 100, and the last $d/2$ are 1. Note that choosing the eigenbasis uniformly at random for \emph{each} sequence ensures that there is a different covariance matrix $\bSigma$ for {each} sequence of datapoints.

As shown in \Cref{fig:ill}, the Transformer model's performance still closely matches \newton's Method with 21 iterations, same as when $\bSigma = \bI$ (see layer 10-12 in \Cref{fig:heatmap}). The convergence of second-order methods has a mild logarithmic dependence on the condition number since they correct for the curvature. On the other hand, GD's convergence is affected polynomially by conditioning.  
As $\kappa(\bSigma)$ increase from 1 to 100, the number steps required for GD's convergence increases significantly (see Fig.~\ref{fig:ill} where GD requires 2,000 steps to converge), making it impossible for a 12-layer Transformers to implement these many gradient updates.  We also note that preconditioning the data by $(\bX^\top \bX)^\dagger$ can make the data well-conditioned, but since the eigenbasis is chosen uniformly at random, with high probability there is no sparse pre-conditioner or any fixed pre-conditioner which works across the data distribution. Computing $(\bX^\top \bX)^\dagger$ appears to be as hard as computing the OLS solution (Eq. \ref{eqn:obj})---in fact \citet{sharan2019memory} conjecture that first-order methods such as gradient descent and its variants cannot avoid polynomial dependencies in condition number in the ill-conditioned case.\footnote{Regarding preconditioning, we also note that---even for well-conditioned instances---preconditioned GD still gets a linear rate of convergence, whereas Transformers and \newton\ get superlinear rates.} 
See \Cref{app:ill} for detailed experiments on ill-conditioned problems. These experiments further strengthen our thesis that Transformers learn to perform second-order optimization methods in-context, not first-order methods such as GD. 

\subsection{Transformers Require $\mathcal O(d)$ Hidden Dimension}

We ablate 12-layer 1-head Transformers with various hidden sizes on $d=20$ problems. As shown in Figure \ref{fig:hiddien-dimension}, we observe that Transformers can mimic OLS solution when the hidden size is 32 or 64, but fail with smaller sizes. This resonates with our theoretical results on $\mathcal O(d)$ hidden dimension in Theorem \ref{thm:transformers_newton}, and in this case, the theorem ensures a construction of transformers to implement Iterative Newton's method. 
\begin{figure}[t]
    \centering
    \includegraphics[width=0.75\linewidth]{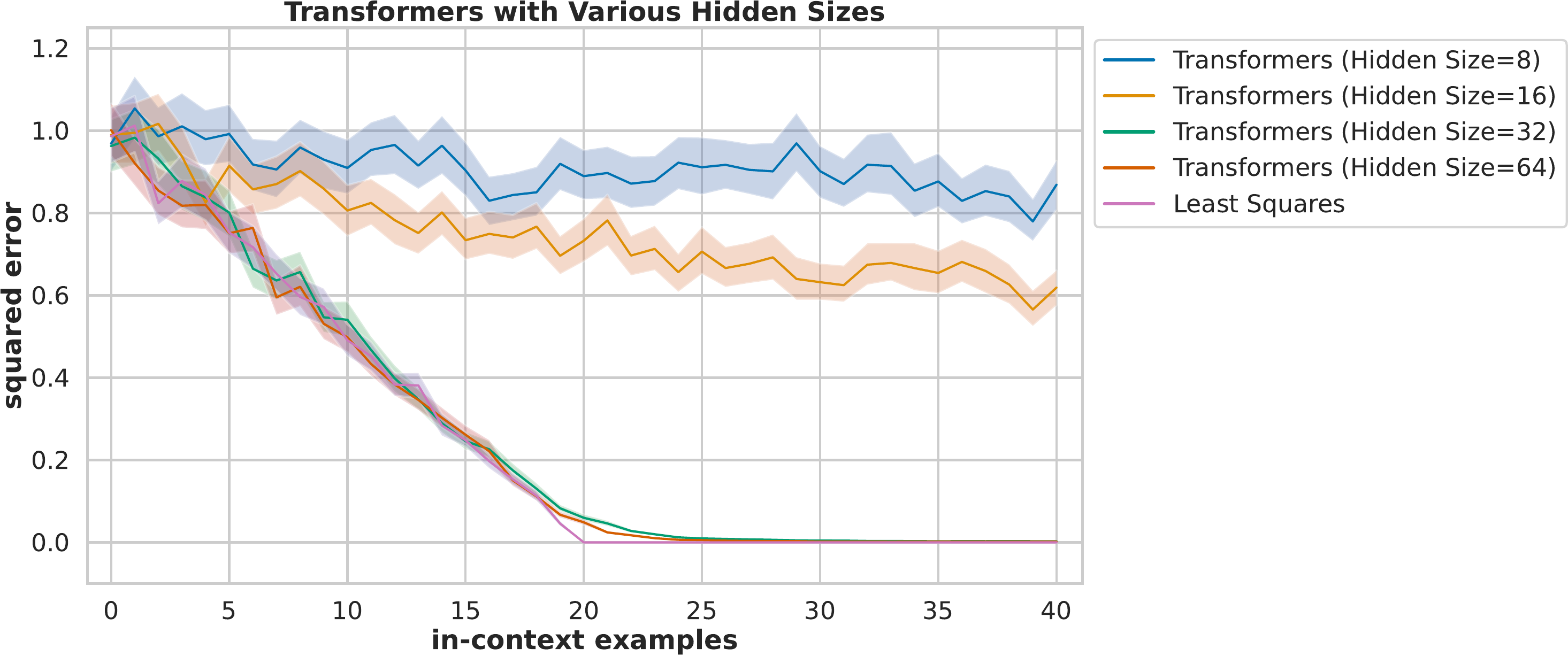}
    \caption{Ablation on Transformer's Hidden Size. For linear regression problems with $d=20$, Transformers need $\mathcal O(d)$ hidden dimension to mimic OLS solutions.}
    \label{fig:hiddien-dimension}
\end{figure}

\vspace{-1mm}
\subsection{LSTM is more similar to OGD than Transformers} \label{ssec:lstm}
\vspace{-1mm}

As discussed in \S\ref{app:lstm}, LSTM is an alternative auto-regressive model widely used before the introduction of Transformers. 
Thus, a natural research question is: \textit{If Transformers can learn in-context, can LSTMs do so as well? If so, do they learn the same algorithms?} To answer this question, we train a LSTM model in an identical manner to the Transformers studied in the previous sections.

\Cref{fig:lstm-experiment} plots the error of Transformers, LSTMs, and other standard methods as a function of the number of in-context (i.e., training) examples provided.
While LSTMs can also learn linear regression in-context, they have much higher mean-squared error than Transformers.
Their error rate is similar to Iterative Newton's Method after only 5 iterations, a point where it is far from converging to the OLS solution.

Finally, we show that LSTMs behave more like an online learning algorithm than Transformers.
In particular, its predictions are biased towards getting more recent training examples correct, as opposed to earlier examples, as shown in \Cref{fig:lstm-experiment}.
This property makes LSTMs similar to online GD.
In contrast, five steps of Newton's method has the same error on average for recent and early examples, showing that the LSTM implements a very different algorithm from a few iterations of Newton.

We hypothesize that since LSTMs have limited memory, they must learn in a roughly online fashion;
in contrast, Transformer's attention heads can access the entire sequence of past examples, enabling it to learn more complex algorithms. See \S \ref{app:lstm} for more discussions.

\vspace{-1mm}
\section{Theoretical Justification}
\label{sec:mechanism}
\vspace{-1mm}

Our empirical evidence demonstrates that Transformers behave much more similarly to \newton's than to GD. \newton\ is a second-order optimization method, and is algorithmically more involved than GD. We begin by first examining this difference in complexity. As discussed in Section \ref{sec:problem_definition}, the updates for \newton\ are of the form,
\begin{align}
   &\hat{\bw}^\mathrm{Newton}_{k+1} =  \bM_{k+1}  \bX^\top \by  \qquad \text{where } \bM_{k+1} = 2 \bM_{k} - \bM_k \bS \bM_k 
\end{align}
and $\bM_{0}=\alpha \bS $ for some $\alpha>0.$ We can express $\bM_k$ in terms of powers of $\bS$ by expanding iteratively, for example $\bM_{1}=2 \alpha \bS - 4 \alpha^2 \bS^3, \bM_{2}= 4 \alpha \bS - 12 \alpha^2 \bS^3  + 16 \alpha^3 \bS^5 - 16\alpha^4 \bS^7$, and in general $\bM_k = \sum_{s=1}^{2^{k+1} - 1} \beta_s \bS^s$ for some $\beta_s\in \R$ (see Appendix \ref{sec:newton_moment} for detailed calculations). Note that $k$ steps of \newton's requires computing $\Omega(2^k)$ moments of $\bS$. Let us contrast this with GD. GD updates for linear regression take the  form,
\begin{align}
\hat{\bw}^\mathrm{GD}_{k+1} = \hat{\bw}^\mathrm{GD}_{k} - \eta (\bS \hat{\bw}^\mathrm{GD}_{k} -\bX^\top \by).
\end{align}
Like \newton, we can express $\hat{\bw}^\mathrm{GD}_{k}$  in terms of powers of $\bS$ and $\bX^\top \by$. However, after $k$ steps of GD, the highest power of $\bS$ is only $O(k)$. This exponential separation is consistent with the exponential gap in terms of the parameter dependence in the convergence rate---$\mathcal O\left(\kappa(\bS) \log(1/\epsilon)\right)$ for GD vs. $\mathcal O(\log \kappa(\bS) + \log \log (1/\epsilon))$ for \newton. Therefore, a natural question is whether Transformers can actually as complicated of a method such as \newton\ with only polynomially many layers? Theorem \ref{thm:transformers_newton} shows that this is indeed possible. 
    
\begin{restatable}{theorem}{TransformersNewton}
\label{thm:transformers_newton}
For any $k$, there exist Transformer weights such that on any set of in-context examples  $\{\bx_i, y_i\}_{i=1}^n$ and test point  $\bx_\mathrm{test}$, the Transformer predicts on $\bx_\mathrm{test}$ using $\bx_\mathrm{test}^\top \hat{\bw}^\mathrm{Newton}_{k}$. Here $\hat{\bw}^\mathrm{Newton}_{k}$ are the \newton\ updates given by $ \hat{\bw}^\mathrm{Newton}_{k} =  \bM_{k}  \bX^\top \by$ where $\bM_{j}$ is updated as 
\begin{align*}
    \bM_{j} = 2 \bM_{j-1} - \bM_{j-1} \bS \bM_{j-1}, 1\le j\le k, \quad  \bM_{0}=\alpha \bS, 
\end{align*}
for some $\alpha>0$ and  $\bS = \bX^\top \bX$. The dimensionality of the hidden layers is $\mathcal{O}(d)$, and the number of layers is $k+8$. One transformer layer computes one Newton iteration. 3 initial transformer layers are needed for initializing $\bM_0$ and 5 layers at the end are needed to read out predictions from the computed pseudo-inverse $\bM_k$.
\end{restatable}

Here we provide a skecth of the proof. 
We note that our proof uses full attention instead of causal attention and ReLU activations for the self-attention layers. The definitions of these and the full proof appear in Appendix \ref{app:mechanism}.

\subsection{Proof Sketch for Theorem \ref{thm:transformers_newton}}
The constructive proof leverages some key operators which \cite{Akyrek2022WhatLA} showed  a single Transformers layer can implement. We summarize these in Proposition \ref{prop:akyrek} . We mainly use the $\mathrm{mov}$ operator, which can copy parts of the hidden state from time stamp $i$ to time stamp $j$ for any $j \geq i$; the $\mathrm{mul}$ operator which can do multiplications; and the $\mathrm{aff}$ operator, which can be used to do addition and subtraction. 
\paragraph{Transformers Implement Initialization $\bT^{(0)} = \alpha \bS$.} Given input sequence $\bH = \{\bx_1, \cdots, \bx_n \}$, we can first use the $\mathrm{mov}$ operator from Proposition \ref{prop:akyrek} 
so that the input sequence becomes $\begin{bmatrix}
    \bx_1 & \cdots & \bx_n \\
    \bx_1 & \cdots & \bx_n 
\end{bmatrix}$. We call each column $\bh_j$. With an full attention layer and normalized ReLU activations, one can construct two heads with query and value matrices of the form $\bQ_1^\top \bK_1 = -\bQ_2^\top \bK_2 = \begin{bmatrix}
    \bI_{d \times d} & \bO_{d \times d} \\
    \bO_{d \times d} & \bO_{d \times d}
\end{bmatrix}$ and value matrices $\bV_m =n \alpha \begin{bmatrix}
    \bI_{d \times d} & \bO_{d \times d} \\
    \bO_{d \times d} & \bO_{d \times d}
\end{bmatrix}$  for some $\alpha \in \mathbb R$.\footnote{The value matrices contain the number of in-context examples $n$. There are ways to avoid this by applying Proposition \ref{prop:akyrek} to count.} Combining the attention layer and skip connections we end up with $\bh_t \leftarrow \begin{bmatrix}
        \bx_t + \alpha \bS \bx_t \\ \bx_t
\end{bmatrix} $. Applying the $\mathrm{aff}$ operator from Proposition \ref{prop:akyrek} to do subtraction, we can have each column of the form $\begin{bmatrix}
        \alpha \bS \bx_t \\ \bx_t
\end{bmatrix}$. We denote $\bT^{(0)} :=  \alpha \bS$ so that Transformers and \newton\ have the similar initialization and we call these columns $\bh_t^{(0)}$.  

\paragraph{Transformers implement Newton Iteration.}  We claim that we can construct  layer $\ell$'s hidden states to be of the form 
    \begin{equation}
        \bH^{(\ell)} = \begin{bmatrix}
            \bh_1^{(\ell)} & \cdots & \bh_n^{(\ell)}
        \end{bmatrix} = \begin{bmatrix}
            \bT^{(\ell)} \bx_1 & \cdots & \bT^{(\ell)} \bx_n \\
            \bx_1 & \cdots & \bx_n
        \end{bmatrix} 
    \end{equation}
    We prove by induction that assuming our claim is true for $\ell$, we work on $\ell + 1$:
Let $\bQ_m = \tilde{\bQ}_m \begin{bmatrix}
        \bO_d & -\frac{n}{2} \bI_d \\
        \bO_d & \bO_d 
    \end{bmatrix} , \bK_m = \tilde{\bK}_m \begin{bmatrix}
        \bI_d & \bO_d \\
        \bO_d & \bO_d 
    \end{bmatrix}$ where $\tilde{\bQ}_1^\top \tilde{\bK}_1 : = \bI$, $\tilde{\bQ}_2^\top \tilde{\bK}_2 : = -\bI$ and $\bV_1 = \bV_2 = \begin{bmatrix}
        \bI_d & \bO_d \\
        \bO_d & \bO_d 
    \end{bmatrix}$. A 2-head self-attention layer gives 
\begin{equation}
    \bh_t^{(\ell+1)} = \begin{bmatrix}
        \left(\bT^{(\ell)} - \frac{1}{2} \bT^{(\ell)} \bS {\bT^{(\ell)}}^\top \right) \bx_t \\ \bx_t 
    \end{bmatrix}
\end{equation}
Note that all $\bT^{(\ell)}$ are symmetric. Passing over an MLP layer gives
    \begin{equation}
        \bh_t^{(\ell+1)} \leftarrow \bh_t^{(\ell+1)} + \begin{bmatrix}
            \bI_d & \bO_d \\ 
            \bO_d & \bO_d 
        \end{bmatrix} \bh_t^{(\ell+1)} = \begin{bmatrix}
        \left(2\bT^{(\ell)} -  \bT^{(\ell)} \bS {\bT^{(\ell)}} \right) \bx_t \\ \bx_t 
    \end{bmatrix}
    \end{equation}
We denote $\bT^{(\ell+1)} := 2\bT^{(\ell)} -  \bT^{(\ell)} \bS {\bT^{(\ell)}}$ and this is the exactly same form as \newton\ updates. 

\paragraph{Transformers can implement $\hat{\bw}_\ell^\mathrm{TF} = \bT^{(\ell)} \bX^\top \by$.}
We insert columns with $[0, 0, \cdots, y_j]^\top$ after each $\bx_j$ (See \Cref{fig:transformer_illustration} for illustration) and keep them unchanged until reaching layer $\ell$. Applying $\mathrm{mov}$ and $\mathrm{mul}$, we have columns $\begin{bmatrix}
    \bxi \\ \bT^{(\ell)} y_j \bx_j
\end{bmatrix}$ where $\bxi$ are irrelevant quantities. Apply \Cref{lemma:sum} for summation, we can gather $\sum_{j=1}^n \bT^{(\ell)} y_j \bx_j = \bT^{(\ell)} \bX^\top \by$, which is again the same as \newton\ and we call this $\hat{\bw}_\ell^\mathrm{TF}$. 

\paragraph{Transformers can make predictions on $\bx_{test}$ by $\inner{\hat{\bw}_\ell^\mathrm{TF}, \bx_\mathrm{test}}$.} 

Now we can make predictions on text query $\bx_\mathrm{test}$:
\begin{equation}
    \begin{bmatrix}
        \bxi &  \bx_\mathrm{test} \\
        \hat{\bw}_\ell^\mathrm{TF} & \bx_\mathrm{test}
    \end{bmatrix} \overset{\mathrm{mov}}{\longrightarrow} \begin{bmatrix}
        \bxi &  \bx_\mathrm{test} \\
        \hat{\bw}_\ell^\mathrm{TF} & \bx_\mathrm{test} \\
        \zero & \hat{\bw}_\ell^\mathrm{TF}
    \end{bmatrix} \overset{\mathrm{mul}}{\longrightarrow} \begin{bmatrix}
        \bxi &  \bx_\mathrm{test} \\
        \hat{\bw}_\ell^\mathrm{TF} & \bx_\mathrm{test} \\
        \zero & \hat{\bw}_\ell^\mathrm{TF} \\
        0 & \inner{\hat{\bw}_\ell^\mathrm{TF}, \bx_\mathrm{test}}
    \end{bmatrix}
\end{equation}
A final readout layer can extract the prediction $\inner{\hat{\bw}_\ell^\mathrm{TF}, \bx_\mathrm{test}}$.

Now we complete the proof that Transformers can perform exactly \newton. Finally, we count the number of layers and the dimension of hidden states. We can see all operations in the proof require a linear amount of hidden state dimensions, which are $\mathcal O(d)$. There are 3 Transformer layers needed to compute the Newton initialization and 5 layers needed for reading out predictions. Operations from Transformers index $\ell$ to $\ell+1$ require 1 Transformer layer. Hence, to perform $k$ \newton\ updates, Transformers require $k + 8$ layers. This implies that the rate of convergence of Transformers to solve linear regression in-context is the same as \newton's: $\mathcal O(\log \log (1/\epsilon))$ and this is consistent with our experimental results in \S\ref{sec:experiments}.

\vspace{-1mm}
\section{Conclusion and Discussion}
\vspace{-1mm}

In this work, we studied how Transformers perform in-context learning for linear regression.
In contrast with the hypothesis that Transformers learn in-context by implementing gradient descent,
\new{our experimental results show that different Transformer layers match iterations of \newton\  \textit{linearly} and \gd\ \textit{exponentially}. This suggests that Transformers share a similar rate of convergence to \newton\ but not to \gd.}
Moreover, Transformers can perform well empirically on ill-conditioned linear regression, whereas first-order methods such as \gd\ struggle. {This empirical evidence --- when combined with existing lower bounds in optimization --- suggests that Transformers use second-order information for solving linear regression, and we also prove that Transformers can indeed represent second-order methods. } 

An interesting direction is to explore a wider range of second-order methods that Transformers can implement. It also seems promising to  extend our analysis to classification problems, especially given 
recent work showing that Transformers resemble SVMs in classification tasks \citep{Li2023TransformersAA,AtaeeTarzanagh2023TransformersAS}. 
Finally, a natural question is to understand the differences in the  model architecture that make Transformers better in-context learners than LSTMs. Based on our investigations with LSTMs, we hypothesize that Transformers can implement more powerful algorithms because of having access to a longer history of examples. 
 Investigating the role of this additional memory in learning appears to be an intriguing direction.

\section*{Acknowledgement}
We would like to thank the USC NLP Group and Center for AI Safety for providing compute resources. DF would like to thank Oliver Liu and Ameya Godbole for their extensive discussions. 
DF and RJ were supported by a Google Research Scholar Award.
RJ was also supported by an Open Philanthropy research grant.
VS was supported by NSF CAREER Award CCF-2239265 and an Amazon Research Award.

\bibliography{Reference}
\bibliographystyle{plainnat}

\clearpage
\onecolumn
\section*{Appendix}
\appendix
\startcontents[appendix]
\addcontentsline{toc}{chapter}{Appendix}
\renewcommand{\thesection}{\Alph{section}} 

\printcontents[appendix]{}{1}{\setcounter{tocdepth}{3}}

\setcounter{section}{0}
\section{Additional Experimental Results} \label{app:add_exp}

\subsection{Contrast with LSTMs} \label{app:lstm}

While our primary goal is to analyze Transformers, we also consider LSTMs \citep{hochreiter1997lstm} to understand whether Transformers learn different algorithms than other neural sequence models trained to do linear regression. In particular, we train a unidirectional $L$-layer LSTM, which generates a sequence of hidden states $\bH^{(\ell)}$ for each layer $\ell$, similarly to an $L$-layer Transformer.
As with Transformers, we add a readout layer that predicts the $\hat{y}_{t+1}^{\mathrm{LSTM}}$ from the final hidden state at the final layer, $\bH_{:,2t+1}^{(L)}$.

\begin{table}[!htp]
    \centering
    \begin{tabular}{c|c|c}
        & Transformers &  LSTM \\
         \midrule
     Newton & \textbf{0.991} & 0.920  \\
     GD & \textbf{0.957} & 0.916 \\
     OGD   & 0.806 & \textbf{0.954}  
    \end{tabular}
    \caption{\textbf{Similarity of errors between algorithms.} Transformers are more similar to full-observation methods such as Newton and GD; and LSTMs are more similar to online methods such as OGD.}
    \vspace{-1ex}
    \label{tab:lstm}
\end{table}

We train a 10-layer LSTM model, with 5.3M parameters, in an identical manner to the Transformers (with 9.5M parameters) studied in the previous sections.\footnote{While the LSTM has fewer parameters than the Transformer, we found in preliminary experiments that increasing the size of the LSTM would not substantively change our results.}

LSTMs' inferior performance to Transformers can be explained by the inability of LSTMs to use deeper layers to improve their predictions.
\Cref{fig:progression} shows that LSTM performance does not improve across layers---a readout head fine-tuned for the first layer makes equally good predictions as the full 10-layer model. 
Thus, LSTMs seem poorly equipped to fully implement iterative algorithms.
Similarly, Table \ref{tab:lstm} shows that LSTMs are more similar to OGD than Transformers are, whereas Transformers are more similar to Newton and GD than LSTMs.

\subsection{Additional Results on Isotropic Data without Noise} \label{app:exp_standard}
\subsubsection{Progression of Algorithms}
\begin{figure*}[!htp]
    \centering
    \hfill
    \subfigure[Transformers]{\includegraphics[width=0.31\linewidth]{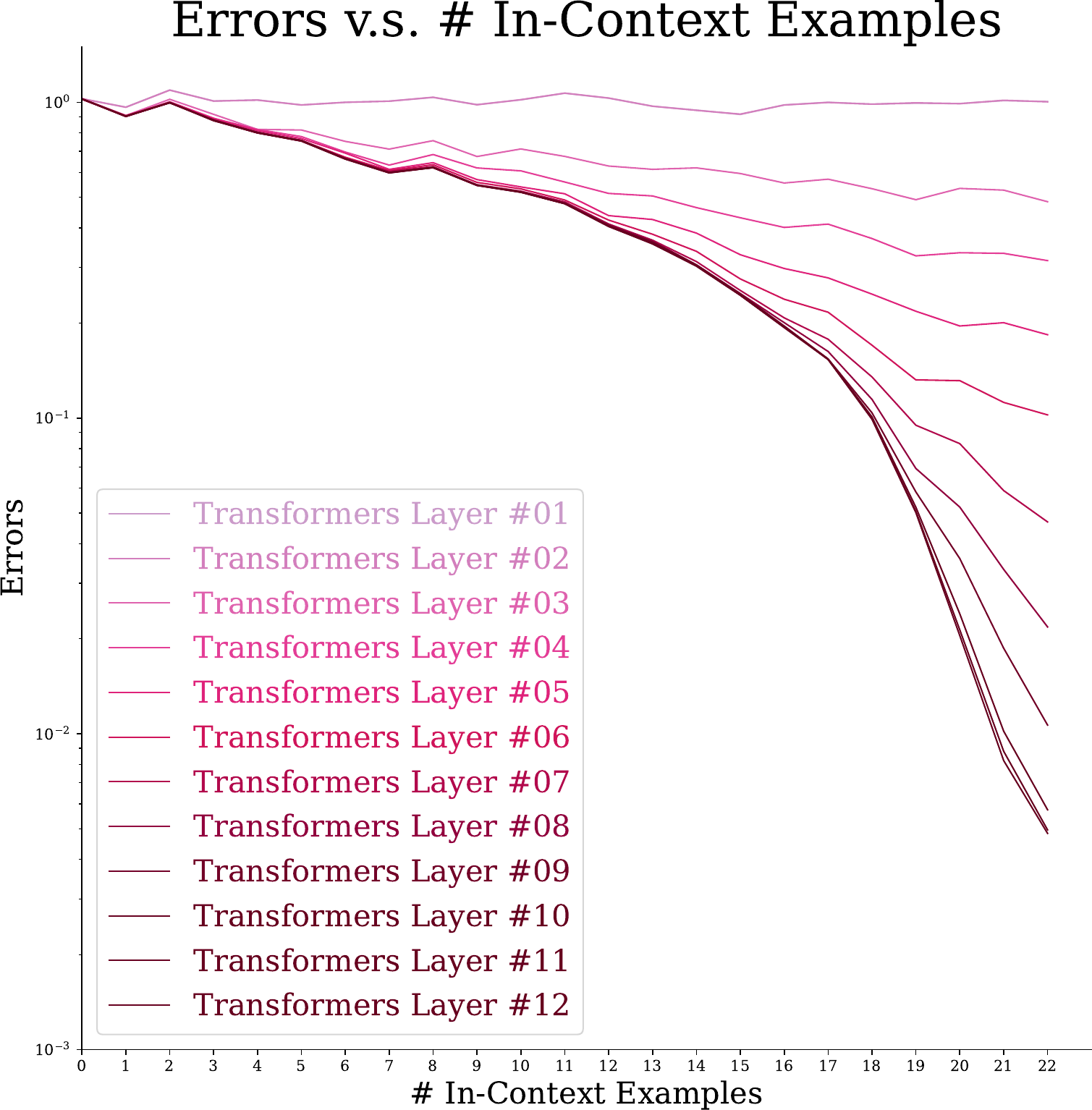} \label{subfig:transformers} } 
    \subfigure[Iterative Newton's Method]{\includegraphics[width=0.31\linewidth]{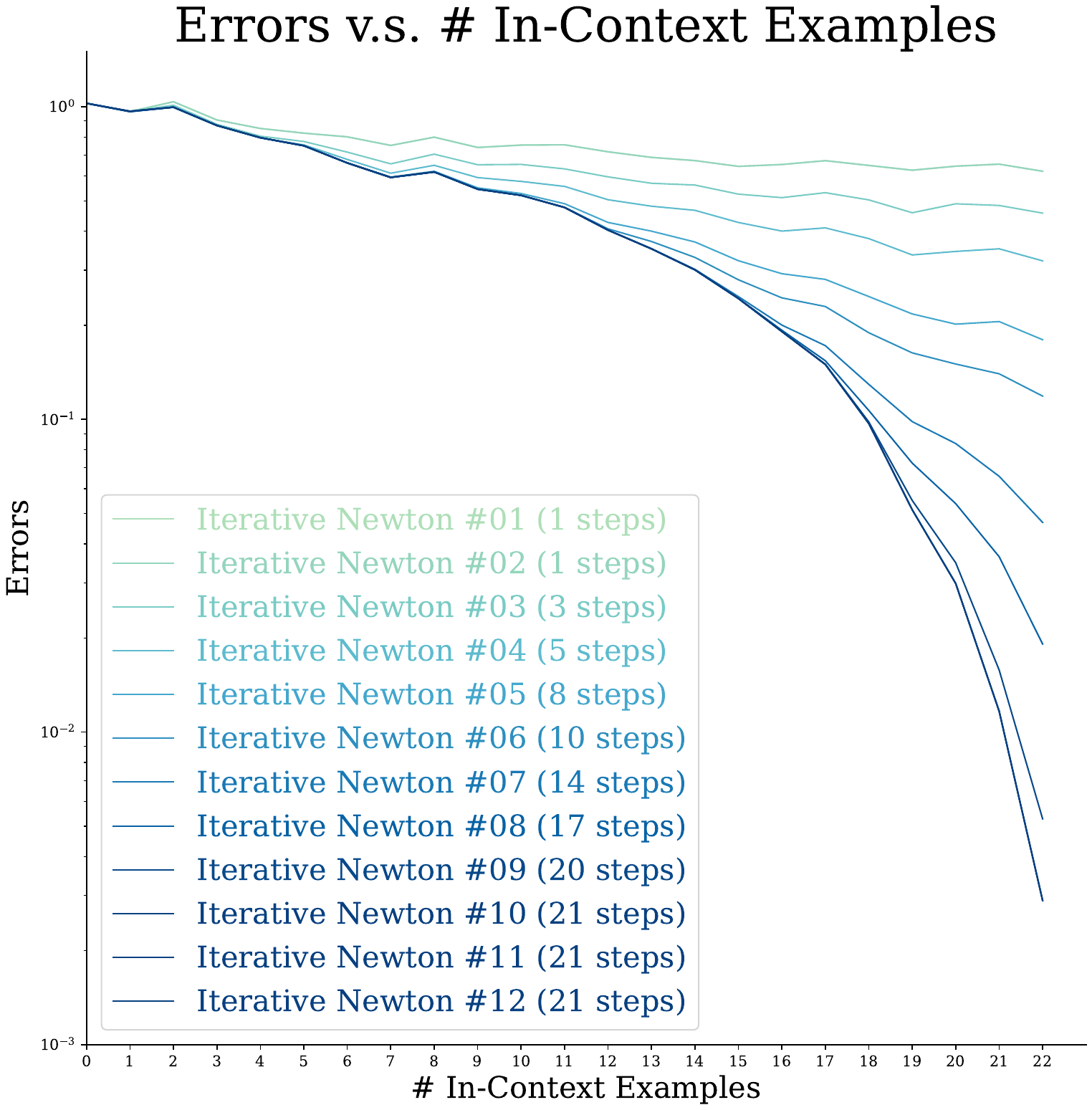} \label{subfig:newton} } 
    \subfigure[LSTM]{\includegraphics[width=0.31\linewidth]{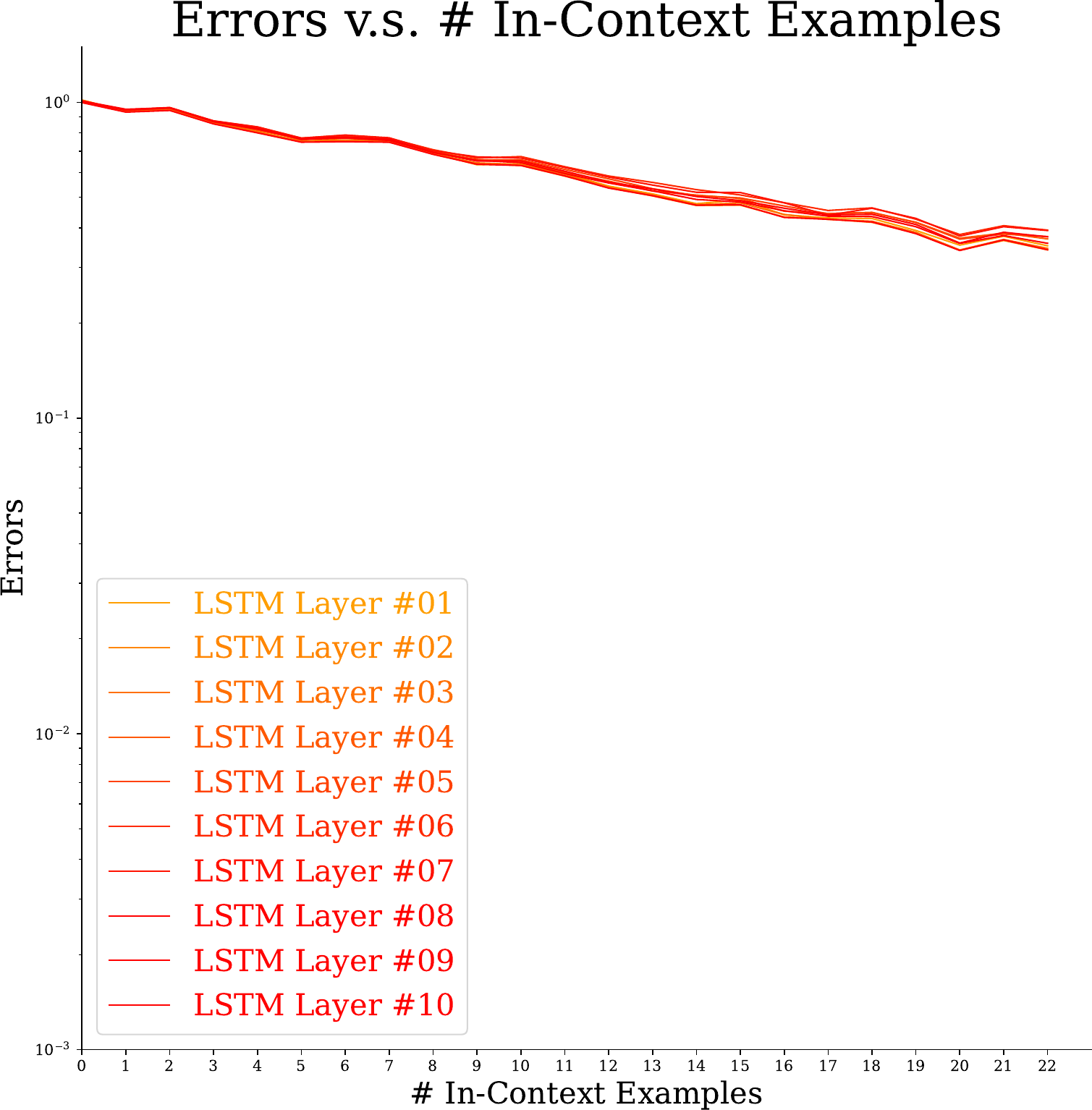} \label{subfig:lstm}}  \hfill
    \caption{\textbf{Progression of Algorithms.} (a) Transformer's performance improves over the layer index $\ell$. (b) \newton's performance improves over the number of iterations $k$, in a way that closely resembles the Transformer. We plot the best-matching $k$ to Transformer's $\ell$ following Definition \ref{def:matching}. (c) In contrast, LSTM's performance does not improve from layer to layer.}
    \vspace{-1.5ex}
    \label{fig:progression}
\end{figure*}
\subsubsection{Heatmaps} 
We present heatmaps with all values of similarities. 
\begin{figure}[!htp]
    \centering
\includegraphics[width=0.49\linewidth]{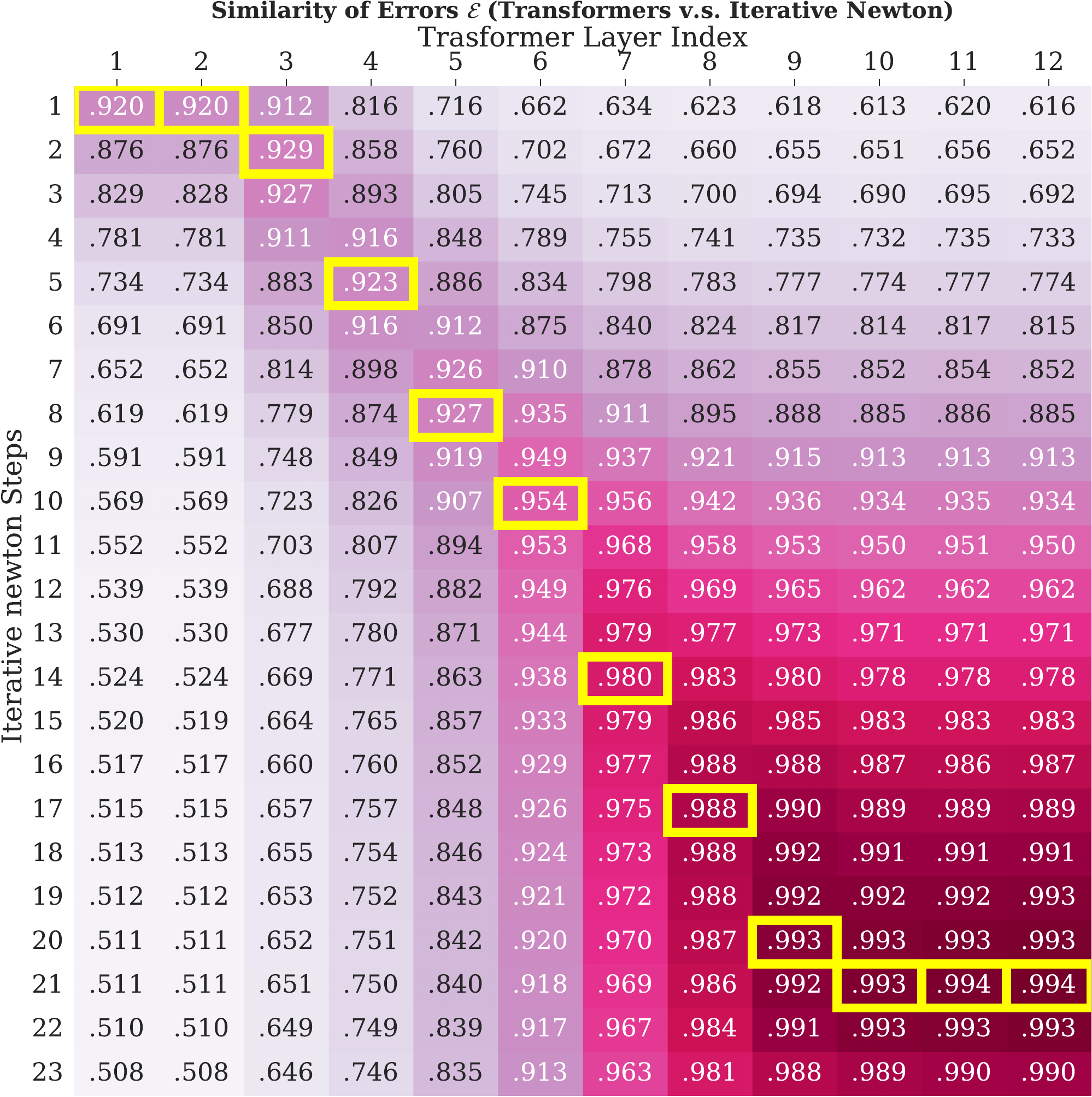}
\includegraphics[width=0.49\linewidth]{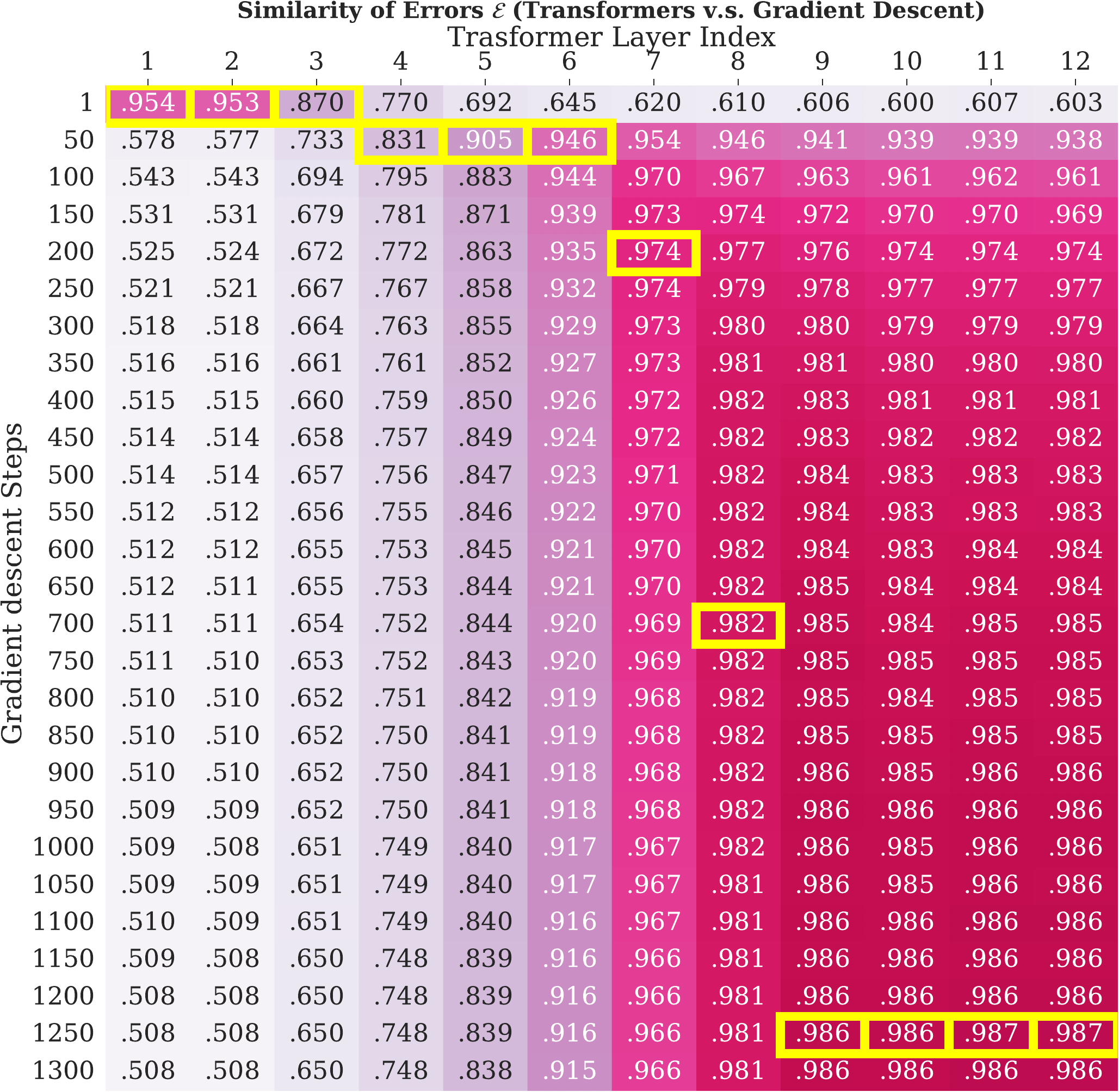}
\caption{\textbf{Similarity of Errors.} The best matching steps are highlighted in yellow.} \label{fig:full_heatmap_sim_e}
\end{figure}
\begin{figure}[!htp]
    \centering
\includegraphics[width=0.49\linewidth]{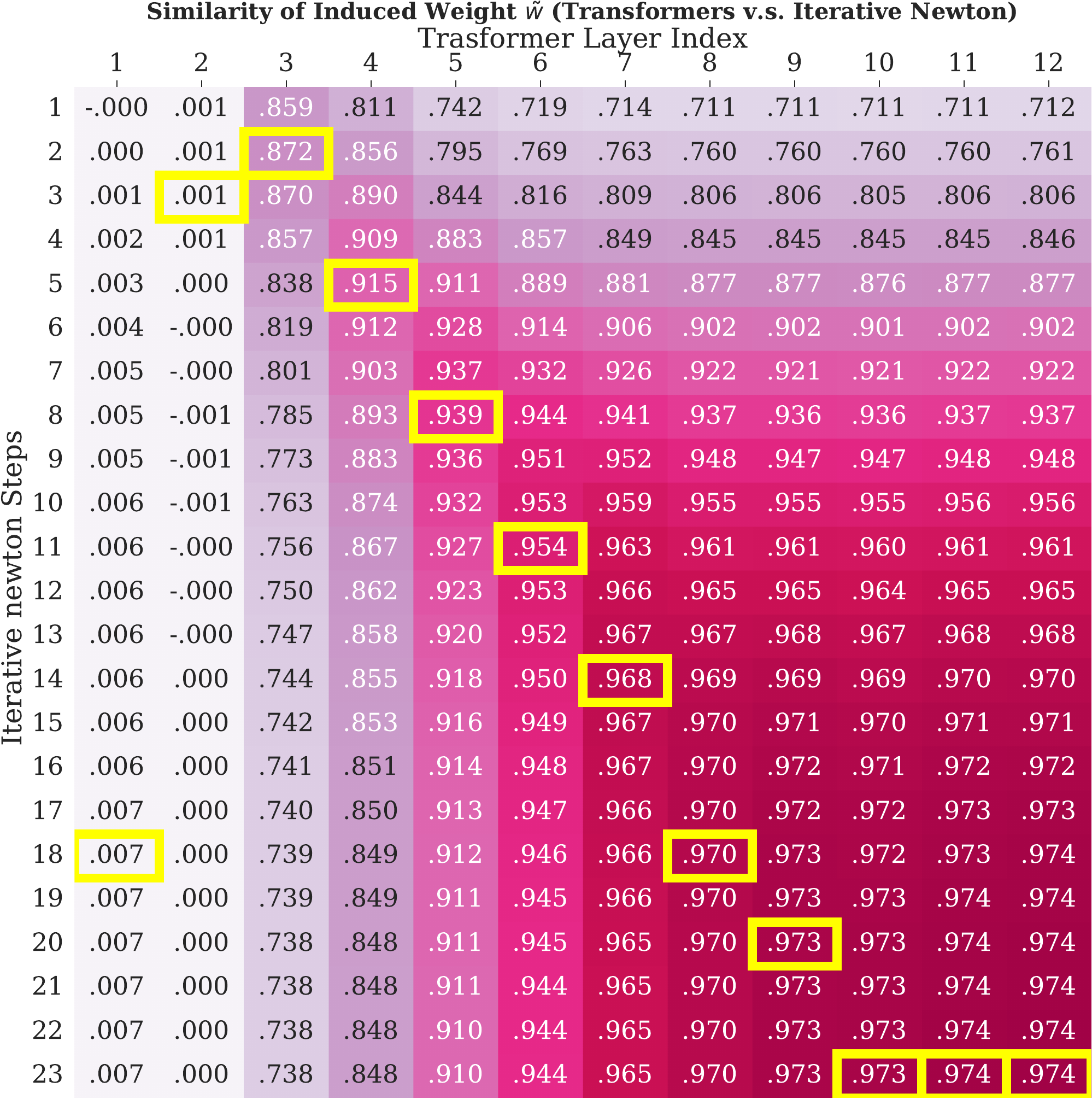}
\includegraphics[width=0.49\linewidth]{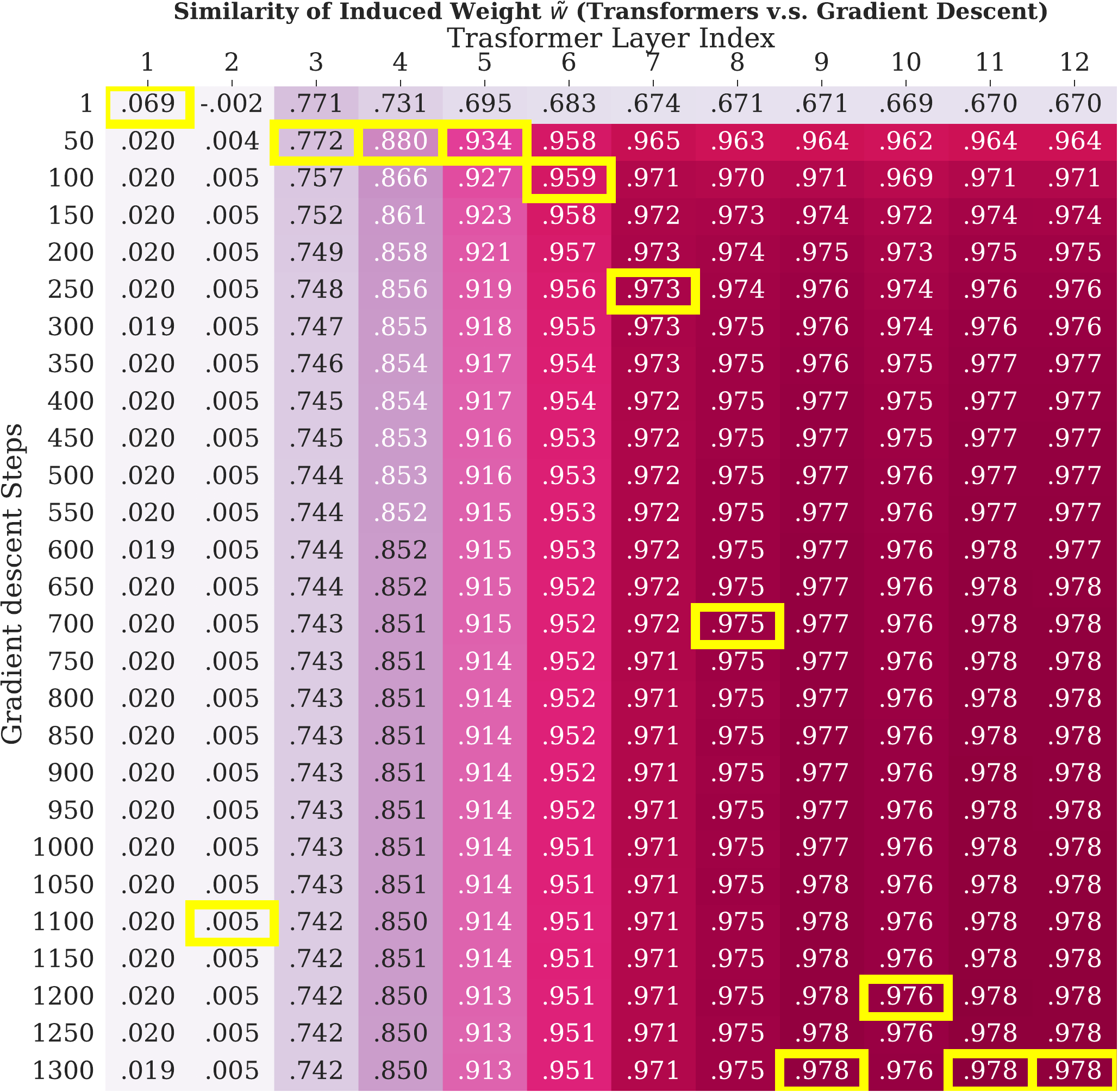}
\caption{\textbf{Similarity of Induced Weight Vectors.} The best matching steps are highlighted in yellow.} \label{fig:full_heatmap_sim_w}
\end{figure}

\begin{figure}[!htp]
    \centering
    \includegraphics[width=0.7\linewidth]{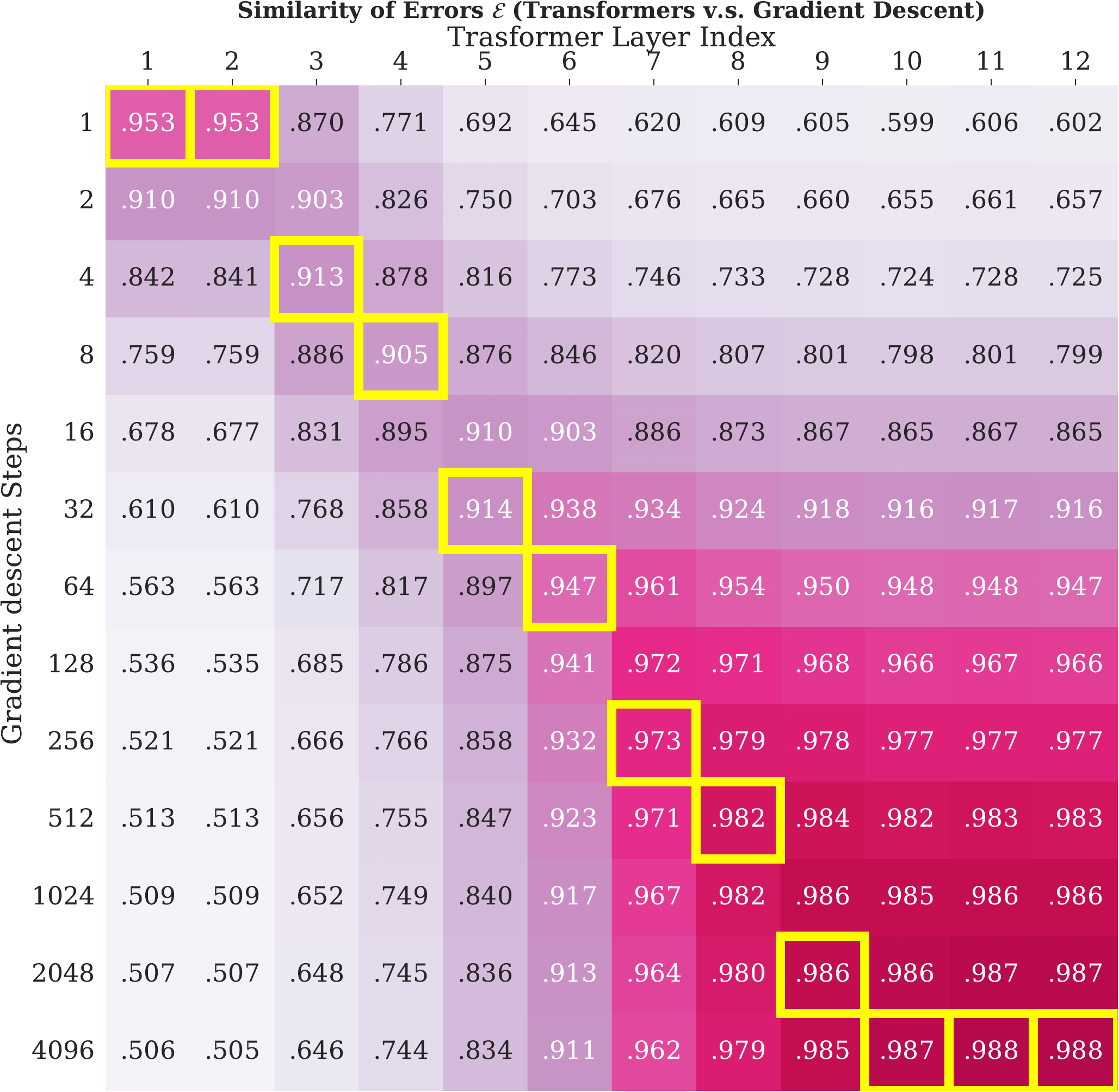}
    \caption{\textbf{Similarity of Errors of Gradient Descent in Log Scale.} The best matching steps are highlighted in yellow. Putting the number of steps of \gd\ in log scale further verifies the claim that Transformer's rate of covergence is exponentially faster than that of \gd.}
    \label{fig:gd_heatmap_log_scale}
\end{figure}

\clearpage
\subsubsection{Comparison with Other Second-Order Methods} \label{sec:second-order}

In this section, we ablate with alternative second-order methods, such as Conjugate Gradient, BFGS, and its limited memory variant, L-BFGS. 

\paragraph{Conjugate Gradient Method. } For linear regression problems, the Conjugate Gradient (CG) method solves the linear system
$$
\underbrace{(\bX^\top \bX)}_{\bS} \bw - \bX^\top \by = 0
$$
CG finds the weight vector $\hat{\bw}^{CG}$ with initialization $\bw_0$ by maintain a set of conjugate gradient $\{\Delta \bw_1, \cdots, \Delta \bw_k\}$. It follows the iterative update rule
\begin{equation}
    \begin{aligned}
        & \bd_k = - \nabla \mathcal L(\bw_k) \\
        & \Delta \bw_k = \bd_k - \sum_{i=0}^{k-1} \frac{\bd_k^\top \bS \Delta \bw_i}{\Delta \bw_i^\top \bS \Delta \bw_i} \Delta \bw_i \\
        & \alpha_k = \arg \min_{\alpha} \mathcal L(\bw_k + \alpha \Delta \bw_k) \\
        &\bw_{k+1} = \bw_k + \alpha_k \Delta \bw_k
    \end{aligned}
\end{equation}

The conjugate Gradient method requires $\mathcal O\left(\sqrt{\kappa} \log(1/\epsilon)\right)$ steps to converge to an $\epsilon$ error on quadratic objectives such as linear regression. 

\paragraph{BFGS. } Broyden– Fletcher–Goldfarb–Shanno (BFGS) is a Quasi-Newton method, designed to approximate the inverse Hessian $\bB_k :\approx \nabla^2 \mathcal L(\bw_k)^{-1}$. The BFGS updates are given by
\begin{equation}
    \bw_{k+1} = \bw_k - \alpha_k \bB_k \nabla \mathcal L(\bw_k)
\end{equation}
where 
\begin{equation*}
    \begin{aligned}
    \bs_k &= \bw_{k+1} - \bw_k \\
    \by_k &= \nabla \mathcal L(\bw_{k+1}) - \nabla \mathcal L(\bw_{k}) \\
    \bB_{k+1} &= \bB_k - \frac{\bB_k \by_k \by_k^\top \bB_k}{\by_k^\top \bB_k \by_k} + \frac{\bs_k \bs_k^\top}{\by_k^\top \bs_k}
    \end{aligned}
\end{equation*}
When $k$ is large, $\bB_k$ approximates the inverse Hessian well.

\paragraph{L(imited-memory)-BFGS.} L-BFGS is a limited-memory version of BFGS. Instead of the inverse Hessian $\bB_k$, L-BFGS maintains a history of past $m$ updates (where $m$ is usually small). Recall the iterative update rule of $\bB_k$ in BFGS
\begin{equation}
    \bB_{k+1} = \bB_k - \frac{\bB_k \by_k \by_k^\top \bB_k}{\by_k^\top \bB_k \by_k} + \frac{\bs_k \bs_k^\top}{\by_k^\top \bs_k}
\end{equation}
Unlike BFGS, which recursively unroll to an initialization $\bB_0$, L-BFGS only unroll to $\bB_{k-m}$ but replacing $\bB_{k-m}$ with $\bB_\mathrm{init}$. In this regard, running $n$ steps of L-BFGS only requires $\mathcal O(mn)$ memory, which is more memory-efficient than BFGS who requires $\mathcal O(n^2)$ memory. The trade-off is that L-BFGS won't have a good estimate of the inverse Hessian when $m < d$, where $d$ is the dimensionality of the quadratic problem. In this regard, it will converge slower than full BFGS. 

In \Cref{fig:bfgs} and \Cref{fig:conjugate-gradient}, we compare Transformers with BFGS, L-BFGS, and Conjugate Gradient method on the metric of similarity of errors. We find that Transformers have a similar \textit{linear} correspondence with BFGS. This is perhaps not surprising, given that BFGS also gets a superlinear convergence rate for linear regression \cite{nocedal1999numerical}. Meanwhile, Transformers show a substantially faster convergence rate than L-BFGS and CG.

\begin{figure}[!htp]
    \centering
    \subfigure[Transformer v.s. BFGS]{\includegraphics[width=0.48\linewidth]{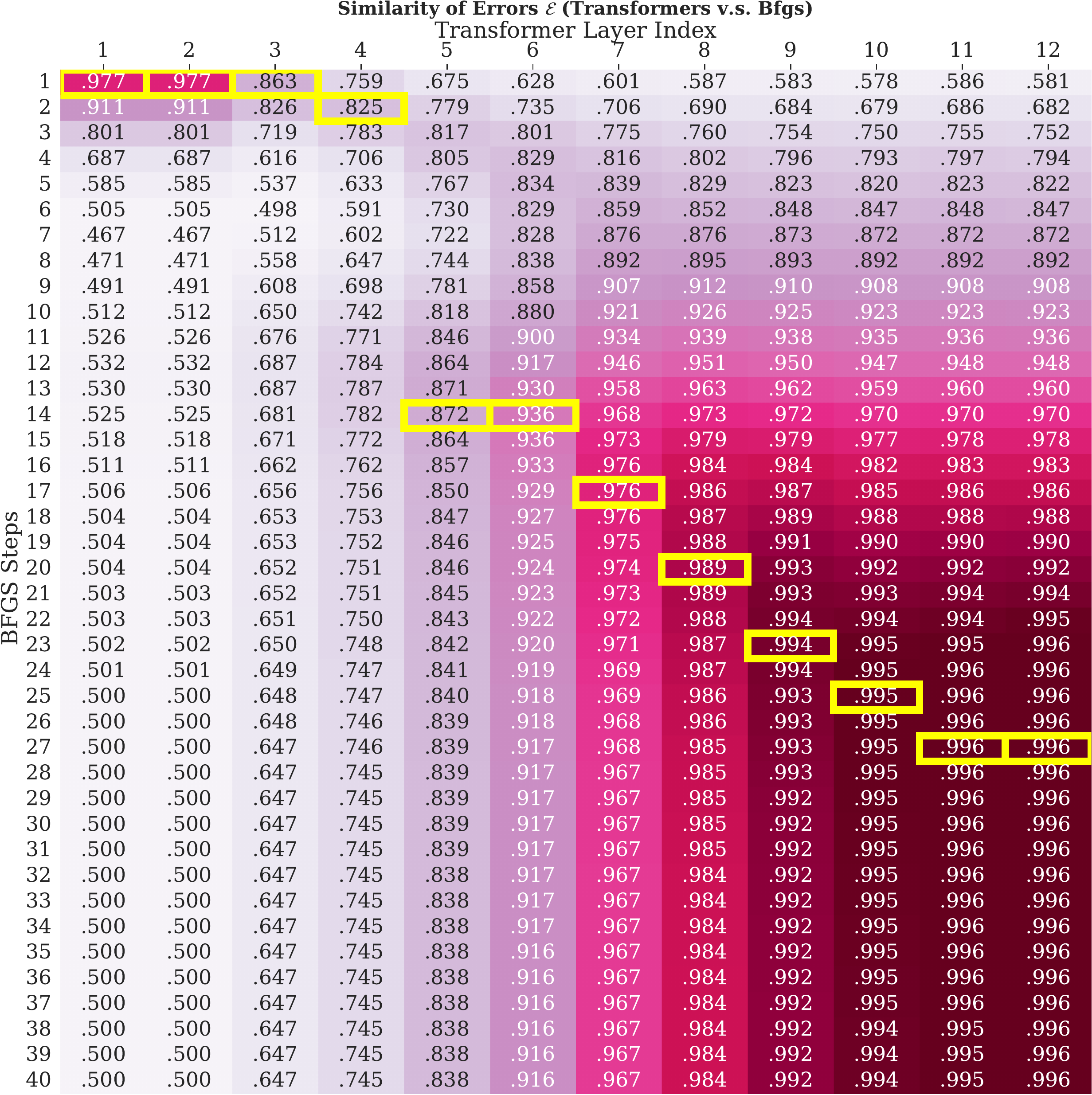}}
    \subfigure[Transformer v.s. L-BFGS]{\includegraphics[width=0.48\linewidth]{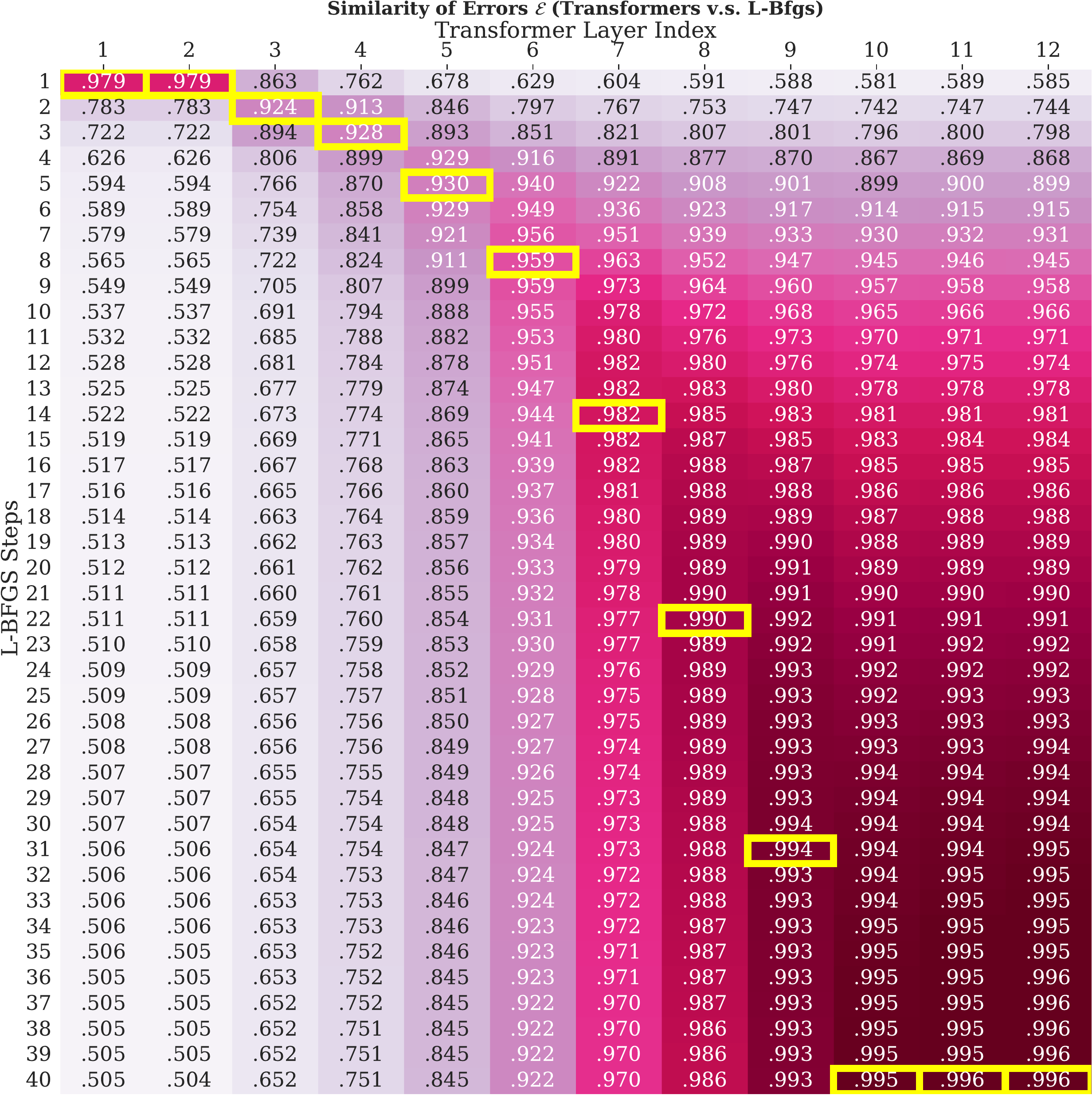}}
    \caption{\textbf{Similarity of Errors between Transformers and BFGS or L-BFGS.} The best matching steps are highlighted in yellow. We find that Transformer, from layers 6 to 11, has a linear correspondence with BFGS. For L-BFGS, due to its limited memory, it approximates second-order information more slowly and results in a slower convergence rate than Transformers.}
    \label{fig:bfgs}
\end{figure}

\begin{figure}[!htp]
    \centering
    \includegraphics[width=0.55\linewidth]{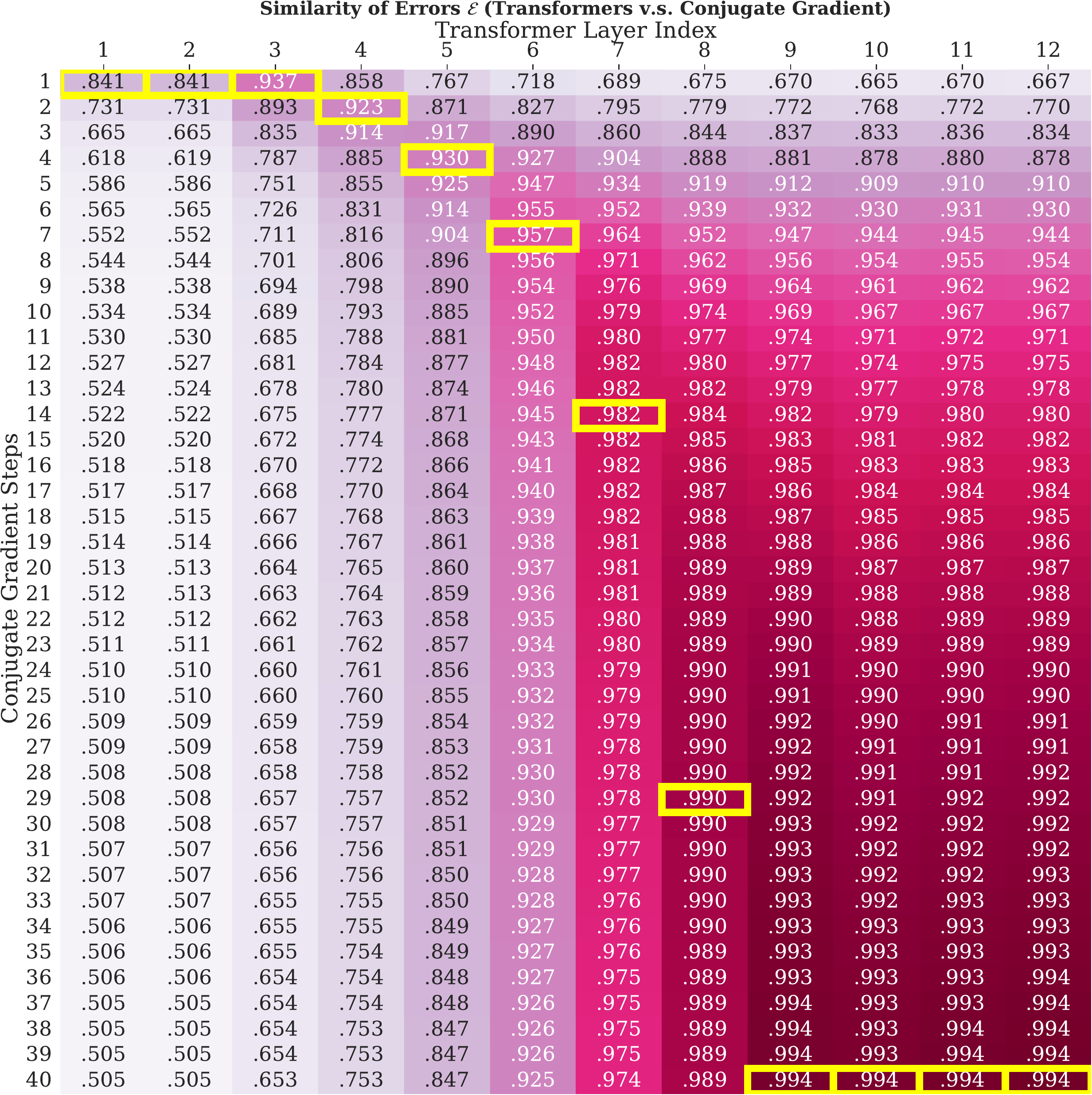}
    \caption{\textbf{Similarity of Errors between Transformers and Conjugate Gradient.} Transformer's convergence rate is still faster than conjugate gradient methods. }
    \label{fig:conjugate-gradient}
\end{figure}

\clearpage
\subsubsection{Additional Results on Comparison over Transformer Layers}
\begin{figure}[!htp]
\centering
\subfigure[Similarity of Errors]{\includegraphics[width=0.4\linewidth]{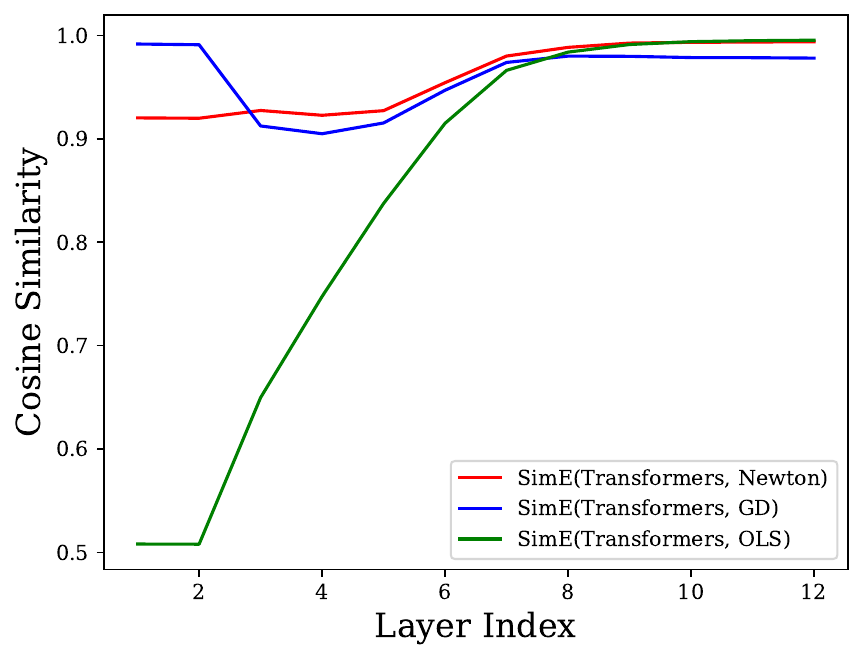}}\label{fig:over_layer_sim_e}
\subfigure[Similarity of Induced Weights]{\includegraphics[width=0.4\linewidth]{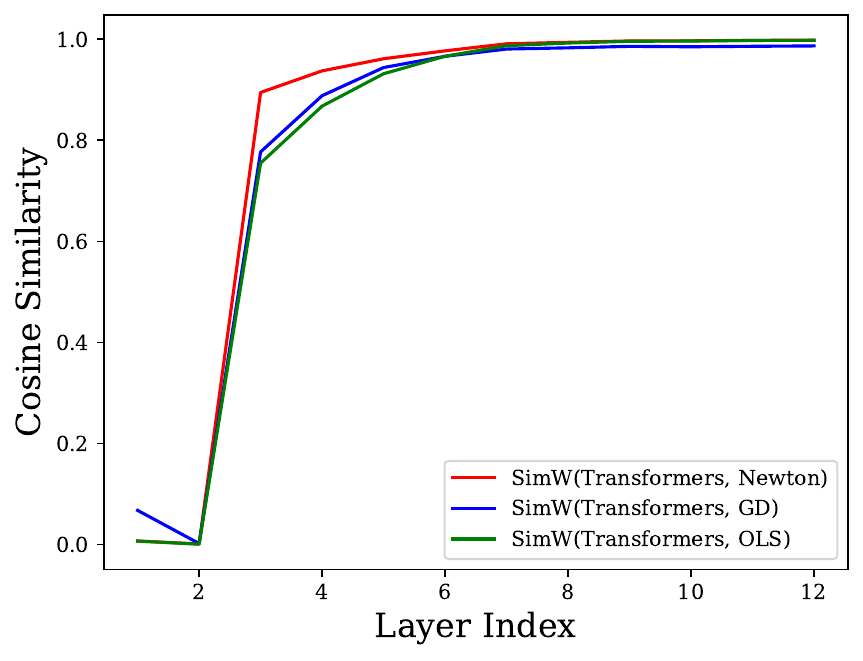}} \label{fig:over_layer_sim_w}
  \caption{Similarities between Transformer and candidate algorithms. Transformers resemble \textit{Iterative Newton's Method} the most.}
  \label{fig:over_layers}
\end{figure}

\subsubsection{Additional Results on Similarity of Induced Weights} \label{app:sim_w}
We present more details line plots for how the similarity of weights changes as the models see more in-context observations $\{\bx_i, y_i\}_{i=1}^n$, i.e., as $n$ increases. We fix the number of Transformers layers $\ell$ and compare with other algorithms with their best-match steps to $\ell$ in \Cref{fig:over_examples}. 

\begin{figure}[hb]
    \centering
    \includegraphics[width=0.32\linewidth]{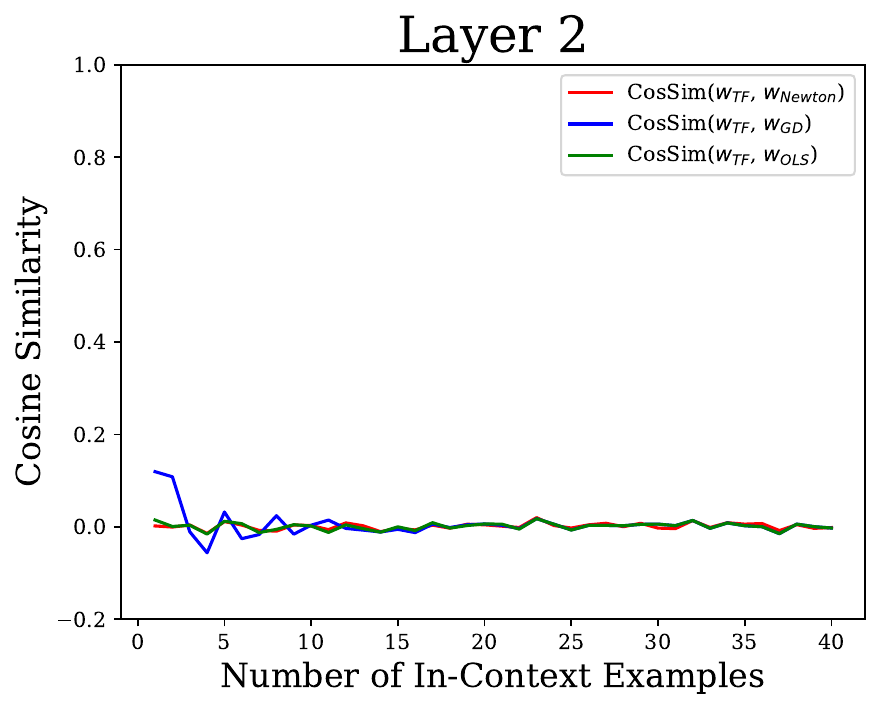}
    \includegraphics[width=0.32\linewidth]{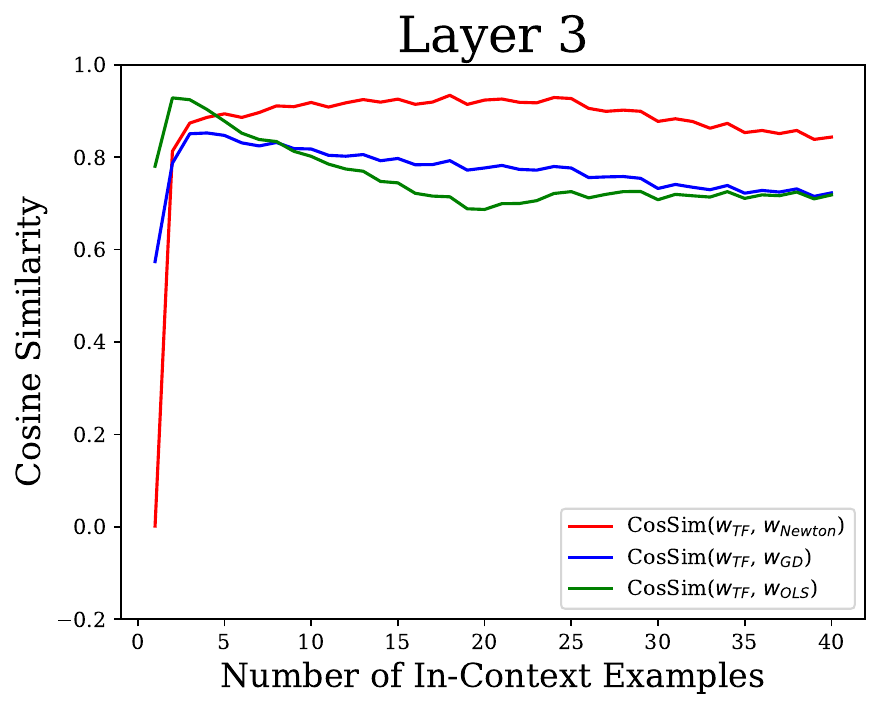}
    \includegraphics[width=0.32\linewidth]{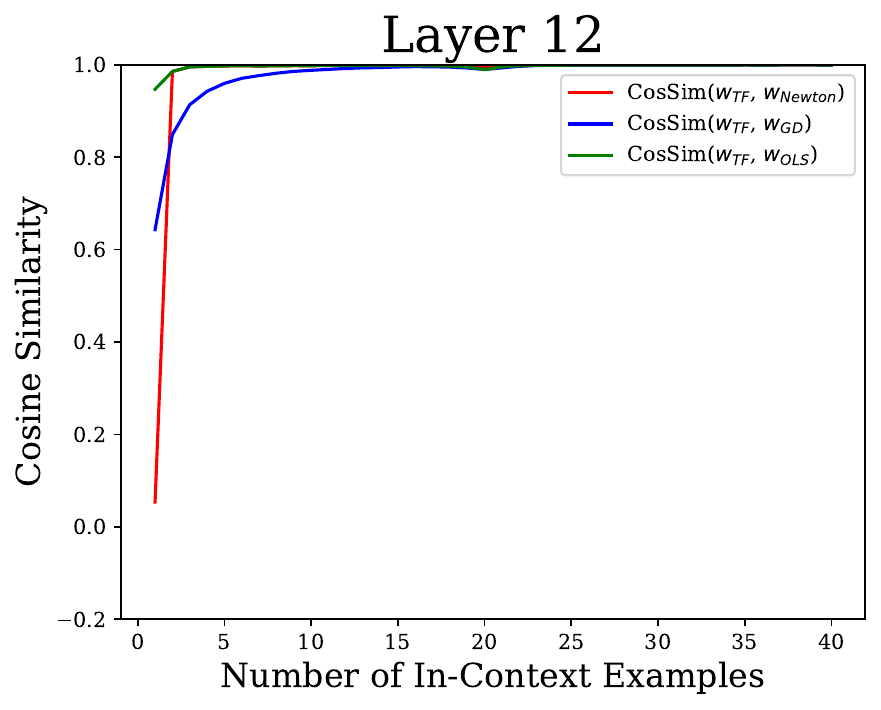}
    \caption{\textit{Similarity of induced weights} over varying number of in-context examples, on three layer indices of Transformers, indexed as 2, 3 and 12. We find that initially at layer 2, the Transformers model hasn't learned so it has zero similarity to all candidate algorithms. As we progress to the next layer number 3, we find that Transformers start to learn, and when provided few examples, Transformers are more similar to OLS but soon become most similar to the Iterative Newton's Method. Layer  12 shows that Transformers in the later layers converge to the OLS solution when provided more than 1 example. We also find there is a dip around $n = d$ for similarity between Transformers and OLS but not for Transformers and Newton, and this is probably because OLS has a more prominent double-descent phenomenon than Transformers and Newton.}
    \label{fig:over_examples}
\end{figure}

\clearpage
\subsection{Varying Data Distribution or Function Class}
\subsubsection{Experiments on Ill-Conditioned Problems}\label{app:ill}
In this section, we repeat the same experiments as we did on isotropic data in the main text and in \Cref{app:exp_standard}, and we change the covariance matrix to be ill-conditioned such that $\kappa(\bSigma) = 100$. 
\begin{figure}[htp]
    \centering
    \subfigure[Transformers]{\includegraphics[width=0.4\linewidth]{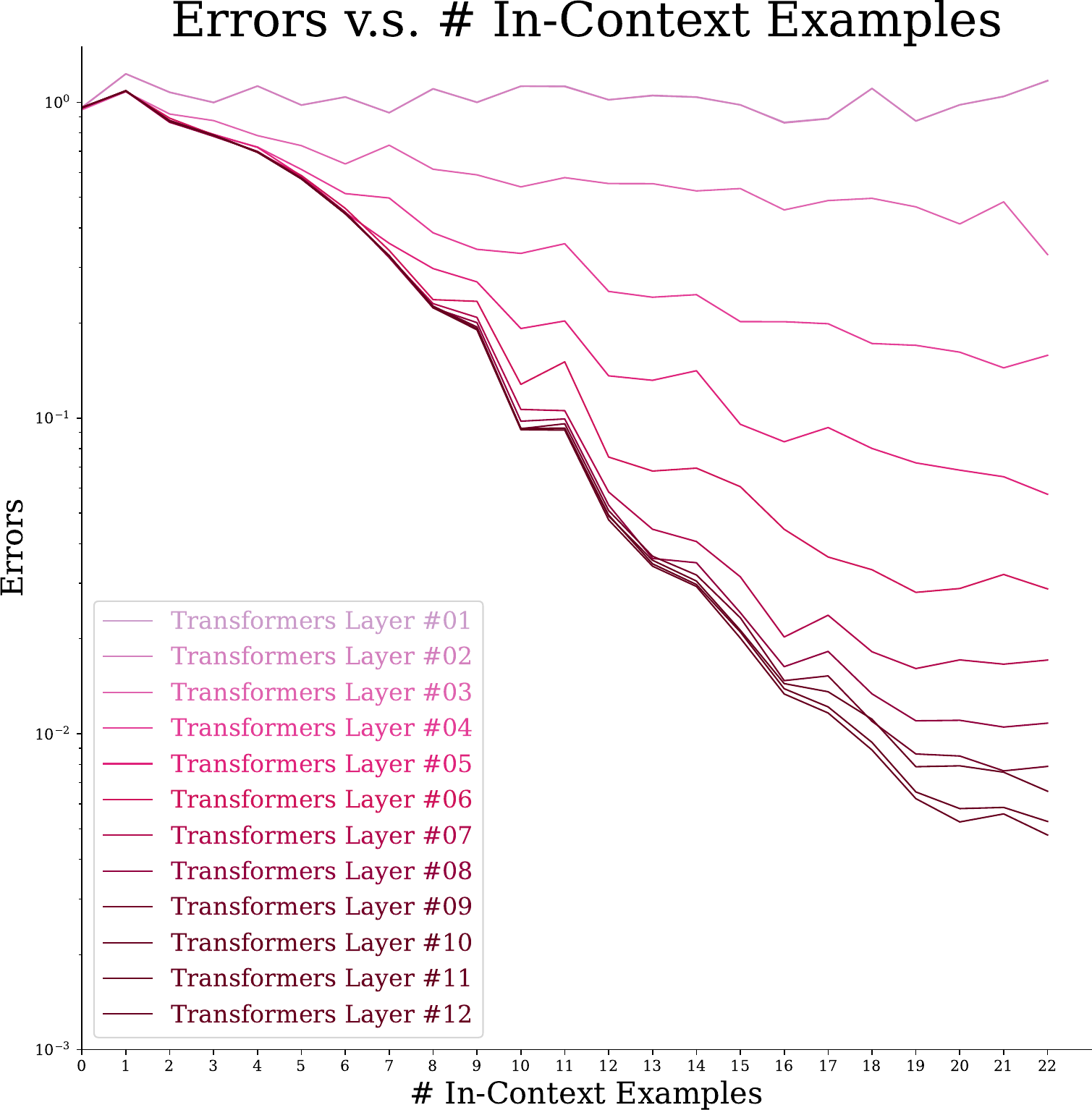}} \label{subfig:transformers_skew_only}
    \subfigure[Iterative Newton's Method]{\includegraphics[width=0.4\linewidth]{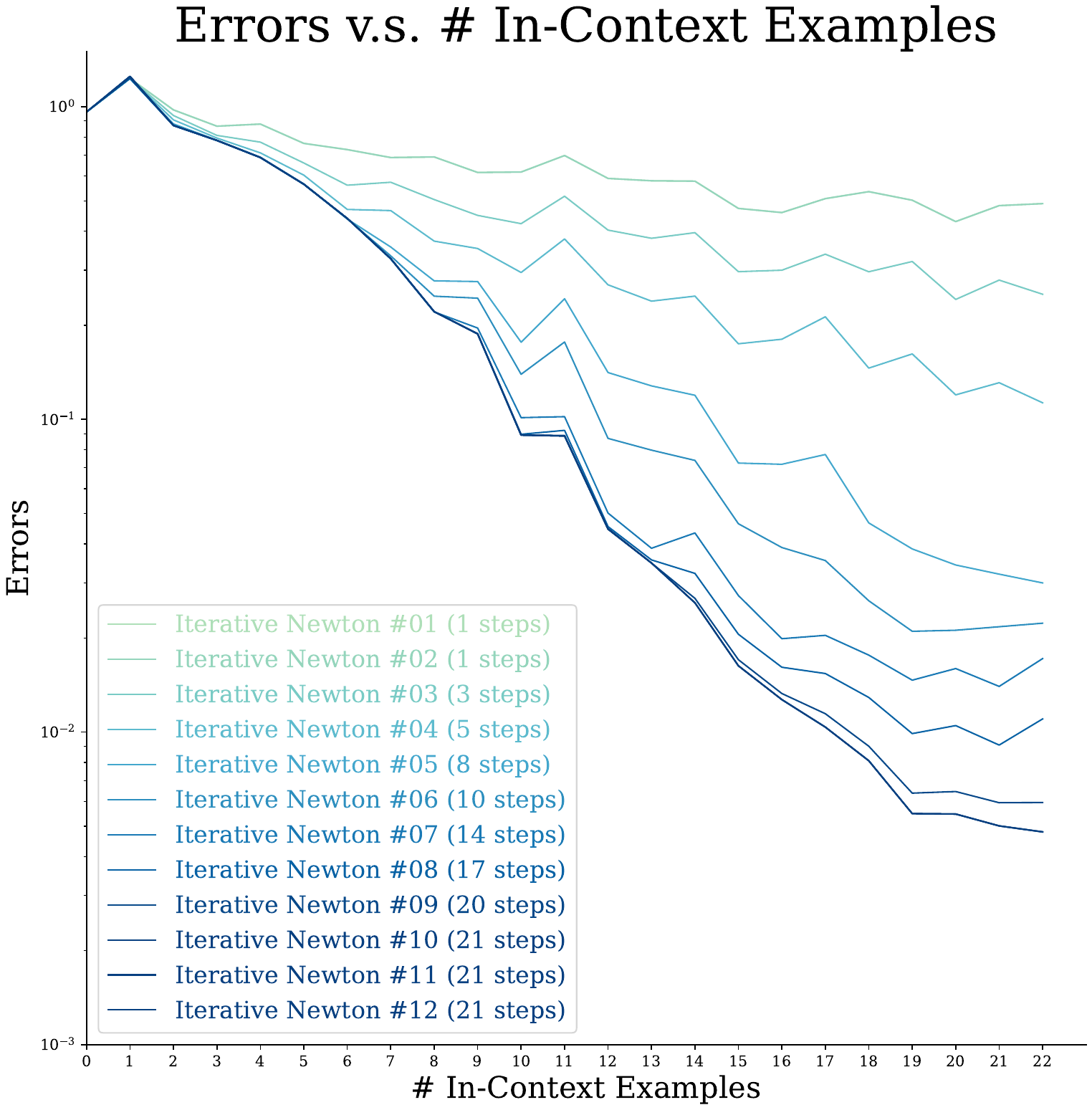}} \label{subfig:newton_skew_only}
    \caption{\textbf{Progression of Algorithms on Ill-Conditioned Data.} Transformer's performance still improves over the layer index $\ell$; Iterative Newton's Method's performance improves over the number of iterations $t$ and we plot the best-matching $t$ to Transformer's $\ell$ following \cref{def:matching}.}
    \label{fig:progression_skew}
\end{figure}

We also present the heatmaps to find the best-matching steps and conclude that Transformers are similar to Newton's method than GD in ill-conditioned data.
\begin{figure}[!htp]
    \centering
\includegraphics[width=0.49\linewidth]{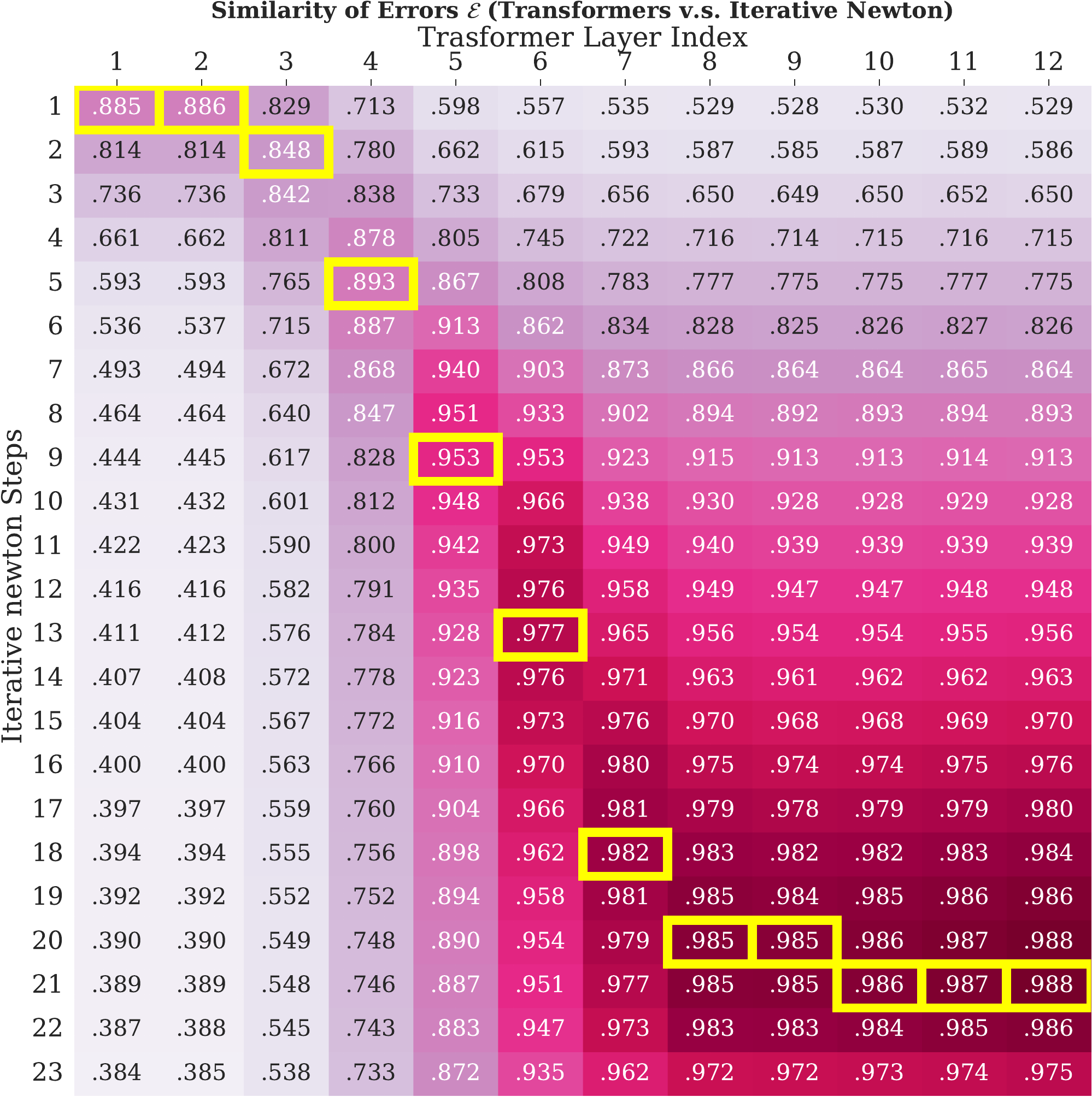}
\includegraphics[width=0.49\linewidth]{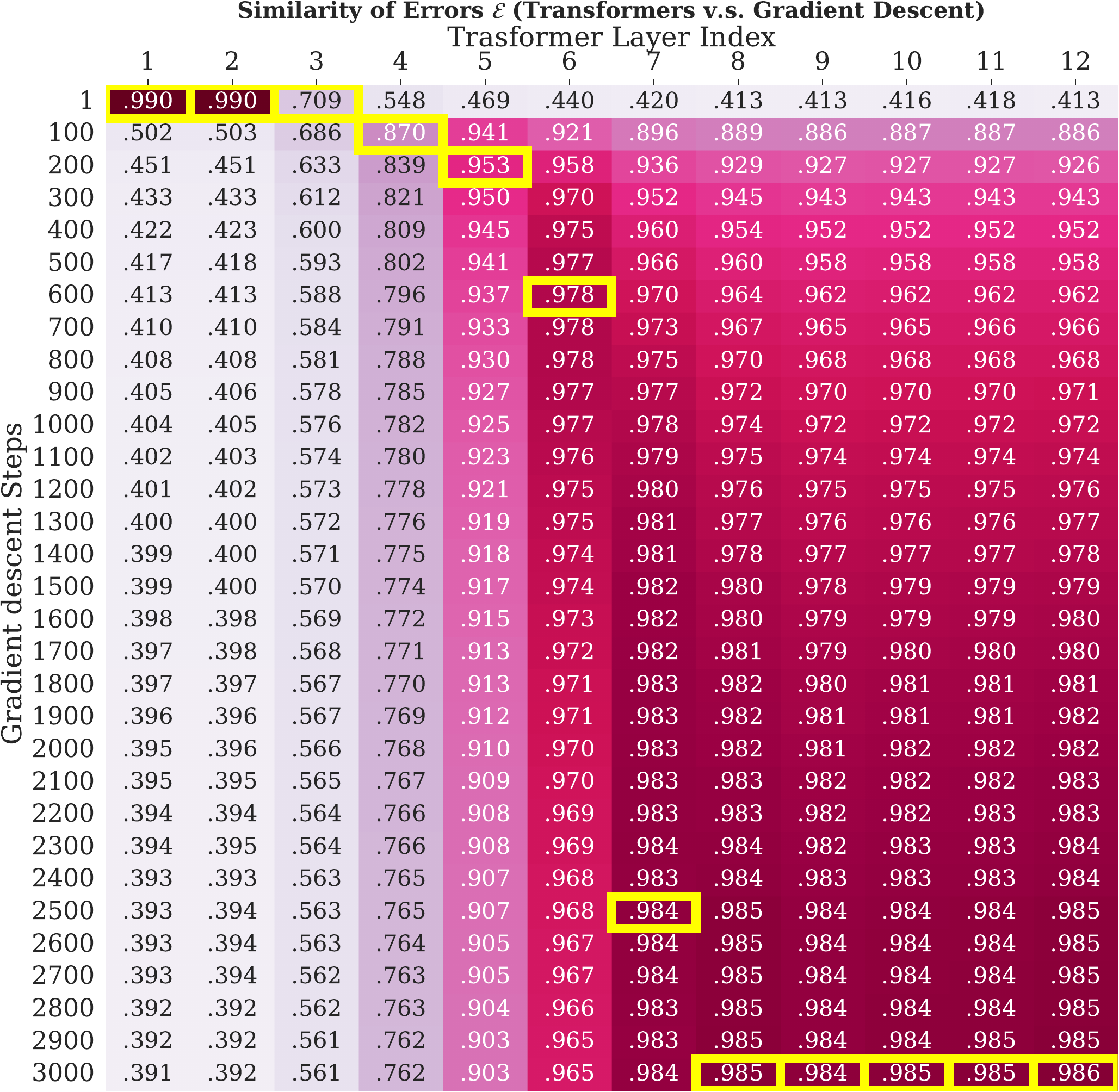}
\caption{\textbf{Similarity of Errors on Ill-Conditioned Data.} The best matching steps are highlighted in yellow.} \label{fig:ill_full_heatmap_sim_e}
\end{figure}

\begin{figure}[!htp]
    \centering
\includegraphics[width=0.49\linewidth]{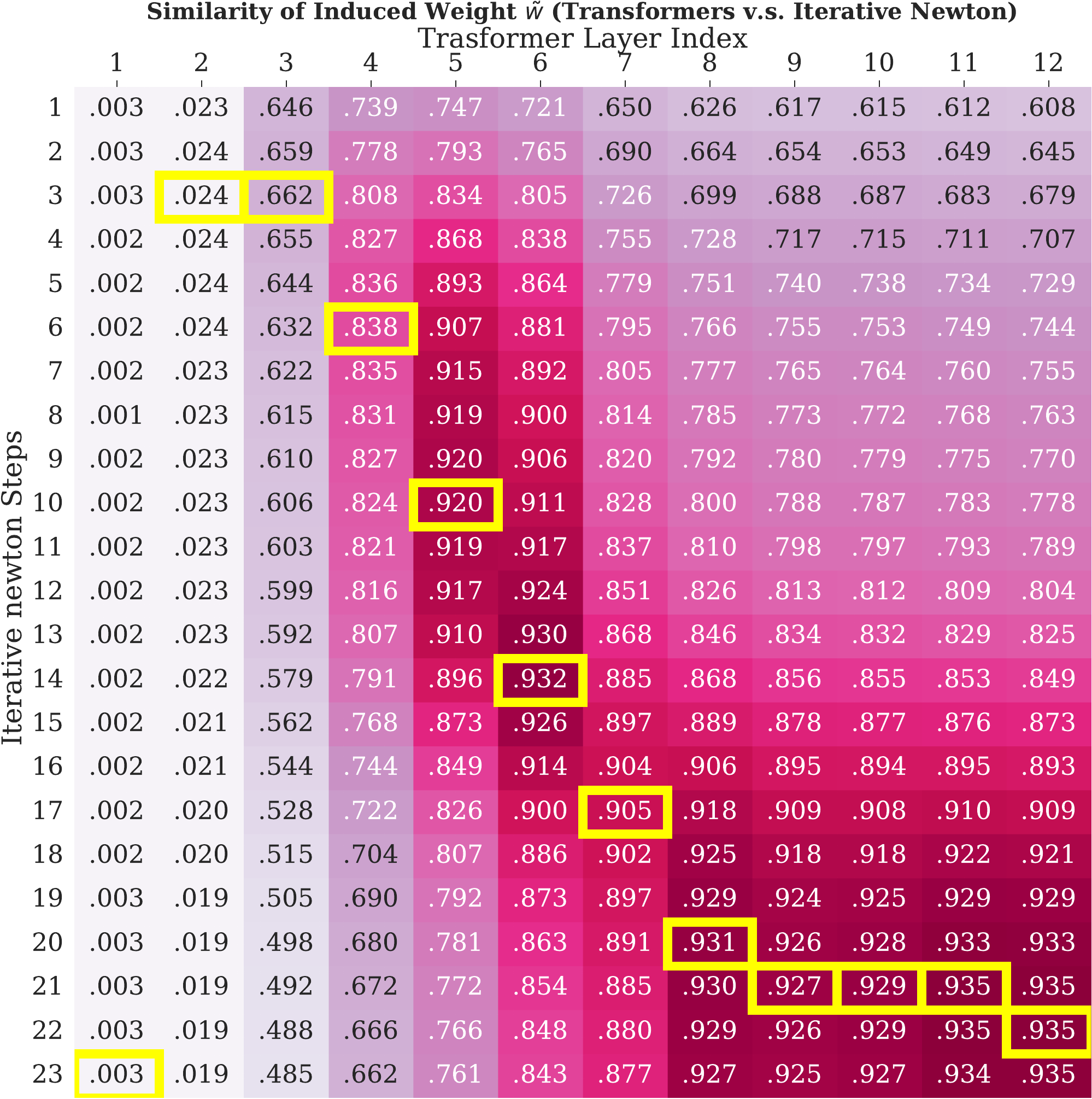}
\includegraphics[width=0.49\linewidth]{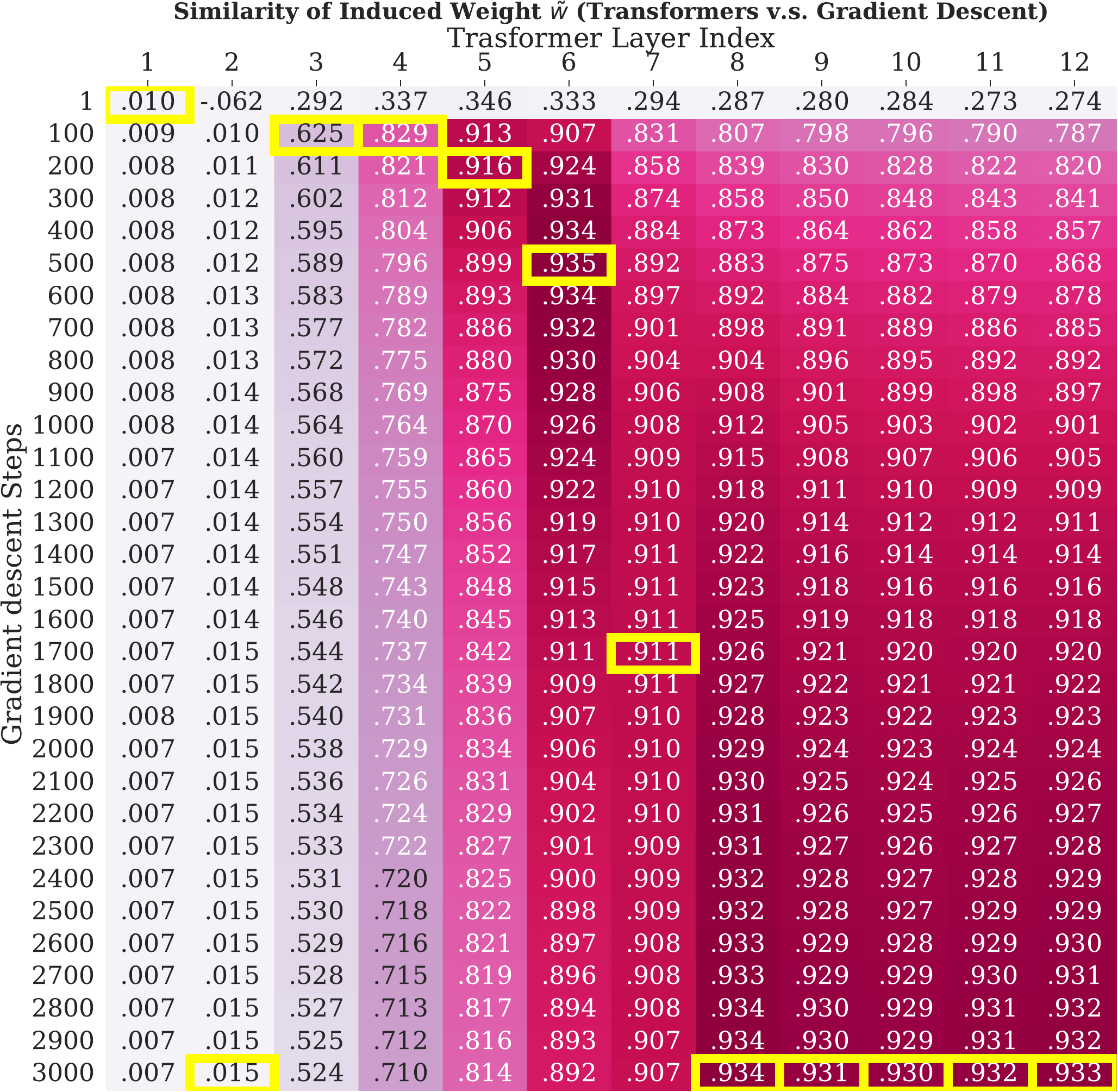}
\caption{\textbf{Similarity of Induced Weights on Ill-Conditioned Data.} The best matching steps are highlighted in yellow.} \label{fig:ill_full_heatmap_sim_w}
\end{figure}

\begin{figure}[!htp]
    \centering
\subfigure[Transformer v.s. BFGS]{\includegraphics[width=0.49\linewidth]{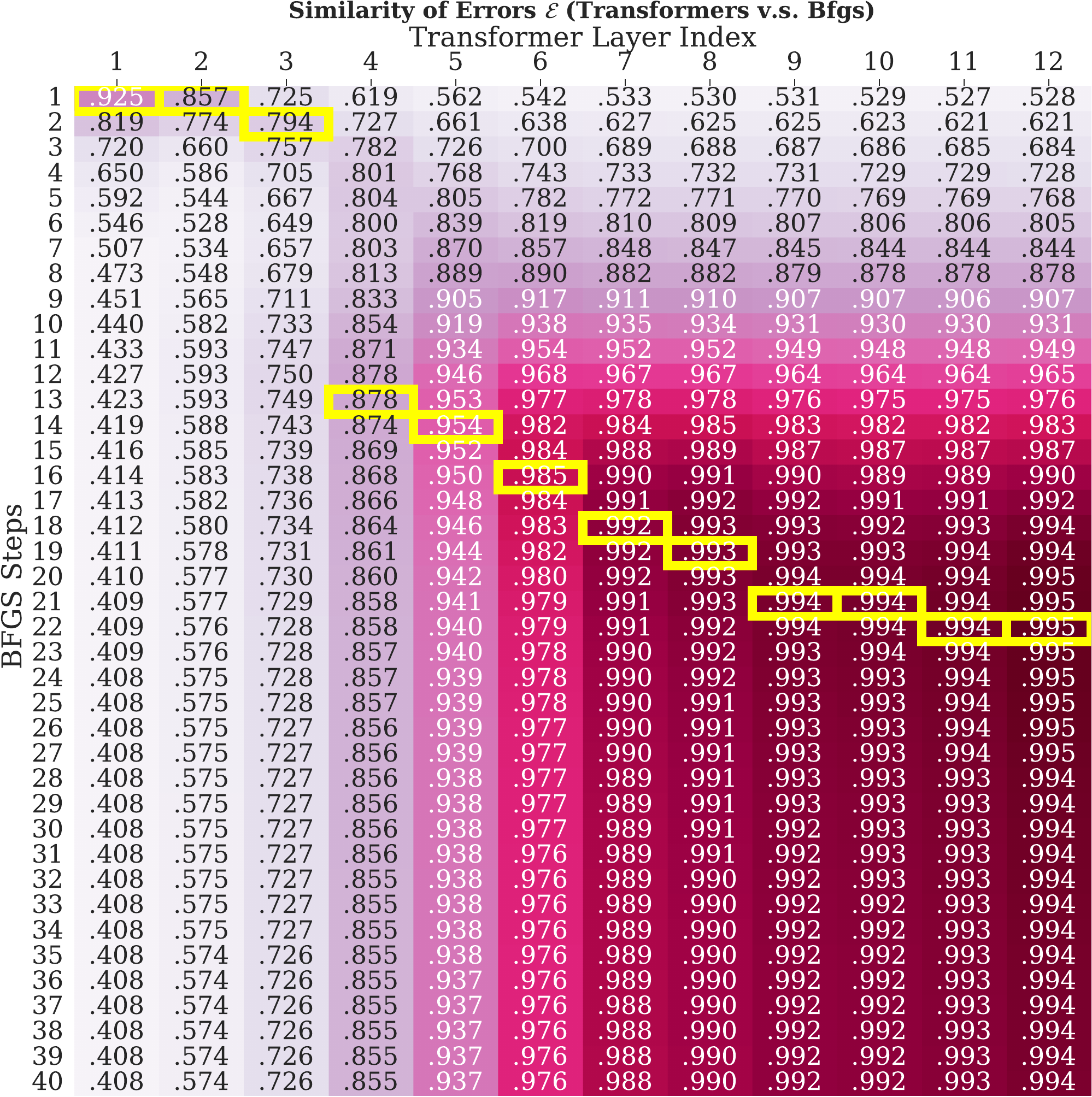}}
\subfigure[Transformer v.s. L-BFGS]{
\includegraphics[width=0.49\linewidth]{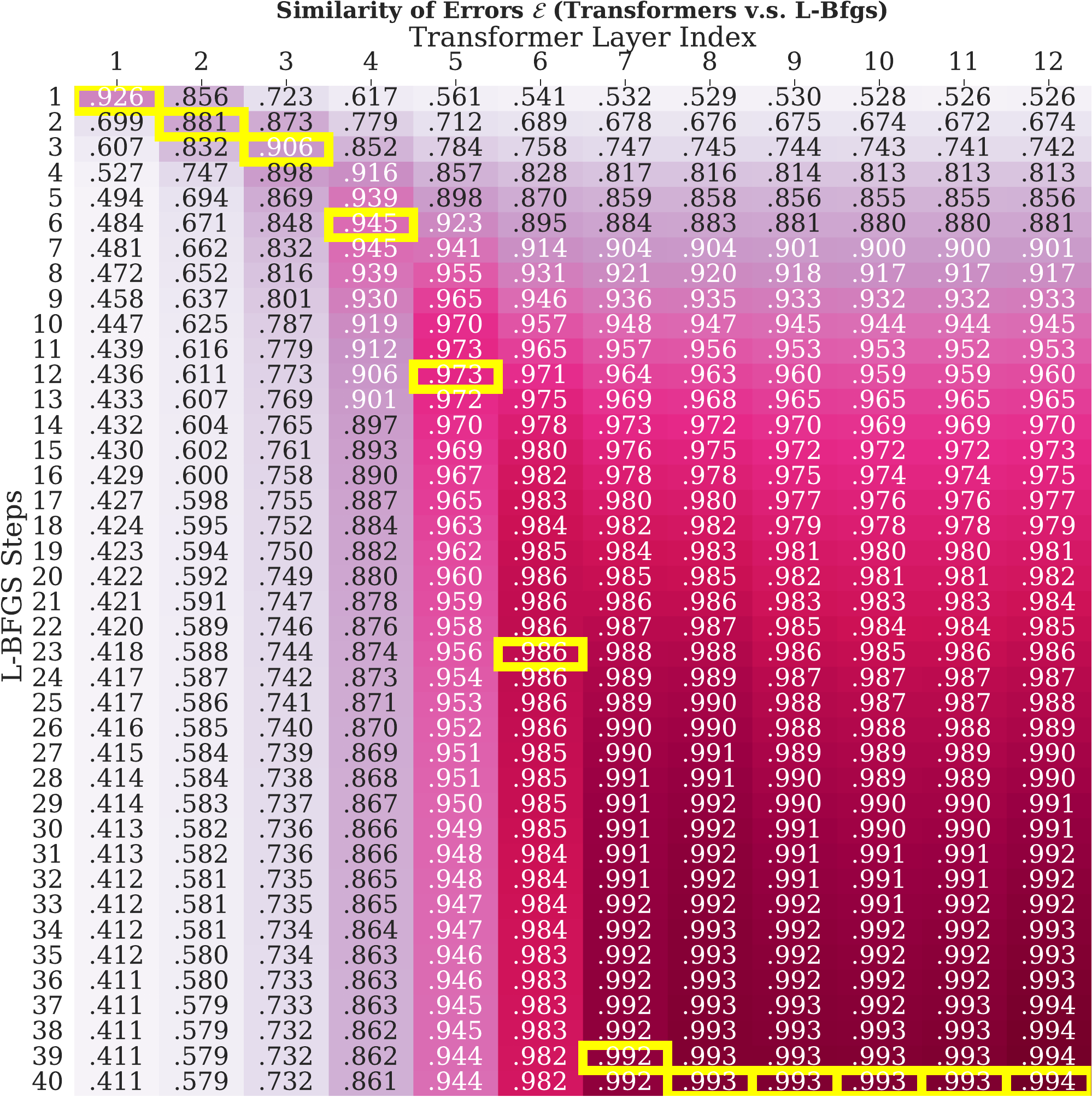}
}
\caption{\textbf{Similarity of Errors on Ill-Conditioned Data with Quasi-Newton Methods.} The best matching steps are highlighted in yellow. Transformer also matches BFGS linearly, from layers 4 to 11. L-BFGS still suffers due to its limited memory but still better than \gd because L-BFGS also attempts to approximate second-order information.} \label{fig:ill_bfgs}
\end{figure}

\clearpage
\subsubsection{Experiments with Noisy Linear Regression} \label{ssec:noisy}
We repeat the same experiments on noisy linear regression tasks with $y = \bw^\top \bx + \varepsilon$ where $\varepsilon \sim \mathcal N(0, \sigma^2)$ with noise level $\sigma = 0.1$. As shown in Figure \ref{fig:noisy-lr}, Transformers still show superlinear convergence on noisy linear regression tasks. Since the predictor is $\hat{\bw} = \left(\bX^\top \bX + \lambda \bI\right)^\dagger \bX^\top \by$ for some $\lambda$, the iterative newton's method is applied to $\bS = \bX^\top \bX + \lambda \bI$. Iterative Newton's method still keeps the same superlinear convergence rates. As it's also shown in Figure \ref{fig:noisy-lr}, Transformers and Iternative Newton's rates match linearly, as in the noiseless linear regression tasks. 
\begin{figure}[htp]
    \centering
      \includegraphics[width=0.45\linewidth]{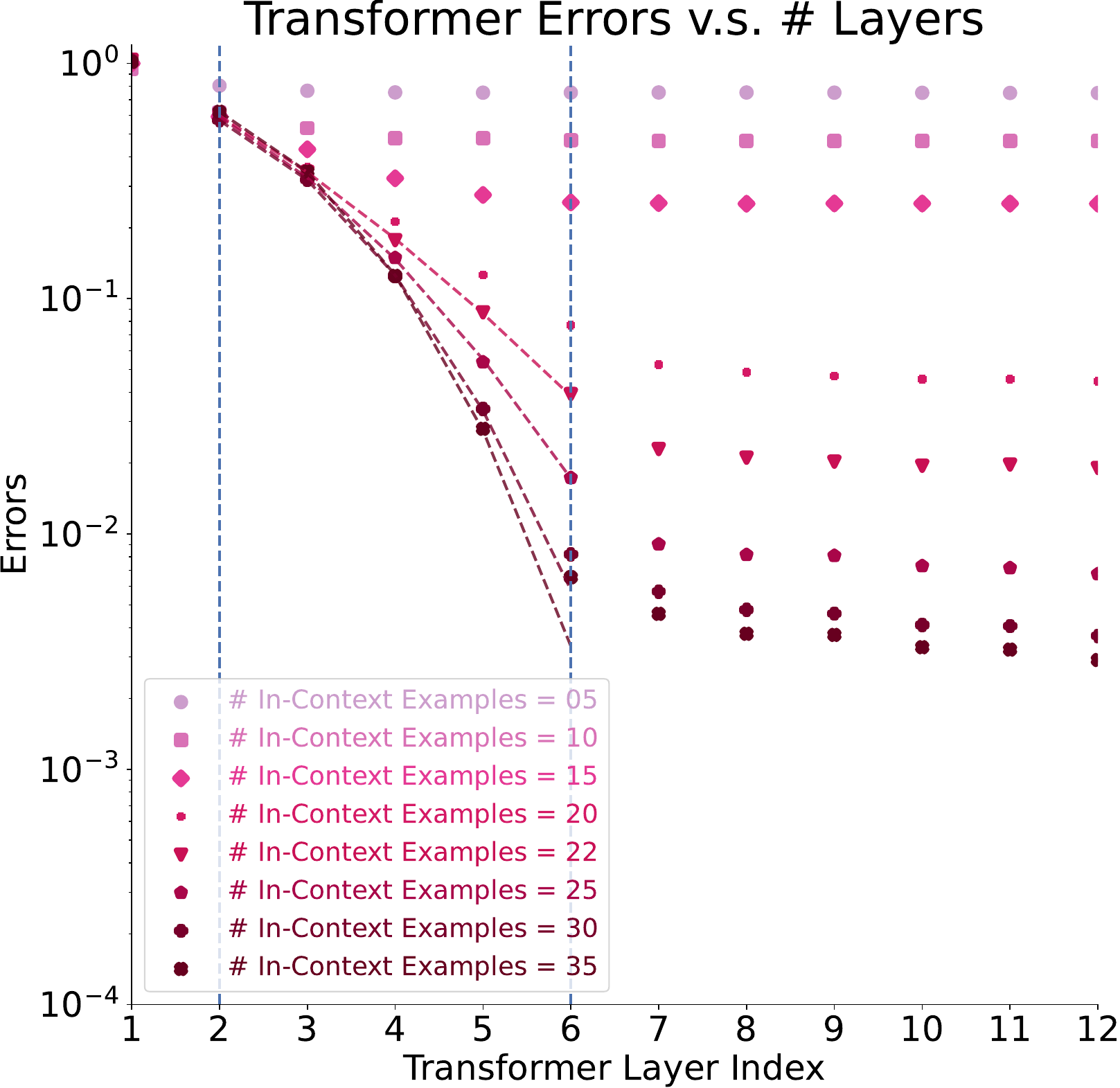}
      
      \includegraphics[width=0.45\linewidth]{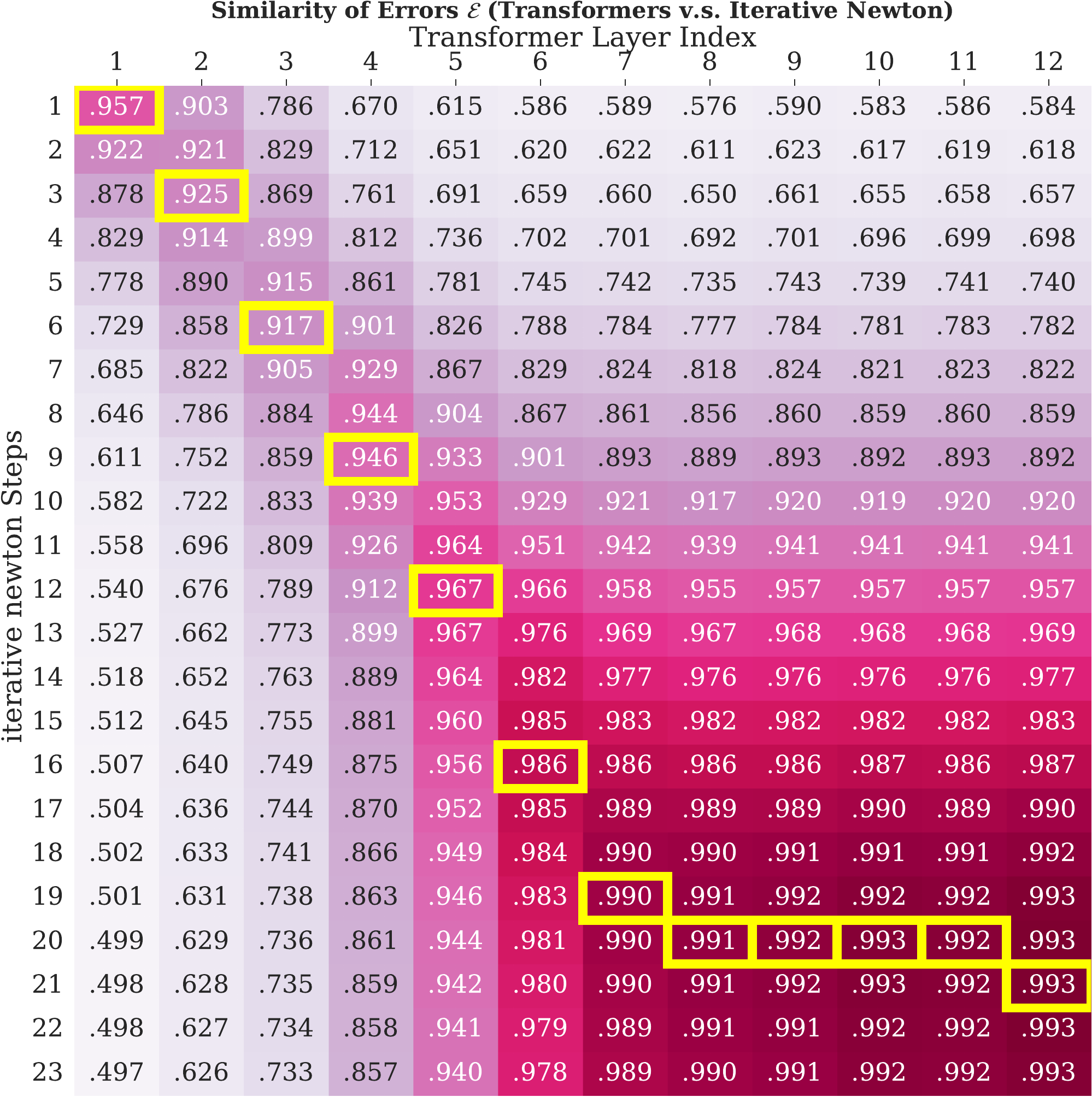}
      \hfill
      \includegraphics[width=0.45\linewidth]{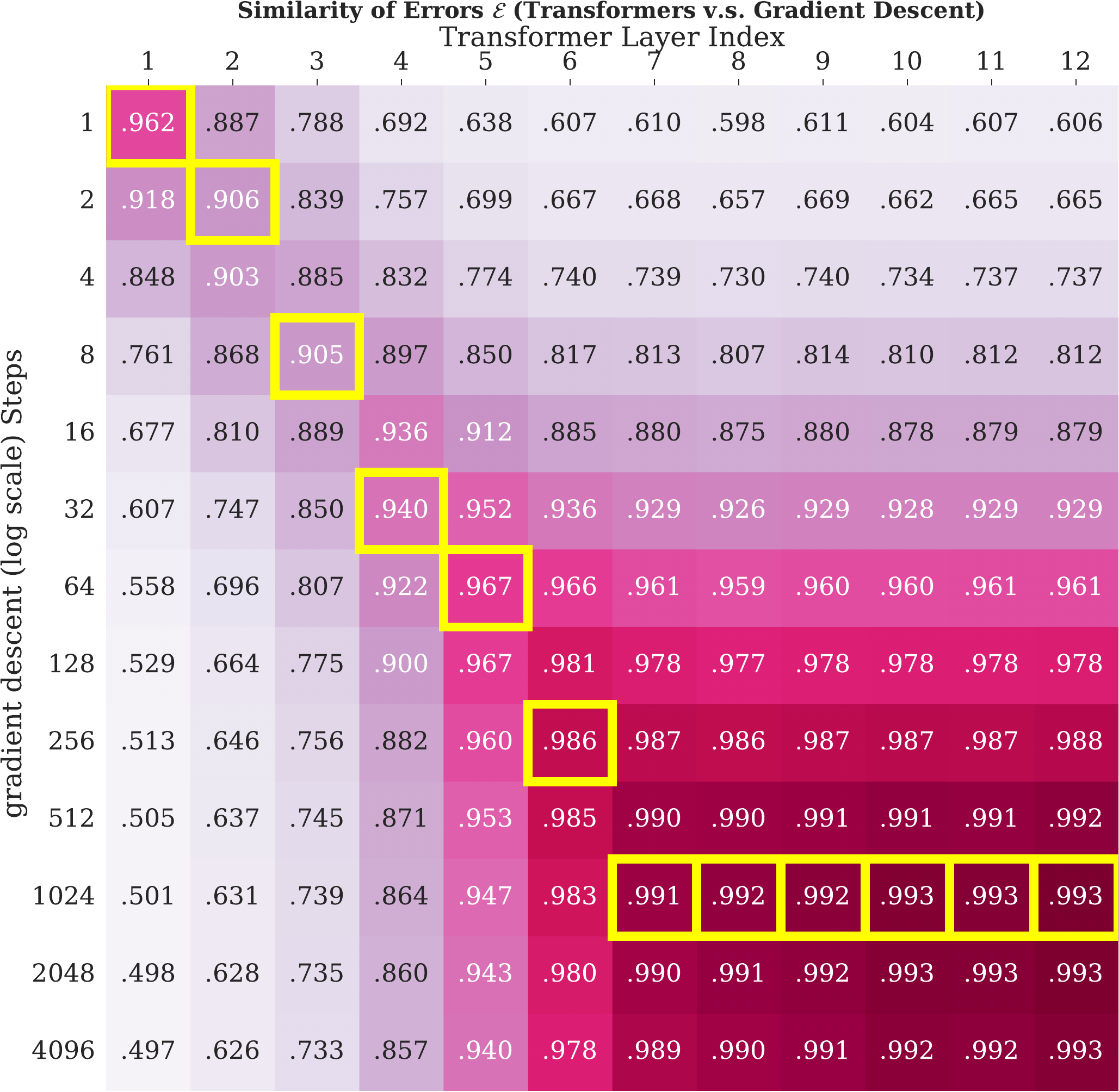}
      \caption{Experiment results on \textbf{Noisy Linear Regression}. \textbf{(Top)} Transformers have superlinear convergence rate. \textbf{(Bottom)} Transformers match Iterative Newton's rate and are exponentially faster than \gd.}
        \label{fig:noisy-lr}
\end{figure}

\subsubsection{Experiments with a Non-Linear Function Class (2-Layer MLP)} \label{app:non-linear}
To extend our experiments to non-linear cases, we adopt the same 2-layer ReLU neural network studied by \citet{Garg2022WhatCT}: see Fig. 5(c) in their paper. For any prompt $(\bx_1, y_1, \cdots, \bx_t, y_t)$, instead of generating labels $y_k = {\bw^\star}^\top \bx$ as mainly studied in the paper, we study a 2-layer neural network function class parameterized by $\bW \in \mathbb R^{d_\mathrm{hidden} \times d}$, $\bv \in \mathbb R^{d_\mathrm{hidden}}$, $\ba \in \mathbb R^{d_\mathrm{hidden}}$, and $b \in \mathbb R$, so that 
\begin{equation}
    y_k = f_{\bW, \bv, \ba, b}(\bx_k) = \ba^\top \mathrm{ReLU}\Big(\bW \bx_k + \bv\Big) + b
    \label{eqn:2nn-relu}
\end{equation}
Then we repeat the same probing experiments as in the main paper. As shown in Figure \ref{fig:transformer-relu2nn}, even on 2-layer neural network tasks with ReLU activation, Transformer shows superlinear convergence rates. Transformer shows an exponentially faster convergence rate than \gd's, because \gd's steps are shown in log scale and the trend is linear -- similar to Figure 9 in the main paper. 

\begin{figure}[htp]
    \centering
    \includegraphics[width=0.49\linewidth]{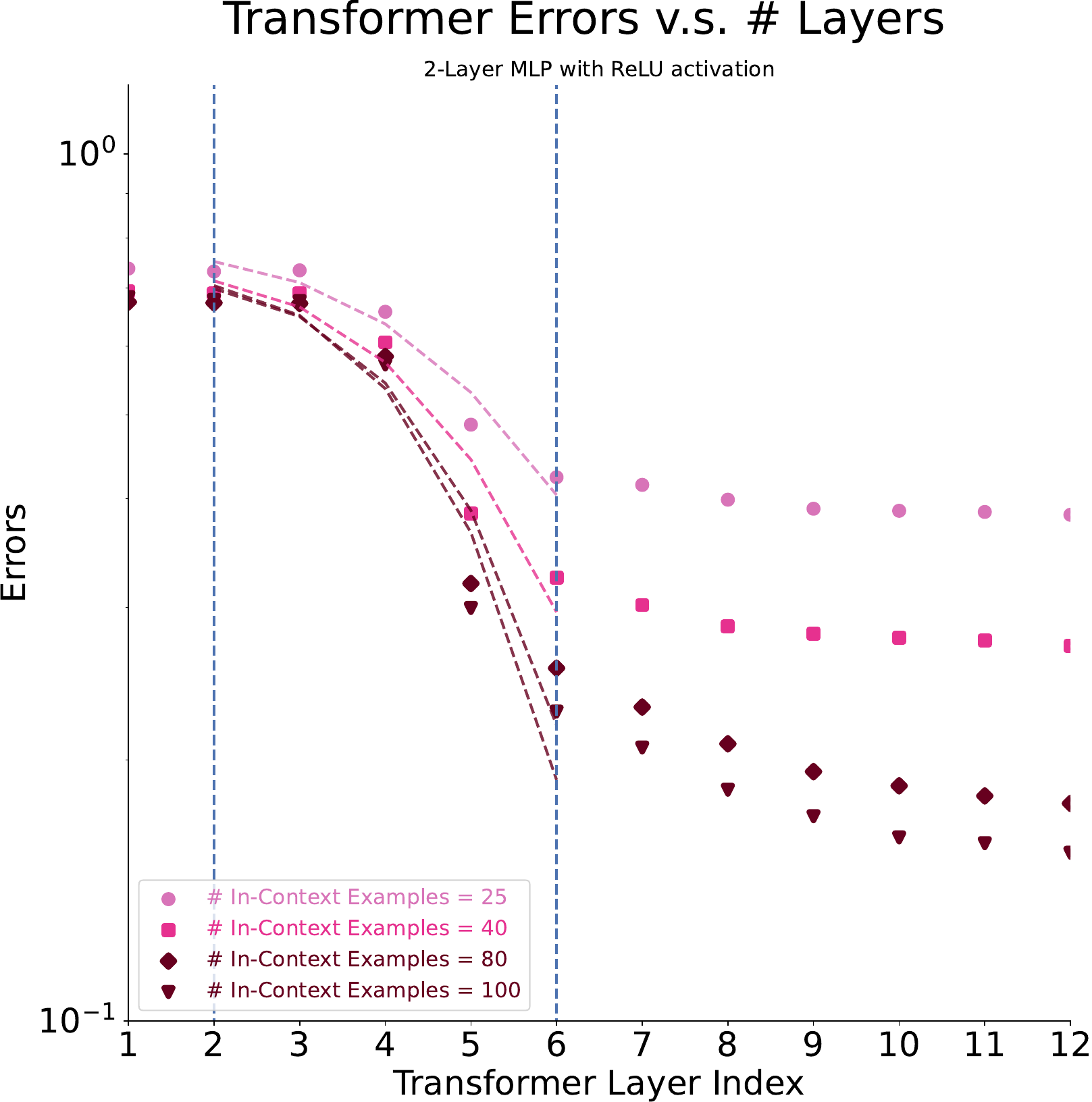}
    \hfill
  \includegraphics[width=0.49\linewidth]{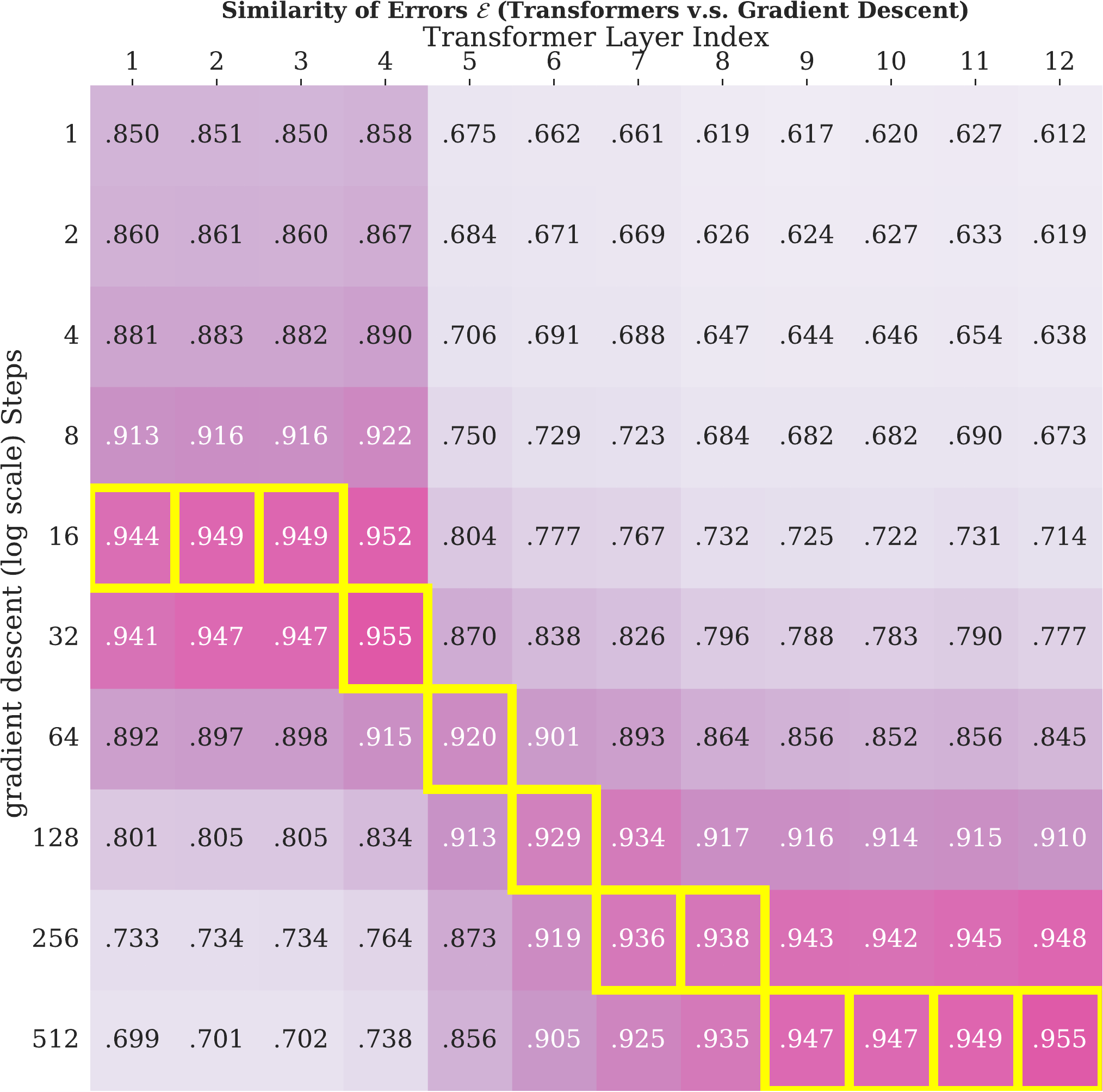}
    \caption{Empirical Results on 2-Layer Neural Network Regression with ReLU activation function. Transformers have superlinear convergence rates and match \gd's convergence rate exponentially}
    \label{fig:transformer-relu2nn}
\end{figure}

It would be interesting to ablate the activation function used in \Cref{eqn:2nn-relu}. We further consider the case when it's using the Tanh activation instead of ReLU, i.e.

\begin{equation}
    y_k = f_{\bW, \bv, \ba, b}(\bx_k) = \ba^\top \mathrm{Tanh}\Big(\bW \bx_k + \bv\Big) + b
    \label{eqn:2nn-tanh}
\end{equation}

Repeating the same experiments as before, as shown in \Cref{fig:transformer-tanh2nn}, we find that Transformers use the entire first 5 layers to pre-process and then only in the next few layers show exponentially faster convergence rate compared to \gd. We further note that in both \Cref{fig:transformer-relu2nn} and \Cref{fig:transformer-tanh2nn}, the cosine similarities between Transformers and \gd\ are significantly lower than the experiments with linear regression tasks. This might due to the over-parameterization of the function class and Transformers and \gd\ may arrive at different optima. 

\begin{figure}[htp]
    \centering
    \includegraphics[width=0.49\linewidth]{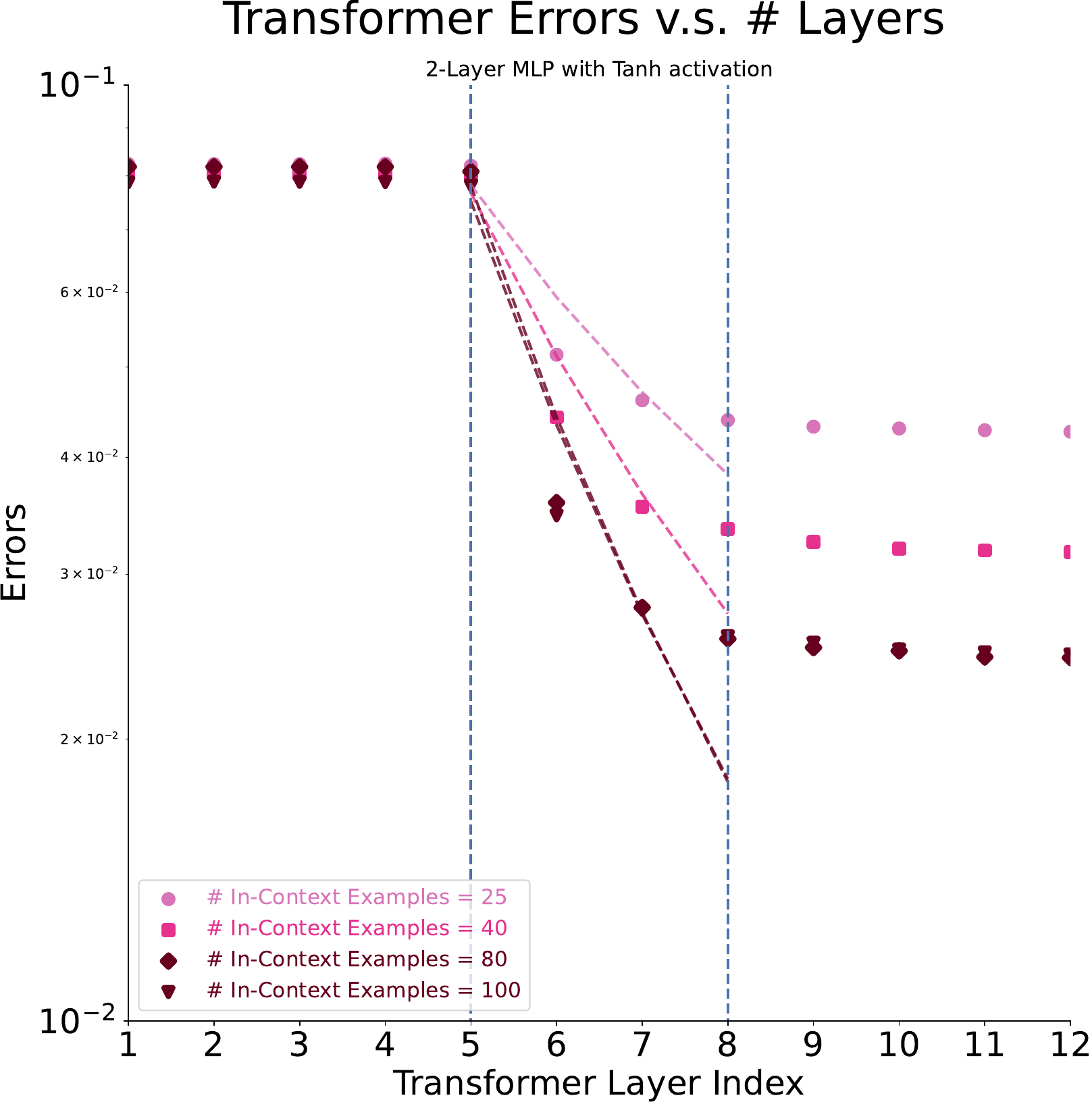}
    \hfill
  \includegraphics[width=0.49\linewidth]{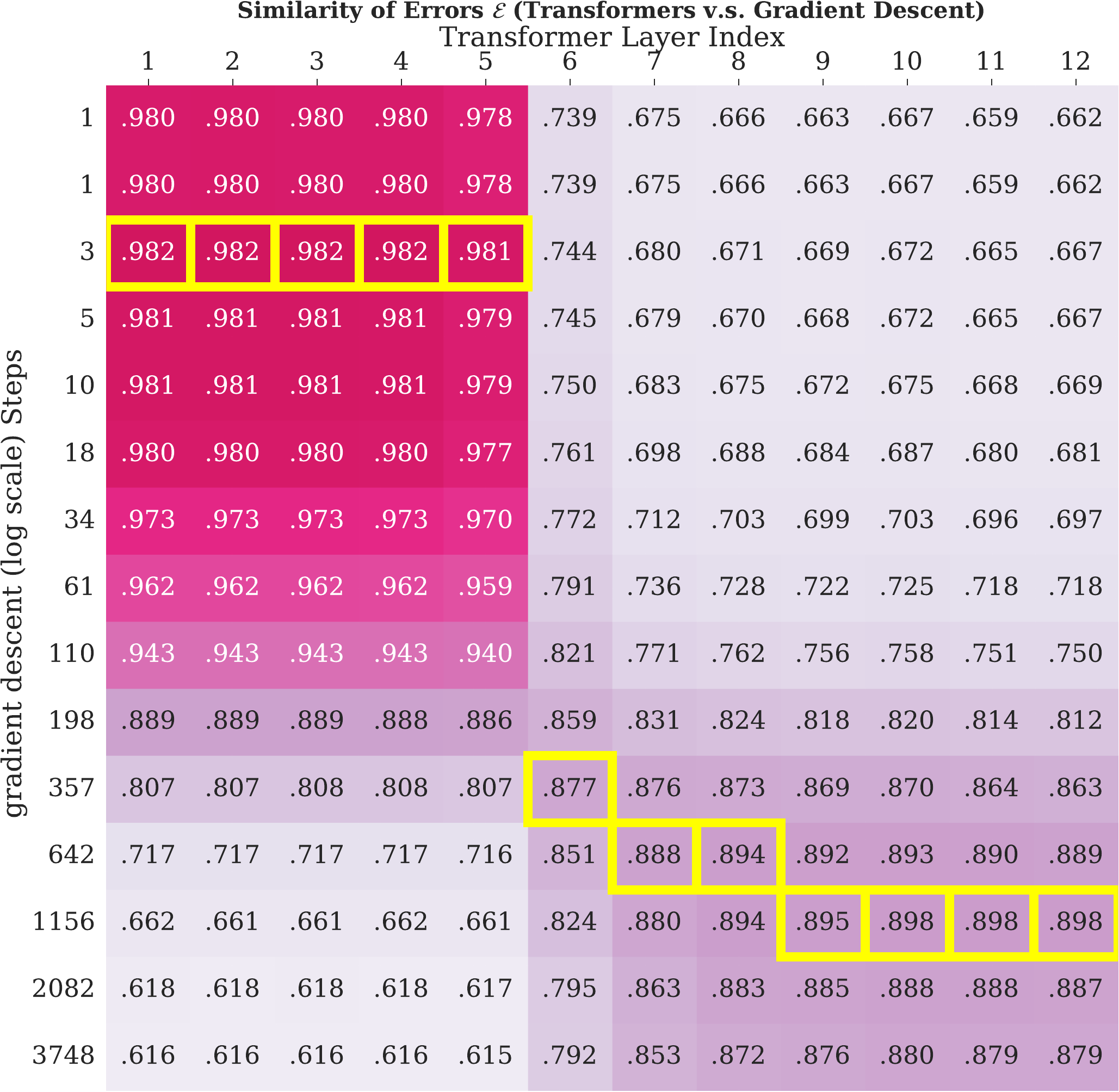}
    \caption{Empirical Results on 2-Layer Neural Network Regression with Tanh activation function. Transformers have superlinear convergence rates and match \gd's convergence rate exponentially}
    \label{fig:transformer-tanh2nn}
\end{figure}

It would be interesting for future research to explore further this function class of 2-layer MLP to understand fully how Transformer solve the regression problem in-context and whether it achieves a different optimum compared to alternative algorithms such as (Stochastic) \gd.
\clearpage
\subsection{Varying Transformer Architecture}
\subsubsection{Experiments on Transformers of Fewer Heads} \label{app:1-head}
In this section, we present experimental results from an alternative model configurations than the main text. We show in the main text that Transformers learn second-order optimization methods in-context where the experiments are using a GPT-2 model with 12 layers and 8 heads per layer. In this section, we present experiments with a GPT-2 model with 12 layers but only 1 head per layer. 

\begin{figure}[!htp]
    \centering
\includegraphics[width=0.45\linewidth]{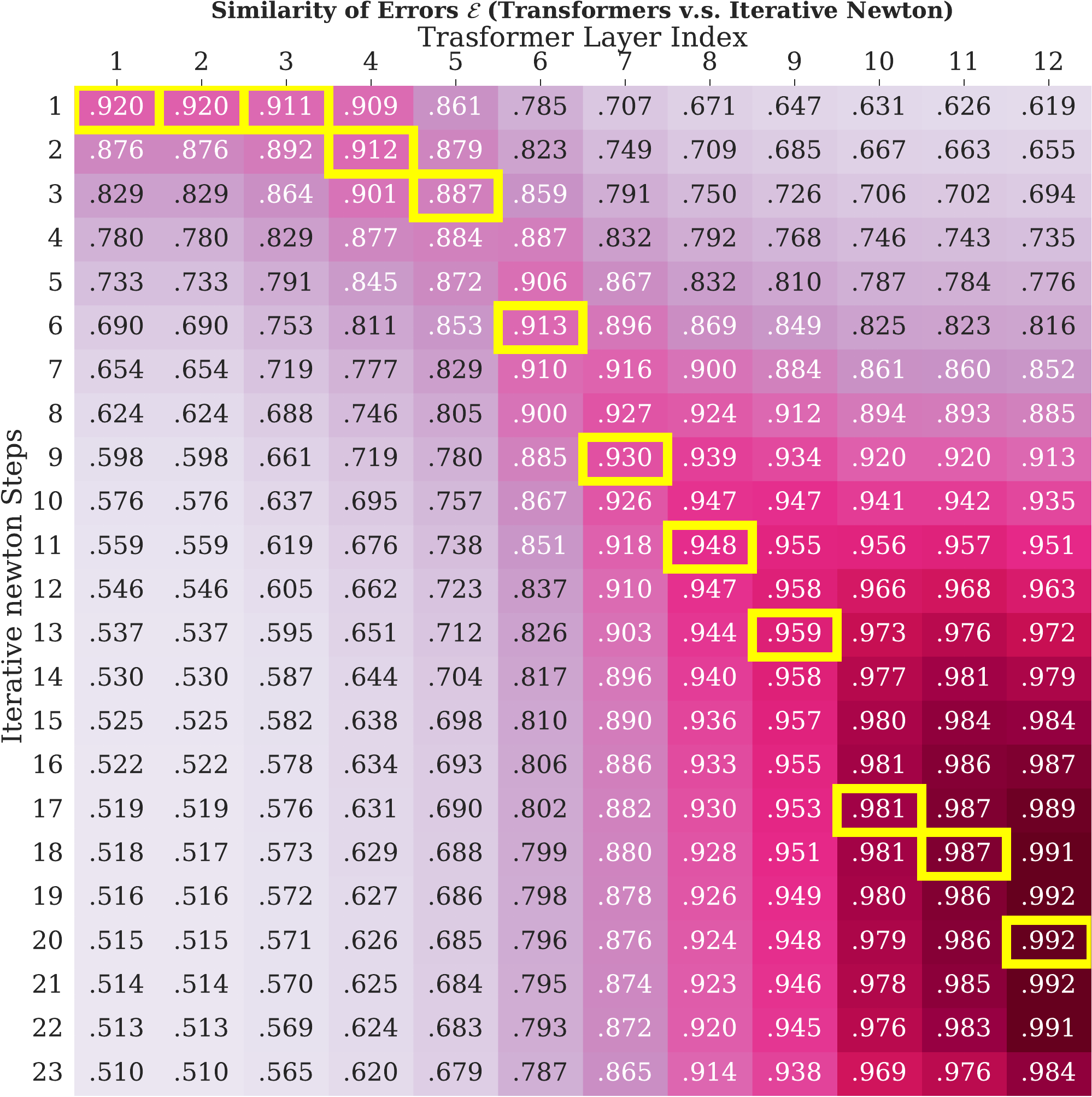}
\includegraphics[width=0.45\linewidth]{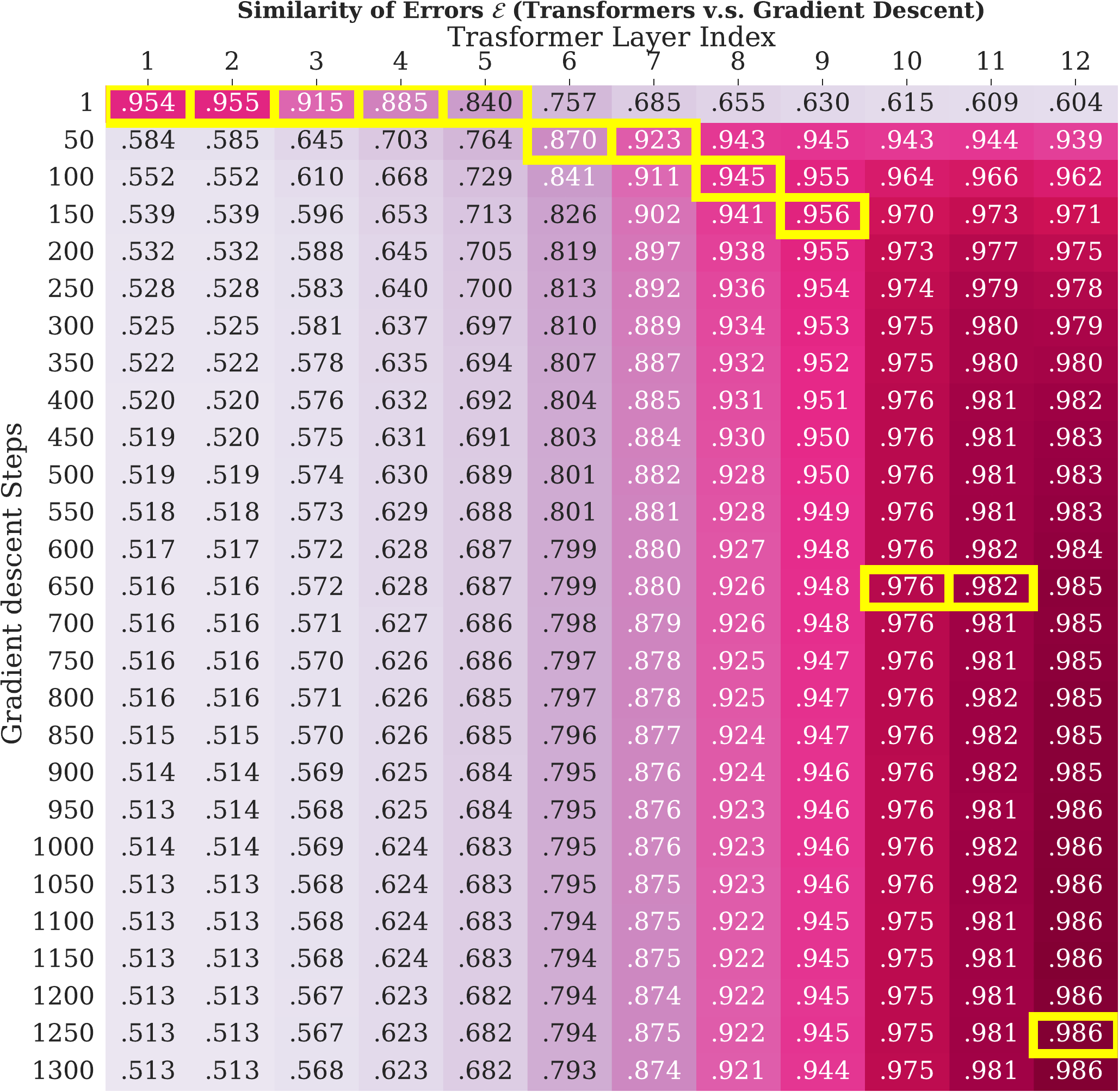}
\caption{\textbf{Similarity of Errors on an alternative Transformers Configuration.} The best matching steps are highlighted in yellow. } \label{fig:12layer_1head_full_heatmap_sim_e}
\end{figure}
\begin{figure}[!htp]
    \centering
\includegraphics[width=0.45\linewidth]{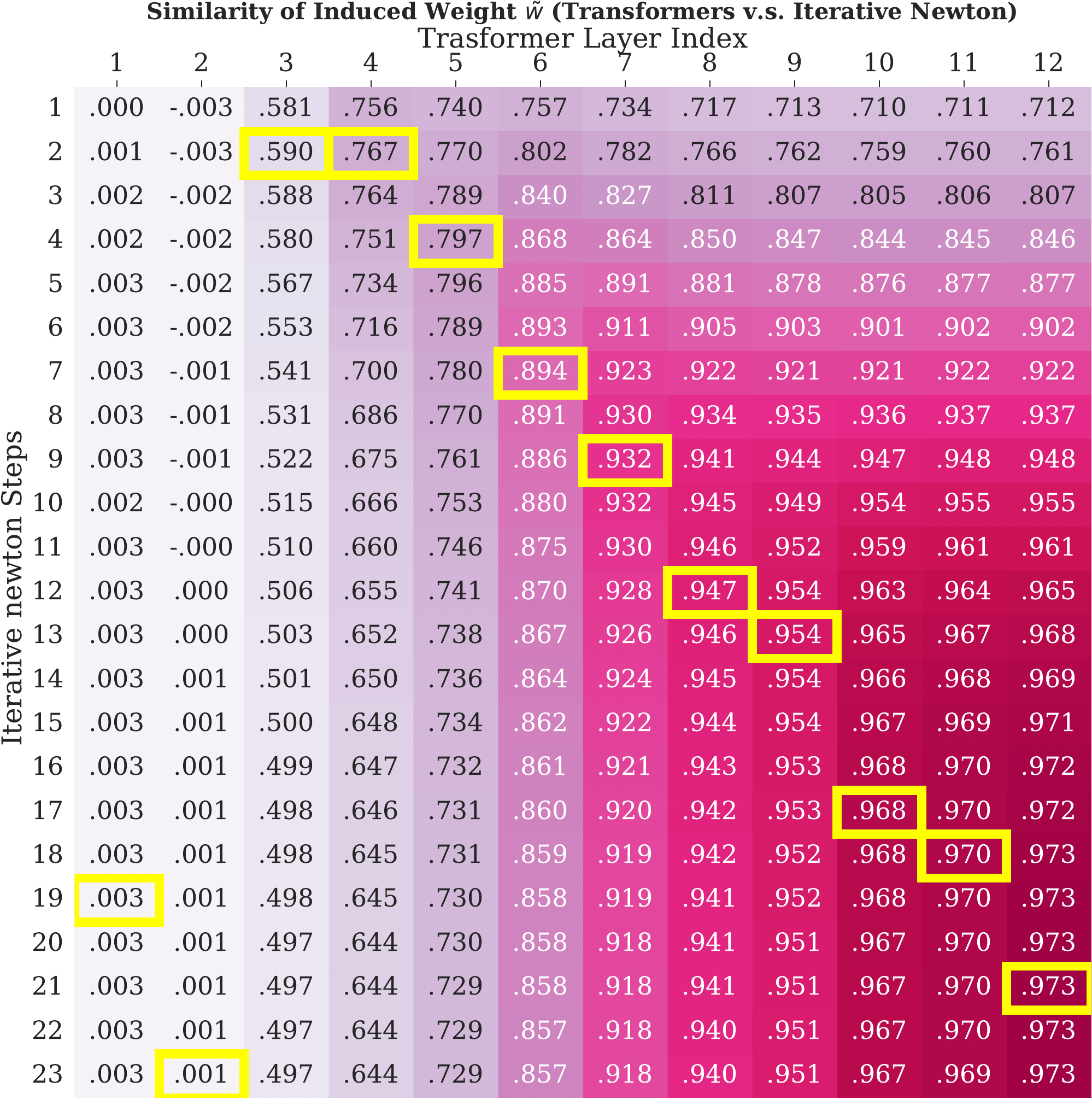}
\includegraphics[width=0.45\linewidth]{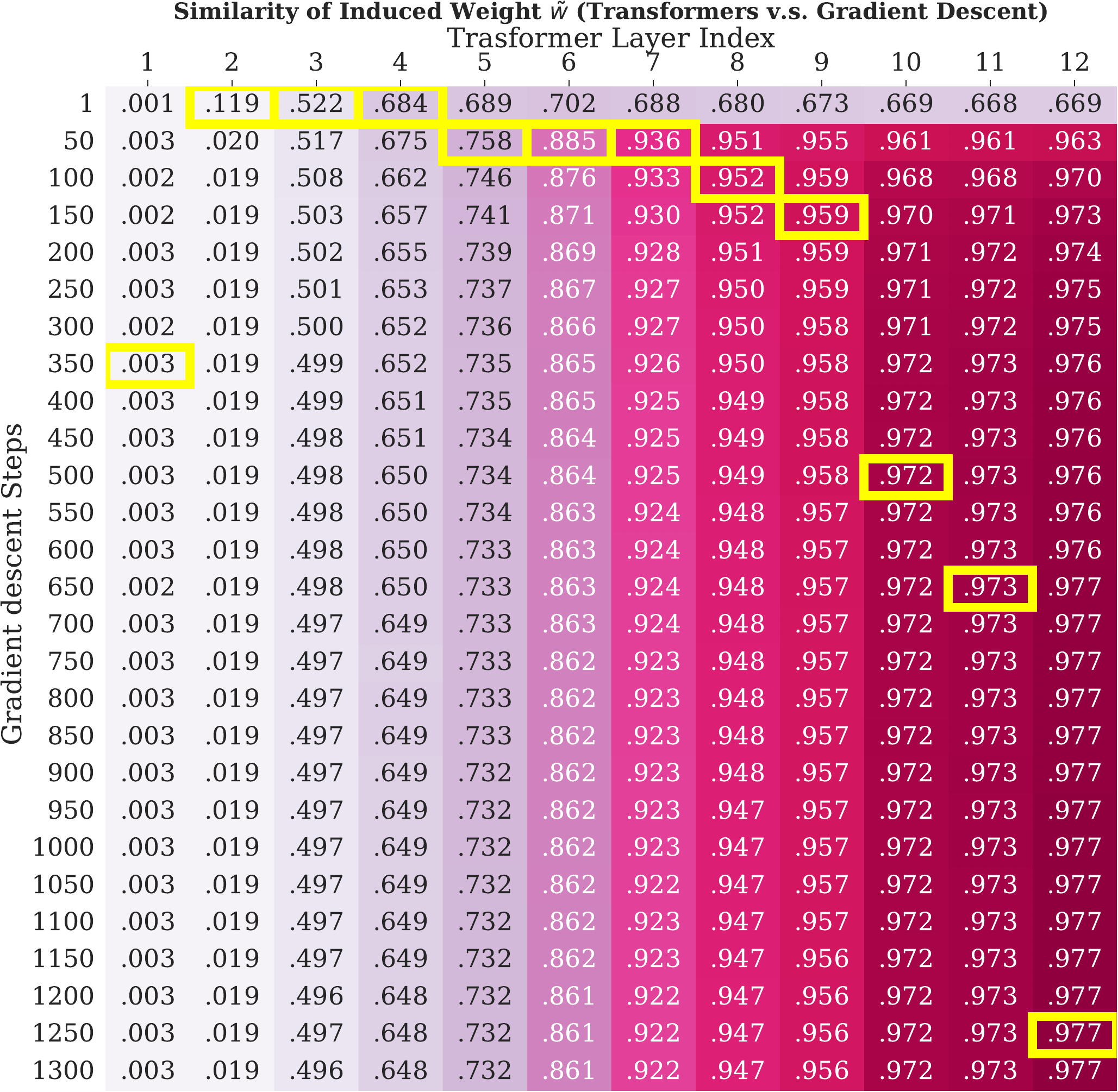}
\caption{\textbf{Similarity of Induced Weights on an alternative Transformers Configuration.} The best matching steps are highlighted in yellow.} \label{fig:12layer_1head_full_heatmap_sim_w}
\end{figure}

We conclude that our experimental results are not restricted to a specific model configurations, smaller models such as GPT-2 with 12 layers and 1 head each layer also suffice in implementing the \newton's method, and more similar than gradient descents, in terms of rate of convergence. 

\subsubsection{Experiments on Transformers with More Layers} \label{app:24-layer}
In this section, we investigate whether deeper models would behave similarly or differently. We work on Transformers with 24 layers and 8 heads each. 

\begin{figure}[!htp]
    \centering
\includegraphics[width=0.45\linewidth]{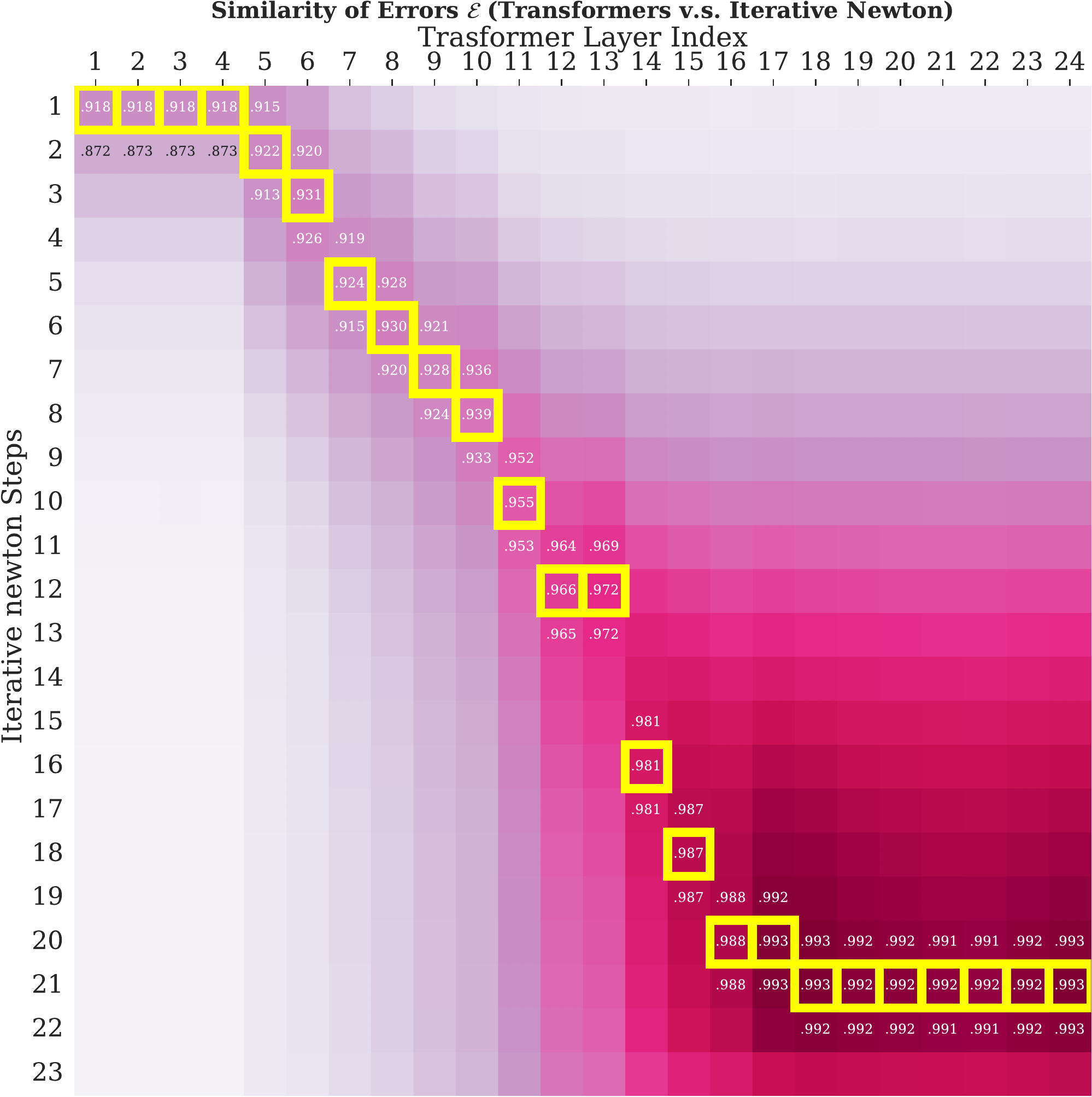}
\includegraphics[width=0.45\linewidth]{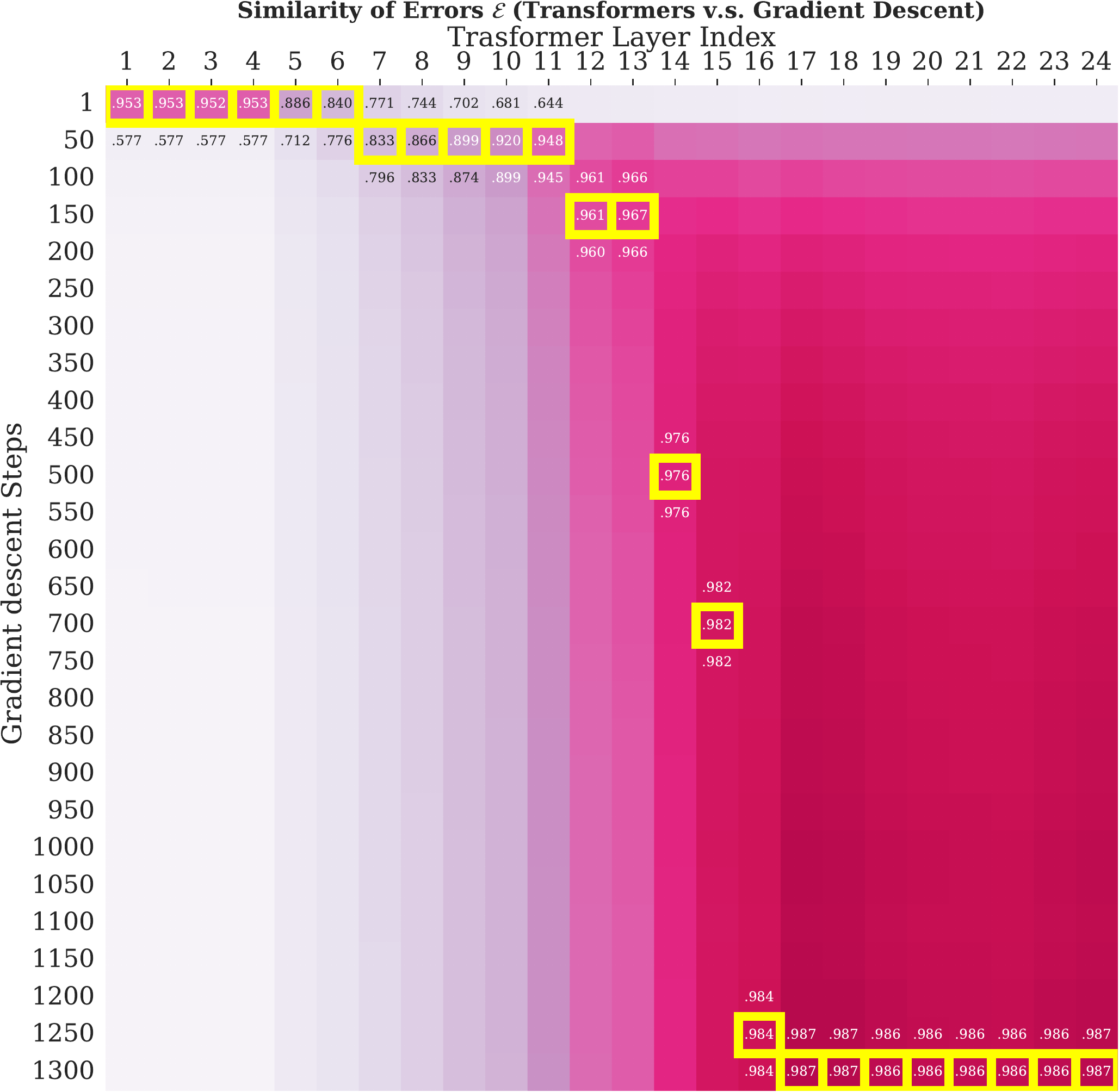}
\caption{\textbf{Similarity of Errors on a 24-layer Transformers Configuration.} The best matching steps are highlighted in yellow. } \label{fig:24layer_full_heatmap_sim_e}
\end{figure}

\begin{figure}[!htp]
    \centering
\includegraphics[width=0.45\linewidth]{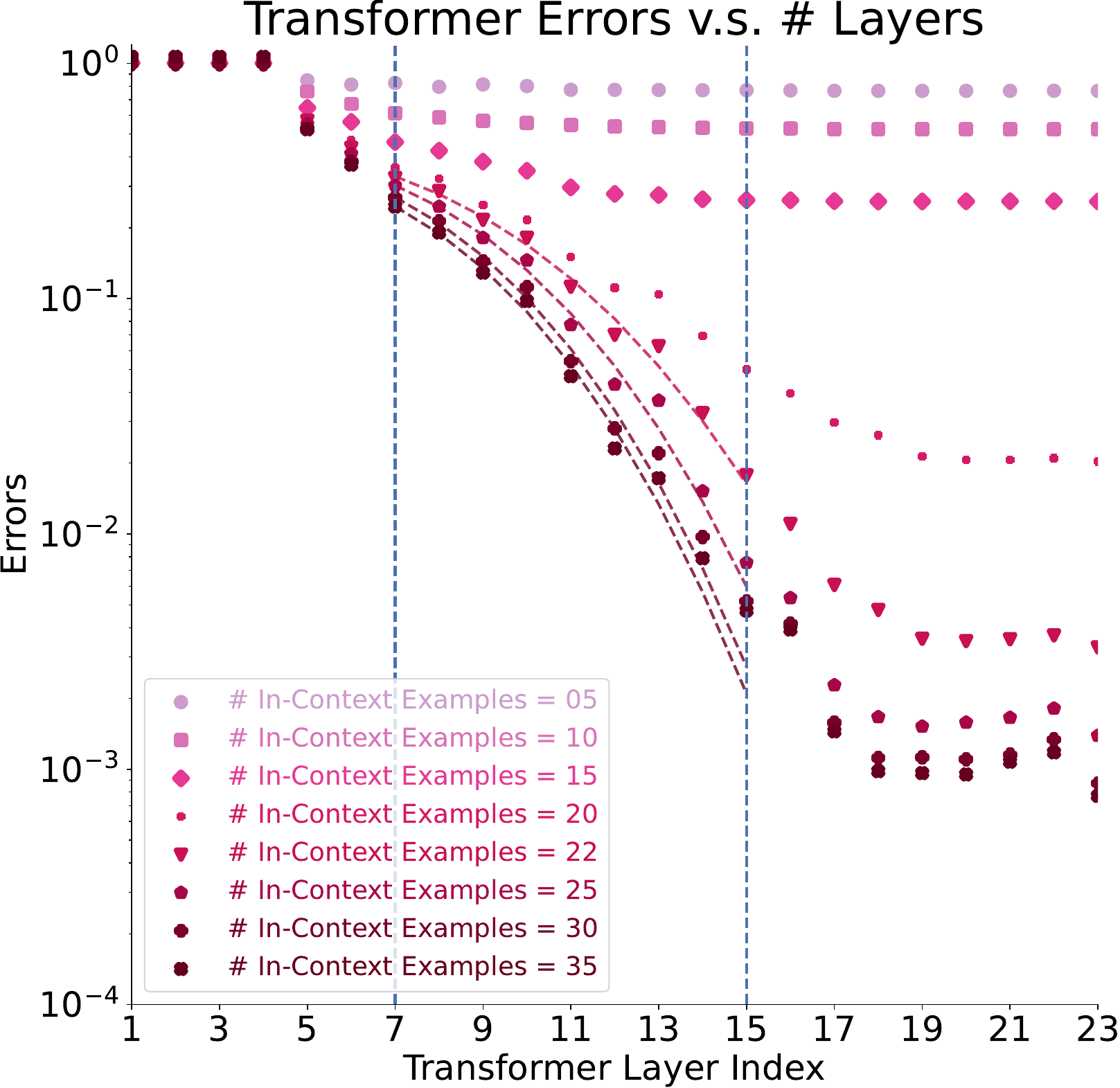}
    \caption{Transformers with 24 layers also converge superlinearly, similar to \newton.}
    \label{fig:convergence-24layer}
\end{figure}

\clearpage
\subsection{Heatmaps with Best-Matching Steps Help Compare Convergence Rates} \label{app:best-match}
{In this section, we show the heatmaps with best-matching steps among \textit{known algorithms.}}
\begin{figure}[!htp]
    \centering
    \subfigure[\newton\ v.s. \gd]{
    \includegraphics[width=0.45\linewidth]{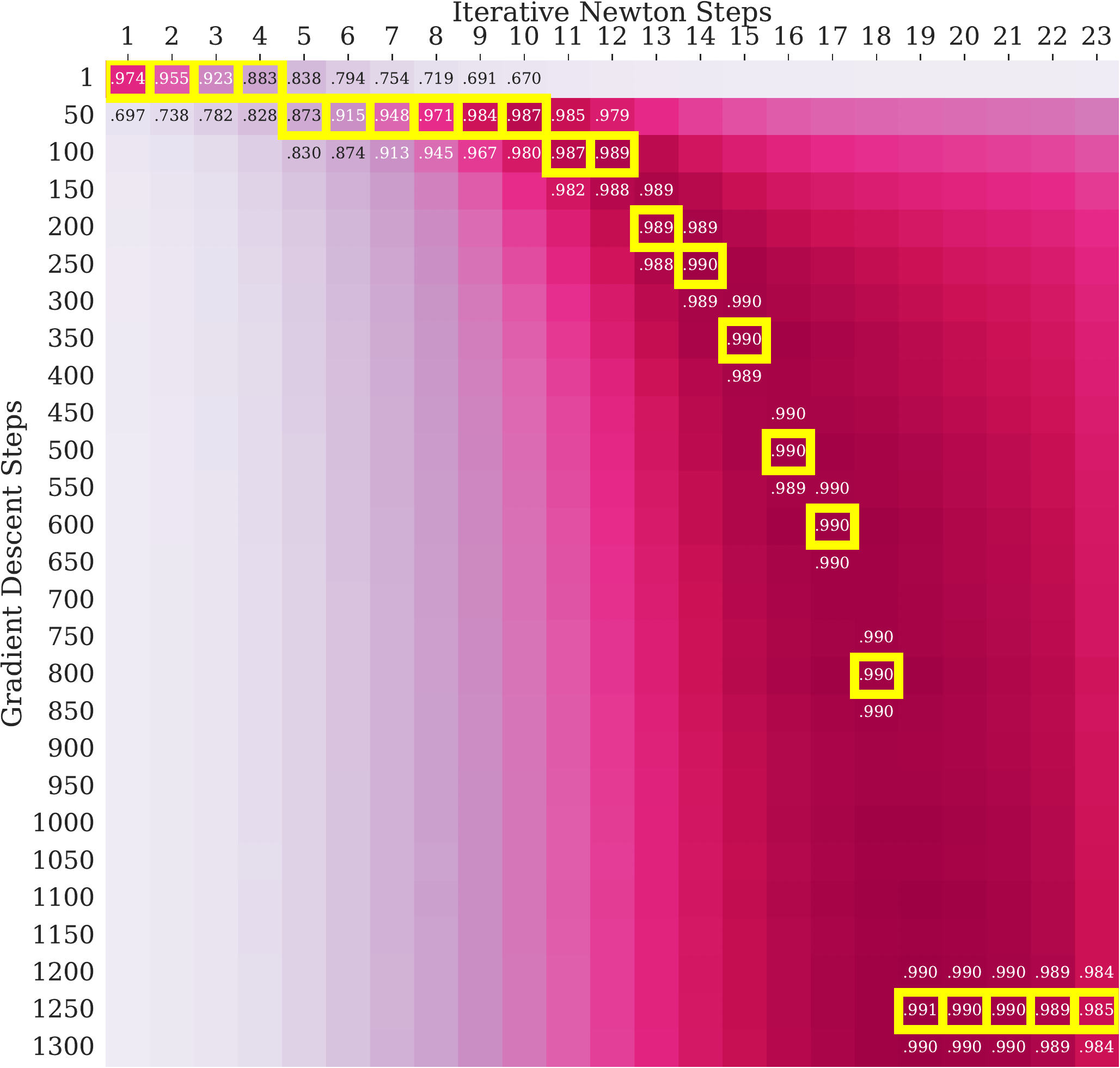}
    }
    \subfigure[\newton\ v.s. \newton]{
    \includegraphics[width=0.45\linewidth]{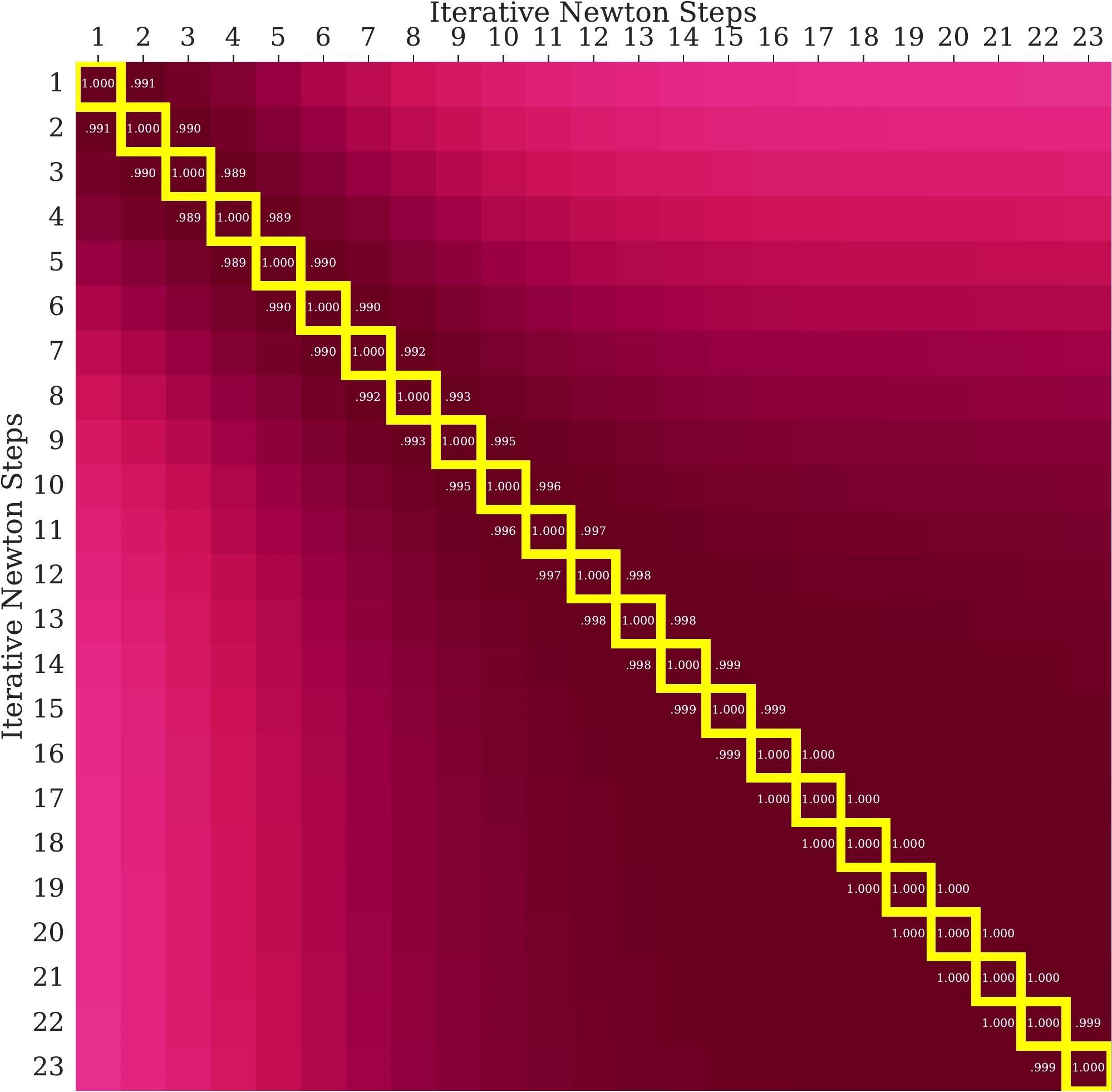}
    }
    \subfigure[\newton\ v.s. BFGS]{
    \includegraphics[width=0.45\linewidth]{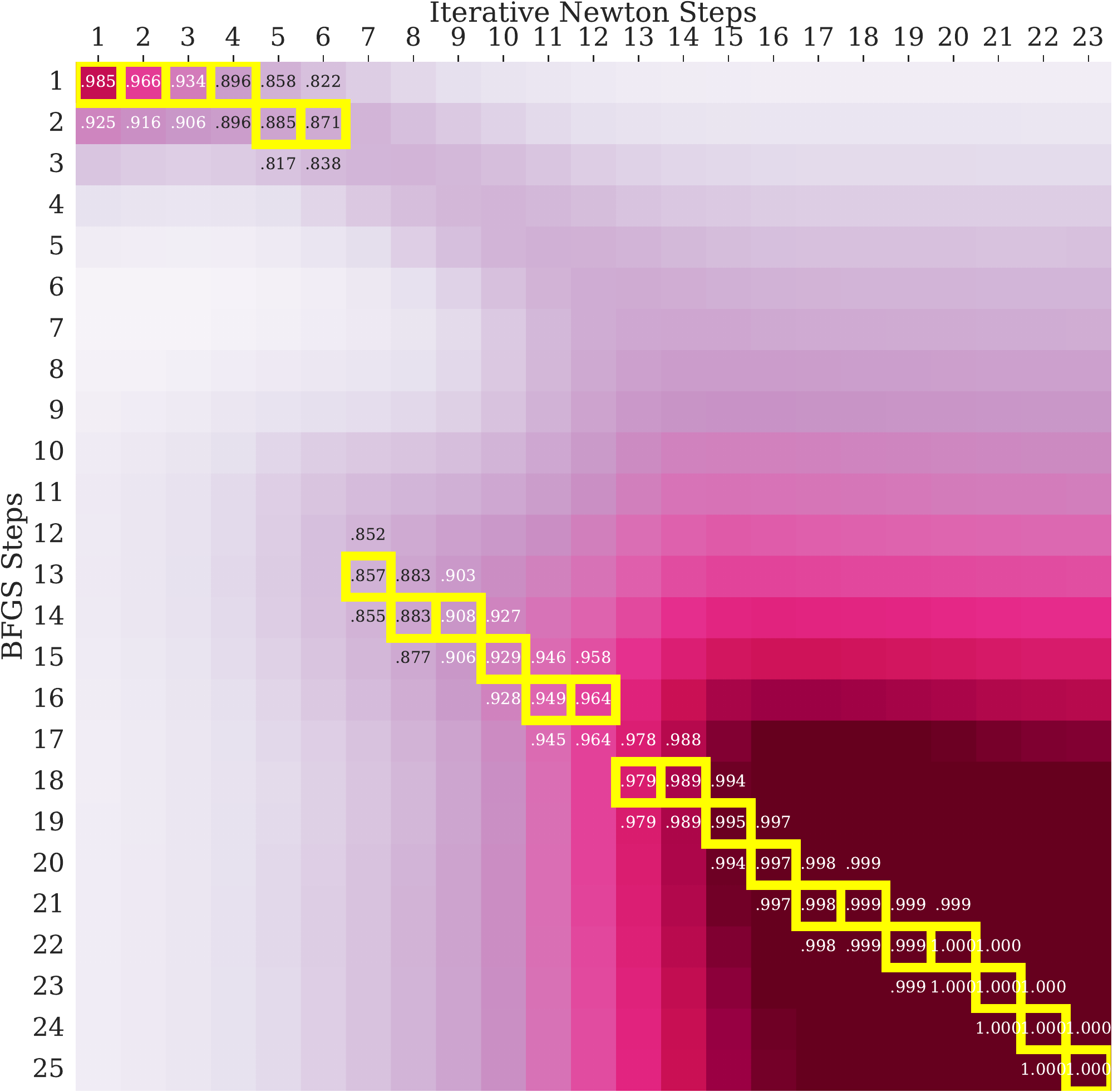}
    }
    \caption{Best-Matching Steps on Similarity of Residuals Help Compare Convergence Rates. (a: top-left) When comparing \newton\ and \gd, there is an exponential trend -- showing \newton\ converges exponentially faster than \gd. (b: top-right) When \newton\ is compared with itself in sub-figure, there is a linear trend -- showing they have the same convergence rate. (c: bottom) When \newton\ is compared to BFGS in sub-figure, there a linear trend after there are enough steps for BFGS to approximate second-order information -- showing \newton\ and BFGS share a similar convergence rate after sufficient BFGS steps.}
    \label{fig:best-matching}
\end{figure}

\subsection{Definitions for Evaluating Forgetting} \label{ssec:forgetting}
We measure the phenomenon of model forgetting by reusing an in-context example within $\{\bx_i, y_i\}_{i=1}^n$ as the test example $\bx_\mathrm{test}$. In experiments of \Cref{fig:lstm-experiment}, we fix $n = 20$ and reuse $\bx_\mathrm{test} = \bx_i$. We denote the ``Time Stamp Gap'' as the distance the reused example index $i$ from the current time stamp $n = 20$. We measure the forgetting of index $i$ as 
\begin{equation}
   \mathrm{Forgetting}(\calA, i) = \mathop{\mathbb E}_{\{\bx_i, y_i\}_{i=1}^n \sim P_{\mathcal{D}}} \mathrm{MSE}\Big(\mathcal A(\bx_i \mid \{\bx_i, y_i\}_{i=1}^n), y_i\Big)
\end{equation}
Note: the further away $i$ is from $n$, the more possible algorithm $\calA$ forgets.


\clearpage
\section{Detailed Proofs for  \Cref{sec:mechanism}} \label{app:mechanism}
In this section, we work on full attention layers with normalized ReLU activation $\sigma(\cdot) = \frac{1}{n} \mathrm{ReLU}(\cdot)$ given $n$ examples. 
\begin{definition}
\label{def:attn_full_relu}
    A full attention layer with $M$ heads and ReLU activation is also denoted as $\mathrm{Attn}$ on any input sequence $\bH = \begin{bmatrix}
        \bh_1, \cdots, \bh_N
    \end{bmatrix} \in \mathbb R^{D \times N}$, where $D$ is the dimension of hidden states and $N$ is the sequence length. In the vector form,
    \begin{equation}
         \tilde{\bh}_t = [\mathrm{Attn}(\bH)]_t = \bh_t +  \frac{1}{n} \sum_{m=1}^M \sum_{j=1}^n \mathrm{ReLU} \left(\inner{\bQ_m \bh_t, \bK_m \bh_j}\right) \cdot \bV_m \bh_j
    \end{equation}
\end{definition}
\begin{remark}
    This is slightly different from the \textbf{causal} attention layer (see \cref{def:attn}) in that at each time stamp $t$, the attention layer in \cref{def:attn_full_relu} has full information of all hidden states $j \in [n]$, unlike causal attention layer which requires $j \in [t]$. 
\end{remark}


\subsection{Helper Results}

We begin by constructing a useful component for our proof, and state some existing constructions from \citet{Akyrek2022WhatLA}.

\begin{restatable}{lemma}{LemmaSum} \label{lemma:sum}
    Given hidden states $\{\bh_1, \cdots, \bh_n\}$, there exists query, key and value matrices $\bQ, \bK, \bV$ respectively such that one attention layer can compute $\sum_{j=1}^n \bh_j$. 
\end{restatable}
\begin{proof}
We can pad each hidden state by $1$ and $0$'s such that $\bh_t' \leftarrow \begin{bmatrix}
    \bh_t \\ 1 \\ \zero_d
\end{bmatrix} \in \mathbb R^{2d+1}$ . We construct two heads where $\bQ_1 = \bK_1 = \bQ_2 = \begin{bmatrix}
    \bO_{d \times d} & \bO_{d \times 1} & \bO_{d \times d}\\
    \bO_{1 \times d} & 1  & \bO_{1 \times d} \\
    \bO_{d \times d} & \bO_{d \times 1} & \bO_{d \times d}
\end{bmatrix}$ and $\bK_2 = -\bK_1$. Then $\begin{bmatrix}
    \bO_{d \times d} & \bO_{d \times 1} & \bO_{d \times d}\\
    \bO_{1 \times d} & 1  & \bO_{1 \times d} \\
    \bO_{d \times d} & \bO_{d \times 1} & \bO_{d \times d}
\end{bmatrix} \bh_t' = \begin{bmatrix}
    \zero_d \\ 1 \\ \zero_d
\end{bmatrix}$.

Let $\bV_1 = \bV_2 = \begin{bmatrix}
    \bO_{(d+1) \times d} & \bO_{(d+1) \times (d+1)} \\
    n\bI_{d \times d}  &  \bO_{d \times (d+1)}
\end{bmatrix}$ so that $\bV_m \begin{bmatrix}
    \bh_j \\ 1 \\ \zero_d
\end{bmatrix} = \begin{bmatrix}
    \zero_{d+1} \\ n \bh_j 
\end{bmatrix}$. 

We apply one attention layer to these 1-padded hidden states and we have 
\begin{equation}
\begin{aligned}
    \tilde{\bh}_t  & = \bh_t' + \frac{1}{n} \sum_{m=1}^2 \sum_{j=1}^n \mathrm{ReLU} \left(\inner{\bQ_m \bh_t', \bK_m \bh_j'}\right) \cdot \bV_m \bh_j' \\
    &= \bh_t' +  \frac{1}{n} \sum_{j=1}^n \Big[\mathrm{ReLU}(1) + \mathrm{ReLU}(-1)\Big] \cdot \begin{bmatrix}
    \zero_{d+1} \\ n \bh_j 
\end{bmatrix} \\
    &= \begin{bmatrix}
        \bh_t \\ 1 \\ \zero_d
    \end{bmatrix} + 
    \begin{bmatrix}
    \zero_{d+1} \\  \sum_{j=1}^n \bh_j 
\end{bmatrix} = \begin{bmatrix}
        \bh_t \\ 1 \\ \sum_{j=1}^n \bh_j
    \end{bmatrix}
\end{aligned}
\end{equation}  
\end{proof}


\begin{proposition}[\citealp{Akyrek2022WhatLA}]\label{prop:akyrek}
    Each of \verb|mov|, \verb|aff|, \verb|mul|, \verb|div|  can be implemented by a single transformer layer. These four operations are mappings $\mathbb R^{D \times N} \rightarrow \mathbb R^{D \times N}$, expressed as follows, 

\verb|mov|$(\bH; s, t, i, j, i', j')$: selects the entries of the $s$-th column of $\bH$ between rows $i$ and $j$, and copies
them into the $t$-th column ($t \geq s$) of $\bH$ between rows $i'$ and $j'$. 

\verb|mul|$(\bH; a, b, c, (i, j), (i', j'), (i'', j''))$: in each column $\bh$ of $\bH$, interprets the entries between $i$ and $j$ as an $a \times b$ matrix $\bA_1$, and the entries between $i'$ and $j'$ as a $b \times c$ matrix $\bA_2$, multiplies these matrices together, and stores the result between rows $i''$ and $j''$, yielding a matrix in which each column has the form 
$\begin{bmatrix}
    \bh_{:i''-1}, \bA_1 \bA_2, \bh_{j'':}
\end{bmatrix}^\top$. This allows the layer to implement inner products. 

\verb|div|$(\bH; (i, j), i', (i'', j''))$: in each column $\bh$ of $\bH$, divides the entries between $i$ and $j$ by the absolute value of the entry at $i'$, and stores the result between rows $i''$ and $j''$, yielding a matrix in which every column has the form $\begin{bmatrix}
    \bh_{:i''-1}, \bh_{i:j} / |\bh_{i'}|, \bh_{j'':}
\end{bmatrix}^\top$.

\verb|aff|$(\bH; (i, j), (i', j'), (i'', j''), \bW_1, \bW_2, \mathbf{b})$: in each column $\bh$ of $\bH$, applies an affine transformation to the entries between $i$ and $j$ and $i'$ and $j'$, then stores the result between rows $i''$ and $j''$, yielding a matrix in which every column has the form $\begin{bmatrix}
    \bh_{:i''-1}, \bW_1 \bh_{i:j} + \bW_2 \bh_{i':j'} + \mathbf{b}, \bh_{j'':}
\end{bmatrix}^\top$. This allows the layer to implement subtraction by setting $\bW_1 = \bI$ and $\bW_2 = -\bI$.

\end{proposition}

\subsection{Proof of Theorem~\ref{thm:transformers_newton}}
\TransformersNewton*
\begin{proof} We break the proof into parts.

\paragraph{Transformers Implement Initialization $\bT^{(0)}  = \alpha \bS$.} 
Given input sequence $\bH := \{\bx_1, \cdots, \bx_n\}$, with $\bx_i \in \mathbb R^d$,  we first apply the \verb|mov| operations given by Proposition~\ref{prop:akyrek} (similar to \cite{Akyrek2022WhatLA}, we show only non-zero rows when applying these operations):
\begin{equation}
    \begin{bmatrix}
        \bx_1 & \cdots & \bx_n \\
        & & 
    \end{bmatrix} \overset{\mathrm{mov}}{\longrightarrow} \begin{bmatrix}
        \bx_1 & \cdots & \bx_n  \\
        \bx_1 & \cdots & \bx_n 
    \end{bmatrix}
\end{equation}

We call each column after \verb|mov| as $\bh_j$. With an full attention layer, one can construct two heads with query and value matrices of the form $\bQ_1^\top \bK_1 = -\bQ_2^\top \bK_2 = \begin{bmatrix}
    \bI_{d \times d} & \bO_{d \times d} \\
    \bO_{d \times d} & \bO_{d \times d}
\end{bmatrix}$ such that for any $t \in [n]$, we have 
\begin{equation}
    \sum_{m=1}^2 \mathrm{ReLU} \left(\inner{\bQ_m \bh_t, \bK_m \bh_j}\right) = \mathrm{ReLU}(\bx_t^\top \bx_j) +  \mathrm{ReLU}(-\bx_t^\top \bx_j) = \inner{\bx_t, \bx_j}
\end{equation}
Let all value matrices $\bV_m =n \alpha \begin{bmatrix}
    \bI_{d \times d} & \bO_{d \times d} \\
    \bO_{d \times d} & \bO_{d \times d}
\end{bmatrix}$ for some $\alpha \in \mathbb R$. Combining the skip connections, we have 
\begin{equation}
    \tilde{\bh}_t = \begin{bmatrix}
        \bx_t \\ \bx_t
    \end{bmatrix} + \frac{1}{n} \sum_{j=1}^n \inner {\bx_t, \bx_j} n \alpha \begin{bmatrix}
        \bx_j \\ \zero 
    \end{bmatrix}  = \begin{bmatrix}
        \bx_t \\ \bx_t
    \end{bmatrix} + \begin{bmatrix}
       \alpha \left(\sum_{j=1}^n \bx_j \bx_j^\top \right) \bx_t \\ \zero
    \end{bmatrix} = \begin{bmatrix}
        \bx_t + \alpha \bS \bx_t \\ \bx_t
    \end{bmatrix} 
\end{equation}
Now we can use the \verb|aff| operator to make subtractions and then
\begin{equation}
    \begin{bmatrix}
        \bx_t + \alpha \bS \bx_t \\ \bx_t
    \end{bmatrix}  \overset{\mathrm{aff}}{\longrightarrow}\begin{bmatrix}
        (\bx_t + \alpha \bS \bx_t) - \bx_t  \\ \bx_t
    \end{bmatrix}  = \begin{bmatrix}
         \alpha \bS \bx_t \\ \bx_t
    \end{bmatrix}
\end{equation}

We call this transformed hidden states as $\bH^{(0)}$ and denote $\bT^{(0)} = \alpha \bS$:
\begin{equation}
        \bH^{(0)} = \begin{bmatrix}
            \bh_1^{(0)} & \cdots & \bh_n^{(0)}
        \end{bmatrix} = \begin{bmatrix}
            \bT^{(0)} \bx_1 & \cdots & \bT^{(0)} \bx_n \\
            \bx_1 & \cdots & \bx_n
        \end{bmatrix} 
\label{eqn:input_prompt}
\end{equation}
Notice that $\bS$ is symmetric and thereafter $\bT^{(0)}$ is also symmetric.

\paragraph{Transformers implement Newton Iteration.} 
Let the input prompt be the same as \Cref{eqn:input_prompt}, 
    \begin{equation}
        \bH^{(0)} = \begin{bmatrix}
            \bh_1^{(0)} & \cdots & \bh_n^{(0)}
        \end{bmatrix} = \begin{bmatrix}
            \bT^{(0)} \bx_1 & \cdots & \bT^{(0)} \bx_n \\
            \bx_1 & \cdots & \bx_n
        \end{bmatrix} 
    \end{equation}
    We claim that the $\ell$'s hidden states can be of the similar form 
    \begin{equation}
        \bH^{(\ell)} = \begin{bmatrix}
            \bh_1^{(\ell)} & \cdots & \bh_n^{(\ell)}
        \end{bmatrix} = \begin{bmatrix}
            \bT^{(\ell)} \bx_1 & \cdots & \bT^{(\ell)} \bx_n \\
            \bx_1 & \cdots & \bx_n
        \end{bmatrix} 
    \end{equation}
    We prove by induction that assuming our claim is true for $\ell$, we work on $\ell + 1$:

    Let $\bQ_m = \tilde{\bQ}_m \underbrace{\begin{bmatrix}
        \bO_d & -\frac{n}{2} \bI_d \\
        \bO_d & \bO_d 
    \end{bmatrix}}_{\bG} , \bK_m = \tilde{\bK}_m \underbrace{\begin{bmatrix}
        \bI_d & \bO_d \\
        \bO_d & \bO_d 
    \end{bmatrix}}_{\bJ}$ where $\tilde{\bQ}_1^\top \tilde{\bK}_1 : = \bI$, $\tilde{\bQ}_2^\top \tilde{\bK}_2 : = -\bI$ and $\bV_1 = \bV_2 = \underbrace{\begin{bmatrix}
        \bI_d & \bO_d \\
        \bO_d & \bO_d 
    \end{bmatrix}}_{\bJ}$. A 2-head self-attention layer, with ReLU attentions, can be written has 
    \begin{equation}
        \bh_t^{(\ell+1)} = [\mathrm{Attn}(\bH^{(\ell)})]_t = \bh_t^{(\ell)}+ \frac{1}{n} \sum_{m=1}^2 \sum_{j=1}^n \mathrm{ReLU} \left(\inner{\bQ_m \bh_t^{(\ell)}, \bK_m \bh_j^{(\ell)}}\right) \cdot \bV_m \bh_j^{(\ell)}
    \end{equation}
    where 
    \begin{equation}
    \begin{aligned}
        & \sum_{m=1}^2\mathrm{ReLU} \left(\inner{\bQ_m \bh_t^{(\ell)}, \bK_m \bh_j^{(\ell)}}\right) \cdot \bV_m \bh_j^{(\ell)} \\
        &=  \Big[\mathrm{ReLU} \Big((\bG \bh_t^{(\ell)})^\top \underbrace{\tilde{\bQ}_1^\top \tilde{\bK}_1}_{\bI} (\bJ \bh_j^{(\ell)})\Big)  + \mathrm{ReLU} \Big((\bG \bh_t^{(\ell)})^\top \underbrace{\tilde{\bQ}_2^\top \tilde{\bK}_2}_{-\bI} (\bJ \bh_j^{(\ell)})\Big)  \Big] \cdot (\bJ \bh_j^{(\ell)})\\
        &=  \Big[\mathrm{ReLU} ((\bG \bh_t^{(\ell)})^\top (\bJ \bh_j^{(\ell)})) + \mathrm{ReLU} (-(\bG \bh_t^{(\ell)})^\top  (\bJ \bh_j^{(\ell)})) \Big] \cdot  (\bJ \bh_j^{(\ell)})  \\
        &= (\bG \bh_t^{(\ell)})^\top (\bJ \bh_j^{(\ell)}) (\bJ \bh_j^{(\ell)}) \\
        &=  (\bJ \bh_j^{(\ell)}) (\bJ \bh_j^{(\ell)})^\top (\bG \bh_t^{(\ell)})
    \end{aligned} 
    \end{equation}
    Plug in our assumptions that $\bh_j^{(\ell)} = \begin{bmatrix}
        \bT^{(\ell)} \bx_j \\ \bx_j 
    \end{bmatrix}$, we have $\bJ \bh_j^{(\ell)} = \begin{bmatrix}
        \bT^{(\ell)} \bx_j \\ \zero_d 
    \end{bmatrix}$ and $\bG \bh_t^{(\ell)} = \begin{bmatrix}
       -\frac{n}{2} \bx_t \\ \zero_d 
    \end{bmatrix}$, 
    we have 
    \begin{equation}\label{eqn:h_update}
        \begin{aligned}
\bh_t^{(\ell+1)} &= \begin{bmatrix}
        \bT^{(\ell)} \bx_t \\ \bx_t 
    \end{bmatrix} + \frac{1}{n} \sum_{j=1}^n \begin{bmatrix}
        \bT^{(\ell)} \bx_j \\ \zero_d 
    \end{bmatrix} \begin{bmatrix}
        \bT^{(\ell)} \bx_j \\ \zero_d 
    \end{bmatrix}^\top  \begin{bmatrix}
         -\frac{n}{2} \bx_t \\ \zero_d 
    \end{bmatrix} \\
    &= \begin{bmatrix}
        \bT^{(\ell)} \bx_t - \frac{1}{2} \sum_{j=1}^n  (\bT^{(\ell)} \bx_j) (\bT^{(\ell)} \bx_j)^\top  \bx_t \\ \bx_t  
    \end{bmatrix} \\
    &= \begin{bmatrix}
        \bT^{(\ell)} \bx_t - \frac{1}{2} \bT^{(\ell)} \left(\sum_{j=1}^n \bx_j  \bx_j^\top\right)  {\bT^{(\ell)}}^\top \bx_t \\ \bx_t  
    \end{bmatrix} \\
    &= \begin{bmatrix}
        \left(\bT^{(\ell)} - \frac{1}{2} \bT^{(\ell)} \bS {\bT^{(\ell)}}^\top \right) \bx_t \\ \bx_t 
    \end{bmatrix}
        \end{aligned}
    \end{equation} 
    Now we pass over an MLP layer with
    \begin{equation}
        \bh_t^{(\ell+1)} \leftarrow \bh_t^{(\ell+1)} + \begin{bmatrix}
            \bI_d & \bO_d \\ 
            \bO_d & \bO_d 
        \end{bmatrix} \bh_t^{(\ell+1)} = \begin{bmatrix}
        \left(2\bT^{(\ell)} -  \bT^{(\ell)} \bS {\bT^{(\ell)}}^\top \right) \bx_t \\ \bx_t 
    \end{bmatrix}
    \end{equation}
    Now we denote the iteration
    \begin{equation}
        \bT^{(\ell+1)} = 2\bT^{(\ell)} -  \bT^{(\ell)} \bS {\bT^{(\ell)}}^\top
    \end{equation}
    We find that ${\bT^{(\ell+1)}}^\top =  \bT^{(\ell+1)}$ since $ \bT^{(\ell)}$ and $\bS$ are both symmetric. It reduces to 
    \begin{equation}
        \bT^{(\ell+1)} = 2\bT^{(\ell)} -  \bT^{(\ell)} \bS {\bT^{(\ell)}}
    \end{equation}
    This is exactly the same as the Newton iteration.

\paragraph{Transformers can implement $\hat{\bw}_\ell^\mathrm{TF} = \bT^{(\ell)} \bX^\top \by$.}
Going back to the empirical prompt format $\{\bx_1, y_1, \cdots, \bx_n, y_n\}$. We can let parameters be zero for positions of $y$'s and only rely on the skip connection up to layer $\ell$, and the $\bH^{(\ell)}$ is then $\begin{bmatrix}
           \bT^{(\ell)} \bx_j & \zero  \\
            \bx_j &  \zero \\
            0 & y_j
        \end{bmatrix}_{j=1}^n $. We again apply operations from Proposition~\ref{prop:akyrek}:
\begin{equation}
    \begin{bmatrix}
           \bT^{(\ell)} \bx_j & \zero  \\
            \bx_j &  \zero \\
            0 & y_j
        \end{bmatrix}_{j=1}^n \overset{\mathrm{mov}}{\longrightarrow} \begin{bmatrix}
           \bT^{(\ell)} \bx_j & \bT^{(\ell)} \bx_j  \\
            \bx_j & \zero \\
            0 & y_j
        \end{bmatrix}_{j=1}^n \overset{\mathrm{mul}}{\longrightarrow} \begin{bmatrix}
            \bT^{(\ell)} \bx_j & \bT^{(\ell)}\bx_j  \\
             \bx_j & \zero \\
             0 & y_j \\
        \zero & \bT^{(\ell)} y_j \bx_j 
        \end{bmatrix}_{j=1}^n 
\label{eqn:Tyjxj}
\end{equation}
Now we apply \Cref{lemma:sum} over all even columns in \Cref{eqn:Tyjxj} and we have 
\begin{equation}
    \mathrm{Output} = \sum_{j=1}^n \begin{bmatrix}
         \bT^{(\ell)}\bx_j \\ \zero \\ y_j \\  \bT^{(\ell)} y_j \bx_j 
    \end{bmatrix} = \begin{bmatrix}
        \bxi \\ \bT^{(\ell)} \sum_{j=1}^n y_j \bx_j 
    \end{bmatrix} = \begin{bmatrix}
        \bxi \\ \bT^{(\ell)} \bX^\top \by 
    \end{bmatrix}
\end{equation}
where $\bxi$ denotes irrelevant quantities. Note that the resulting $\bT^{(\ell)} \bX^\top \by $ is also the same as \newton's predictor $\hat{\bw}_k = \bM_k \bX^\top \by$ after $k$ iterations. We denote $\hat{\bw}_\ell^\mathrm{TF} = \bT^{(\ell)} \bX^\top \by$.

\paragraph{Transformers can make predictions on $\bx_{test}$ by $\inner{\hat{\bw}_\ell^\mathrm{TF}, \bx_\mathrm{test}}$.} 

Now we can make predictions on text query $\bx_\mathrm{test}$:
\begin{equation}
    \begin{bmatrix}
        \bxi &  \bx_\mathrm{test} \\
        \hat{\bw}_\ell^\mathrm{TF} & \bx_\mathrm{test}
    \end{bmatrix} \overset{\mathrm{mov}}{\longrightarrow} \begin{bmatrix}
        \bxi &  \bx_\mathrm{test} \\
        \hat{\bw}_\ell^\mathrm{TF} & \bx_\mathrm{test} \\
        \zero & \hat{\bw}_\ell^\mathrm{TF}
    \end{bmatrix} \overset{\mathrm{mul}}{\longrightarrow} \begin{bmatrix}
        \bxi &  \bx_\mathrm{test} \\
        \hat{\bw}_\ell^\mathrm{TF} & \bx_\mathrm{test} \\
        \zero & \hat{\bw}_\ell^\mathrm{TF} \\
        0 & \inner{\hat{\bw}_\ell^\mathrm{TF}, \bx_\mathrm{test}}
    \end{bmatrix}
\end{equation}

Finally, we can have an readout layer $\bbeta_\mathrm{ReadOut} = \{\bu, v\}$ applied (see \cref{def:readout}) with $\bu = \begin{bmatrix}
    \zero_{3d} & 1 
\end{bmatrix}^\top$ and $v = 0$ to extract the prediction $ \inner{\hat{\bw}_\ell^\mathrm{TF}, \bx_\mathrm{test}}$ at the last location, given by $\bx_\mathrm{test}$. This is exactly how \newton\ makes predictions. 

\paragraph{To Perform $k$ steps of Newton's iterations, Transformers need $\mathcal O(k)$ layers.}

Let's count the layers:
\begin{itemize}
    \item \textbf{Initialization}: \verb|mov| needs $\mathcal O(1)$ layer; gathering $\alpha \bS$ needs $\mathcal O(1)$ layer; and \verb|aff| needs $\mathcal O(1)$ layer. In total, Transformers need $\mathcal O(1)$ layers for initialization. 
    \item \textbf{Newton Iteration}: each exact Newton's iteration requires $\mathcal O(1)$ layer. Implementing $k$ iterations requires $\mathcal O(k)$ layers.
    \item \textbf{Implementing $\hat{\bw}_\ell^\mathrm{TF}$}: We need one operation of \verb|mov| and \verb|mul| each, requiring $\mathcal O(1)$ layer each. Apply \Cref{lemma:sum} for summation also requires $\mathcal O(1)$ layer. 
    \item \textbf{Making prediction on test query}:  We need one operation of \verb|mov| and \verb|mul| each, requiring $\mathcal O(1)$ layer each.
\end{itemize}
Hence, in total, Transformers can implement $k$-step \newton\ and make predictions accordingly using $\mathcal{O}(k)$ layers. 
\end{proof}

{
\begin{remark}
We note that \cite{pmlr-v202-giannou23a} used 13 layers to compute one Newton Iteration, and in our construction, we need only one Transformer layer (with one attention layer and one MLP layer) to compute one Newton Iteration. At the same time, we didn’t use \cite{Akyrek2022WhatLA} for constructing Newton Iterations. \cite{Akyrek2022WhatLA} is applied to initialize Newton and for reading out the prediction. 

In our construction, only the initialization and read-out prediction components use causal attention and softmax because \cite{Akyrek2022WhatLA}'s construction is applied. To be more specific, those are the first 3 layers in initializing Iterative Newton and the last 5 layers in reading out the predictions from the computed pseudo-inverse. All the layers corresponding to the Iterative Newton updates are using full attention and normalized ReLU activations.
\end{remark}

}

{
\begin{remark}
\label{rmk:causal}
We note that our proof can be extended to causal attention for $n$ sufficiently larger than $d$. Under causal attention (see \Cref{def:attn}) with normalized ReLU activation, \Cref{eqn:h_update} can be rewritten as follows, given $t > d$, we first choose $\bG = \begin{bmatrix}
        \bO_d & -\frac{1}{2} \bI_d \\
        \bO_d & \bO_d 
    \end{bmatrix}$, where the coefficient on the upper right block is $-\frac{1}{2}$ instead of $-\frac{n}{2}$ originally. Then
\begin{equation}
    \begin{aligned}
        \bh_t^{(\ell+1)} &= \begin{bmatrix}
        \bT^{(\ell)} \bx_t \\ \bx_t 
    \end{bmatrix} + \frac{1}{t} \sum_{j=1}^t \begin{bmatrix}
        \bT^{(\ell)} \bx_j \\ \zero_d 
    \end{bmatrix} \begin{bmatrix}
        \bT^{(\ell)} \bx_j \\ \zero_d 
    \end{bmatrix}^\top  \begin{bmatrix}
         -\frac{1}{2} \bx_t \\ \zero_d 
    \end{bmatrix} \\
    &= \begin{bmatrix}
        \bT^{(\ell)} \bx_t - \frac{1}{2}~ \frac{1}{t}\sum_{j=1}^t  (\bT^{(\ell)} \bx_j) (\bT^{(\ell)} \bx_j)^\top  \bx_t \\ \bx_t  
    \end{bmatrix} \\
    &= \begin{bmatrix}
        \bT^{(\ell)} \bx_t - \frac{1}{2} \bT^{(\ell)} \left(\frac{1}{t}\sum_{j=1}^t \bx_j  \bx_j^\top\right)  {\bT^{(\ell)}}^\top \bx_t \\ \bx_t  
    \end{bmatrix} \\
    &= \begin{bmatrix}
        \left(\bT^{(\ell)} - \frac{1}{2} \bT^{(\ell)} \hat{\bSigma} {\bT^{(\ell)}}^\top \right) \bx_t \\ \bx_t 
    \end{bmatrix}
    \end{aligned}
\end{equation}
where $\hat{\bSigma} = \frac{1}{t}\sum_{j=1}^t \bx_j  \bx_j^\top$ is the estimate of the covariance matrix given seen in-context examples $\{\bx_j, y_j\}_{j=1}^t$ so far. Since $t > d$, $\hat{\bSigma}$ is an unbiased estimate for $\bSigma \approx \frac{1}{n} \bS$ if $n$ is sufficiently large. The rest of the proof follows similarly, up to the perturbation introduced by the error in the estimate of $\hat{\bSigma}$. 

We also note when $t < d$, the estimate  $\hat{\bSigma} = \frac{1}{t}\sum_{j=1}^t \bx_j  \bx_j^\top$ is no longer a valid covariance matrix since it's singular. Then this gives different $\bT^{(\ell+1)}$ for different time stamp $t < d$ and such error may propagate in our proof. Hence, a formal extension to causal models requires extensive analysis of the error bounds and it is beyond the scope of this work. Nonetheless, we provide a plausible direction of such an extension. 
    
\end{remark}
}

\subsection{\newton\ as a Sum of Moments Method}\label{sec:newton_moment}
Recall that \newton's method finds $\bS^\dagger$ as follows 
\begin{equation}
    \bM_0 = \underbrace{\frac{2}{\|\bS\bS^\top\|_2}}_{\alpha} ~ \bS^\top, \qquad \bM_k = 2 \bM_{k-1} - \bM_{k-1} \bS \bM_{k-1}, \forall k \geq 1.
\end{equation}

We can expand the iterative equation to moments of $\bS$ as follows. 

\begin{equation}
    \begin{aligned}
        \bM_1 &= 2 \bM_0 - \bM_0 \bS \bM_0  = 2 \alpha \bS^\top - 4 \alpha^2 \bS^\top \bS \bS^\top = 2 \alpha \bS - 4 \alpha^2 \bS^3 .
    \end{aligned}
\end{equation}

Let's do this one more time for $\bM_2$.
\begin{equation}
    \begin{aligned}
        \bM_2 &= 2 \bM_{1} - \bM_{1} \bS \bM_{1} = 2 (2 \alpha \bS - 4 \alpha^2 \bS^3) - (2 \alpha \bS - 4 \alpha^2 \bS^3) \bS (2 \alpha \bS - 4 \alpha^2 \bS^3) \\
        &= 4 \alpha \bS - 8 \alpha^2 \bS^3 - 4 \alpha^2 \bS^3 + 16 \alpha^3 \bS^5 - 16\alpha^4 \bS^7 \\
        &= 4 \alpha \bS - 12 \alpha^2 \bS^3  + 16 \alpha^3 \bS^5 - 16\alpha^4 \bS^7.
    \end{aligned}
\end{equation}

We can see that $\bM_k$ are  summations of moments of $\bS$, with respect to some pre-defined coefficients from the Newton's algorithm. Hence \newton\ is a special of an algorithm which computes an approximation of the inverse using second-order moments of the matrix, 
\begin{equation}
    \bM_k = \sum_{s=1}^{2^{k+1} - 1} \beta_s \bS^s
\label{eqn:som}
\end{equation}
with coefficients $\beta_s \in \mathbb R$.

We note that Transformer circuits can represent other sum of moments other than Newton's method. We can introduce different coefficients $\beta_i$ than in the proof of Theorem \ref{thm:transformers_newton} by scaling the value matrices or through the MLP layers.

\subsection{Estimated weight vectors lie in the span of previous examples} \label{app:span}

What properties can we infer and verify for the weight vectors which arise from Newton's method? A straightforward one arises from interpreting any sum of moments method as a kernel method. 

We can expand $\bS^s$ as follows
\begin{equation}
    \begin{aligned}
        \bS^s &= \left(\sum_{i=1}^t \bx_i \bx_i^\top\right)^s = \sum_{i=1}^t \left(\sum_{{j_1}, \cdots, {j_{s-1}}} \inner{\bx_i, \bx_{j_1}} \prod_{v=1}^{s-2} \inner{\bx_{j_{v}}, \bx_{j_{v+1}}}\right) \bx_i \bx_{j_{s-1}}^\top. 
    \end{aligned}
\end{equation}
Then we have 
\begin{equation}
    \begin{aligned}
        \hat{\bw}_t &=\bM_t \bX^\top \by  = \sum_{s=1}^{2^{t+1} - 1} \beta_s \bS^s \bX^\top \by  \\
        &=  \sum_{s=1}^{2^{t+1} - 1} \beta_s \left\{\sum_{i=1}^t \left(\sum_{{j_1}, \cdots, {j_{s-1}}} \inner{\bx_i, \bx_{j_1}} \prod_{v=1}^{s-2} \inner{\bx_{j_{v}}, \bx_{j_{v+1}}}\right) \bx_i \bx_{j_{s-1}}^\top \right\} \left\{\sum_{i=1}^t y_i \bx_i \right\} \\
        &= \sum_{s=1}^{2^{t+1} - 1} \beta_s \left(\sum_{i=1}^t \left(\sum_{{j_1}, \cdots, {j_{s}}} y_{j_s} \inner{\bx_i, \bx_{j_1}} \prod_{v=1}^{s-1} \inner{\bx_{j_{v}}, \bx_{j_{v+1}}}\right) \bx_i \right) \\
        &= \sum_{i=1}^t \underbrace{\left(\sum_{s=1}^{2^{t+1} - 1} \sum_{{j_1}, \cdots, {j_{s}}}  \beta_s y_{j_s} \inner{\bx_i, \bx_{j_1}} \prod_{v=1}^{s-1} \inner{\bx_{j_{v}}, \bx_{j_{v+1}}} \right)}_{\phi_t(i \mid \bX, \by, \bbeta)} \bx_i \\
        &= \sum_{i=1}^t \phi_t(i \mid \bX, \by, \bbeta) ~ \bx_i
    \end{aligned}
\end{equation}
where $\bX$ is  the data matrix, $\bbeta$  are coefficients of moments given by the sum of moments method and $\phi_t(\cdot)$ is some function which assigns some weight to the $i$-th datapoint, based on all other datapoints.
Therefore if the Transformer implements a sum of moments method (such as Newton's method), then its induced weight vector $\tilde{\bw}_t(\mathrm{Transformers} \mid \{\bx_i, y_i\}_{i=1}^t)$ after seeing in-context examples $\{\bx_i, y_i\}_{i=1}^t$  should lie in the span of the examples $\{\bx_i\}_{i=1}^t$:
\begin{equation}
    \tilde{\bw}_t(\mathrm{Transformers} \mid \{\bx_i, y_i\}_{i=1}^t) \stackrel{?}{=} \mathrm{Span}\{\bx_1, \cdots, \bx_t\} = \sum_{t=1}^t a_i \bx_i \qquad \textrm{for coefficients $a_i$}.
    \label{eqn:linearity}
\end{equation}
We test this hypothesis. Given a sequence of in-context examples $\{\bx_i, y_i\}_{i=1}^t$, we fit  coefficients
$\{a_i\}_{i=1}^t$ in \Cref{eqn:linearity} to minimize MSE loss:
\begin{equation}
   \{\hat{a}_i\}_{i=1}^t  = \argmin_{a_1, a_2, \cdots, a_t \in \mathbb R}\left\|\tilde{\bw}_t(\mathrm{Transformers} \mid \{\bx_i, y_i\}_{i=1}^t) - \sum_{t=1}^t a_i \bx_i\right\|_2^2.
\end{equation}
We then measure the quality of this fit across different number of in-context examples $t$, and visualize the residual error  in \Cref{fig:higher_order}. We find that even when $t < d$, Transformers' induced weights still lie close to the span of the observed examples $\bx_i$'s. This provides an additional validation of our proposed mechanism. 

\begin{figure}[H]
    \centering
    \includegraphics[width=0.6\linewidth]{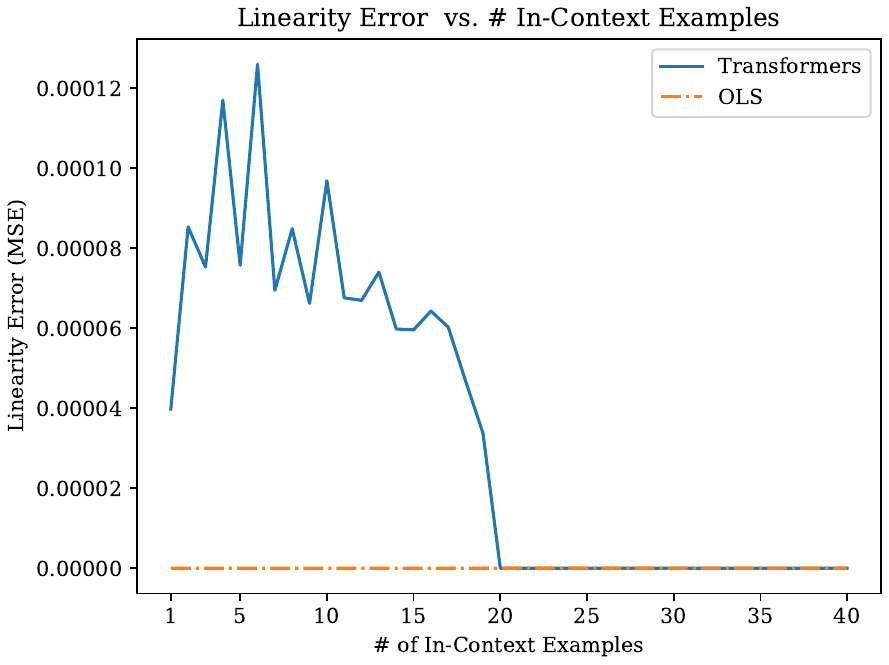}
    \caption{Verification of hypothesis that the Transformers induced weight vector $\bw$ lies in the span of observed examples $\{\bx_i\}$. }
    \label{fig:higher_order}
\end{figure}

\end{document}